\documentclass[11pt, english,english]{article}
\usepackage[applemac]{inputenc}

\usepackage{graphicx,epsfig,amsthm}
\usepackage{graphicx,epsfig,color}
\usepackage{amsmath,mathrsfs} 
\usepackage{amssymb} 
\usepackage{amsfonts}
\usepackage{mathtools}
\usepackage{epstopdf}
\usepackage{ifthen}
\usepackage{bbm}
\graphicspath{{eps/}}
\usepackage{bookmark}

\usepackage{comment}
\excludecomment{extra}
\excludecomment{k}
\includecomment{krad}
\newcommand{\inR}{\in \mathbb{R}}

\newcommand{\C}{ \mathbb{C}}
\newcommand{\R}{ \mathbb{R}}
\newcommand{\Z}{ \mathbb{Z}}

\newcommand{\N}{ \mathbb{N}}

\newcommand{\eqdef}{\stackrel{\vartriangle}{=}}
\newcommand{\Top}{\mathsf{T}}
\newcommand{\Lop}{{\rm L}}
\newcommand{\Dop}{{\rm D}}

\newcommand{\dint}{{\rm d}}

\newcommand{\Fourier}{ \mathcal{F}}

\newcommand{\bw}{{\boldsymbol \omega}}

\newcommand{\bk}{{\boldsymbol k}}

\DeclareMathOperator*{\esssup}{ess\,sup}

\def\V#1{{\boldsymbol{#1}}}         % vectors
\def\Spc#1{{\mathcal{#1}}}  % spaces
\def\M#1{{\bf{#1}}}  % matrices
\def\Op#1{{\mathrm{#1}}}  % operator
\def\ee{\mathrm{e}} 
\def\jj{\mathrm{i}}

\def\Proj{\mathrm{Proj}} 
\def\Identity{\mathrm{Id}} %

\newcommand{\embedC}{\xhookrightarrow{}}
\newcommand{\embedD}{\xhookrightarrow{\mbox{\tiny \rm d.}}}
\newcommand{\embedIso}{\xhookrightarrow{\mbox{\tiny \rm iso.}}}
\newcommand{\toC}{\xrightarrow{\mbox{\tiny \ \rm c.  }}}

\renewcommand{\[}{\begin{equation}}
\renewcommand{\]}[1]{\label{eq:#1}\end{equation}}

\newcommand{\astrad}{\circledast}

\newtheorem{definition}{Definition}
\newtheorem{proposition}{Proposition}
\newtheorem{corollary}{Corollary}

\newtheorem{lemma}{Lemma}
\newtheorem{theorem}{Theorem}

\title{From Kernel Methods to Neural Networks:\\ A Unifying 
Variational Formulation% of Supervised Learning:  
}

\author{
Michael Unser \thanks{Biomedical Imaging Group, \'Ecole polytechnique f\'ed\'erale de Lausanne (EPFL),
Station 17, CH-1015, Lausanne, Switzerland ({\tt michael.unser@epfl.ch}). }
 }

\begin{document}

\maketitle

% the following settings can be set or left blank at first

%\frontmatter  % title page, contents, catalog information

%\maketitle
%\showindexmarks

%
%\mainmatter
\begin{abstract}
%\begin{k}
%test
%\end{k}
%\begin{krad}
%test2 $2$
%\end{krad}
The minimization of a data-fidelity term and an additive regularization functional gives rise to a powerful framework for supervised learning. 
In this paper, we present a unifying regularization functional that depends on an operator $\Lop$ 
and on a generic Radon-domain norm. % $\|\cdot\|$. % in Radon domain. 
We establish the existence of a minimizer and give the parametric form of the solution(s) under very mild assumptions.  When the norm is Hilbertian, %$\|\cdot\|=\|\cdot\|_{L_2}$, 
the proposed formulation yields a solution that involves radial-basis functions and is compatible with the classical methods of machine learning. By contrast, for %$\|\cdot\|=\|\cdot\|_{\Spc M}$ (
the total-variation norm, the solution takes the form of a two-layer neural network with an activation function that is determined by the regularization operator. In particular, we retrieve the popular ReLU networks %with skip connections 
by letting the operator be the Laplacian.
We also characterize the solution for the intermediate regularization norms $\|\cdot\|=\|\cdot\|_{L_p}$ with $p\in(1,2]$. %
Our framework offers guarantees of universal approximation for a broad family
of regularization operators or, equivalently, for a wide variety of shallow neural networks, including the cases (such as ReLU) where the activation function is increasing polynomially. It also explains the favorable role of bias and skip connections in neural architectures.
%The case $p=1$ is excluded because the corresponding functional minimization problem is ill-defined.
\end{abstract}
\pagebreak
\tableofcontents

%\section{Introduction}
%Deep neural networks are overtaking the planet.
%
%The purpose of these notes are to provide a concise introduction to neural networks that is geared to mathematically-inclined readers.

\section{Introduction}
Regularization theory constitutes a powerful framework for the derivation of algorithms for supervised learning \cite{Chen2002,Poggio1990,Poggio2003}. Given a series of data points $(\V x_m, y_m) \in \R^d \times \R$, $m=1,\dots,M$, the basic problem (regression) is to find a mapping $f: \R^d \to \R$ such that $f(\V x_m)\approx y_m$, without overfitting. The standard paradigm is to let $f$ be the
minimizer of a cost that consists of a data-fidelity term and an additive regularization functional  \cite{Bishop2006}. The minimization proceeds over a prescribed class $\Spc H$ of candidate functions. One usually distinguishes between the parametric approaches (e.g., neural networks), where $\Spc H=\Spc H_{\Theta}$ is a family of functions specified by a finite set of parameters $\theta \in \Theta$ (e.g., the weights of the network), and the nonparametric ones, %which are the ones of interest here and 
where the properties of the solution are controlled by the regularization functional. % $R(f)$.
The focus of this paper is on the nonparametric techniques. They rely on {\em functional optimization}, which means that the minimization proceeds over a space of functions rather than over a set of parameters.
The regularization is usually chosen to be an increasing function of the norm associated with a particular Banach space, which results in a well-posed problem \cite{deBoor1966, DeBoor1976,Unser_2020}.
% and the solution is then determined determined by functional optimization. This powerful concept that is exploited in the classical formulation of learning 
%This is motivated by the fact that the underlying functional optimization problem is tractable mathematically---in the sense that it admits a solution--- and that it sometimes 
%Under suitable conditions, this formulation %While the functional setting is mathematically more involved than the parametric one, it
%can results in elegant closed-form descriptions, as exemplified by the classical theory of reproducing kernel Hilbert spaces (RKHS) \cite{DEvito2004}.

The functional-optimization point of view is often constructive, in 
that 
it suggests or supports explicit learning architectures. For instance, the choice of the Hilbertian regularization $R(f)=\|f\|^2_{\Spc H}$
%=\|\Lop f\|^2_{L_2}$ 
where $\Spc H$ is a reproducing kernel Hilbert space (RKHS) 
%and $\Lop$ is an appropriate\footnote{A Hilbert space $\Spc H$ of functions on $\R^d$ is a RKHS  iff.\
%$\delta(\cdot-\V y) \in \Spc H'$ for all $\V y \in \R^d$ \cite{Aronszajn1950,Berlinet2004}. Its reproducing kernel is then given by
%$r_\Spc H(\V x,\V y)=\Op R^{-1}\{\delta(\cdot-\V y)\}(x)$ where $\Op R: \Spc H \to \Spc H'$ is the Riesz map that isometrically maps the Hilbert space $\Spc H$ into its dual $\Spc H'$. The underlying regularization operator is $\Op L=\Op R^{1/2}: \Spc H \to L_2(\R^d)$ (the square-root of $\Op R=\Lop\Lop^\ast$, which is symmetric and positive-definite).} operator 
results in a closed-form solution that is a linear combination of kernels positioned on the data \cite{Wahba1990,Berlinet2004}. In fact, the RKHS setting yields a generic class of  estimators that is compatible with the classical kernel-based methods of machine learning, including support vector machines \cite{Wahba1990,Poggio1990,Scholkopf1997,Scholkopf2001,Alvarez2012,Unser_2020}. Likewise, adaptive kernel methods are justifiable from the minimization of a generalized total-variation norm, which favors sparse representions \cite{Boyer2018,Bredies2018,Aziznejad2021}. These latter results actually take their root in spline theory \cite{Fisher1975,Mammen1997,Unser2017}. Similarly, it has been demonstrated that shallow ReLU networks are solutions 
of functional-optimization problems with an appropriate %form of 
regularization.
One way to achieve this is to start from an explicit parameterization of an infinite-width network \cite{Bach2017} (the reverse engineering/synthesis approach). Another way is to consider a regularization operator that is matched to the neuronal activation with a $L_1$-type penalty\footnote{The precise formulation involves the $\Spc M$-norm (or total variation), which is the weak form of $L_1$ associated with space of bounded Radon measures. In our account, we take it as the default norm for the Lebesgue exponent $p=1$, with a slight abuse of language.}; for instance, a second-order derivative for $d=1$ \cite{Savarese2019, Parhi2020} or, more generally, the Radon-domain counterpart of the Laplace operator whose Green's function is precisely a ReLU ridge \cite{Ongie2020b,Parhi2021b,Unser2022_Ridges}.
% first in 1D \cite{Savarese2019, Parhi2020}, and then in multiple dimensions by considering a Radon-domain regularization that involves the Laplace operator \cite{Parhi2021b}. 
Similar optimality results can also be stated within the framework of reproducing-kernel Banach spaces \cite{Bartolucci2021}, which is a formal point of view that bridges the synthesis and analysis approach of \cite{Bach2017} and \cite{Parhi2021b}, respectively. 

The second important benefit of the functional-optimization approach is that it gives insight on the approximation capabilities (expressivity) of the resulting learning architectures. This information is encapsulated in the choice/definition of the native space $\Spc H$ (typically, a Sobolev space), which goes hand-in-hand with the regularization functional. Roughly speaking, the native space $\Spc H$ ought to be ``large enough'' to allow for the approximation of any continuous function with an arbitrary degree of precision. This universal approximation property is a central theme in the theory of radial-basis functions (RBFs) \cite{Micchelli1986,Wendland2005}. In machine learning, the kernel estimators that meet this approximation requirement are called {\em universal} \cite{Micchelli2006}. When the basis functions are shifted replicates of a single template $h: \R^d \to \R$, then the condition is equivalent to $h$ being strictly positive definite, which means that its Fourier transform is real-valued symmetric, and (strictly) positive % and non-vanishing 
\cite{Buhmann2003}. Similar guarantees of universal approximation exist for (shallow) neural networks under  mild conditions on the activation functions \cite{Cybenko1989,Hornik1989,Mhaskar1992,Barron1993,Pinkus1999}. The main difference with the RKHS framework, however, is that the universality results for neural nets usually make the assumption that the input domain is a compact subset of $\R^d$.

%A recurrent theme in machine learning is to relate learning architectures to functional optimization.
%The better known examples are kernel methods and support vector machines, which can be derived from the minimization of a cost functional that involves a Hilbertian regularization \cite{Wahba1990,Poggio1990,Scholkopf1997,Scholkopf2001,Alvarez2012,Unser_2020}.

The purpose of this paper is to unify and extend these various approaches by introducing a universal regularization functional. The latter has two components: an admissible differential operator $\Lop$, and an $L_p$-type Radon-domain norm. The resulting regularization operator is $\Lop_{\rm R}=\Op K_{\rm rad}\Op R \Lop$, where $\Op R$ is the Radon transform and $\Op K_{\rm rad}$
the ``filtering'' operator of computer tomography \cite{Natterer1984}. Our main result (Theorem \ref{Theo:RadonSplines}) gives the parametric form of the solution of the corresponding functional-optimization problems under minimal hypotheses. Remarkably, the solution for $p=2$ is compatible with the type of kernel expansions (RBFs)  of classical machine learning while, for $p=1$, %(or rather the weak form of $L_1$ associated with the space of bounded Radon measures) 
it maps into a neural network with one hidden layer. 
%In either case, the basis functions are directly deducible from the frequency response of the operator.
The foundation for this characterization is an abstract representer theorem for direct-sum Banach spaces \cite{Unser2022}. 

The primary effort in this paper consists in the development of a functional framework that is adapted to the Radon transform %matched to the supervised learning problem 
and that fulfills the hypotheses of the abstract theorem.
%Accordingly, the primary effort in this paper was to develop an explicit functional framework 
%%that fulfills
%%such as to meet
%%with the insurance that %
%with the goal of fulfilling 
%the mathematical hypotheses of the abstract theorem.
%% are met. 
The main contributions can be summarized as follows.

%The advantage of the functional framework is two-fold. First, it can help us chose and fine tune learning architectures.
%Second, by characterizing the underlying function spaces, it gives us a better understanding of their capabities; in particular, issues pertaining to expresiveness and universal approximation.

\begin{enumerate}
\item Construction and characterization of an extended family of native Banach spaces $\Spc X'_{\Op L_{\rm R}}(\R^d)$ associated with a generic Radon-domain norm $\|\cdot\|_{\Spc X'}$ and a differential operator $\Lop$, under the general admissibility conditions stated in Definition \ref{Def:SplineAdNonTrivial} (Theorem \ref{Theo:Native}). 
\item Proof that (i) the sampling functionals $\delta(\cdot-\V x_m): \Spc X'_{\Op L_{\rm R}}(\R^d) \to \R$ are weak*-continuous; and (ii) the adjoint of the regularization operator has a stable generalized inverse $\Lop^{\ast\dagger}_{\rm R}$ %, under our general admissibility conditions 
(see Theorem \ref{Theo:GenImpulse} and accompanying explanations). These technical points are essential to the argumentation (existence of solution). 
%To the best of our knowledge, these had not been addressed explicitly so far.
%, including the preferred ReLU scenario with $\Spc X'=\Spc M$ and $\Lop=\Delta$ \cite{Ongie2020b,Parhi2021b}.

\item Extension and unification of a number of earlier optimality results for RBF expansions and neural networks. While the present setup for $p=2$ and $\Lop=(-\Delta)^{\gamma}$ is reminiscent of thin-plate splines \cite{Duchon1977,Meinguet1979}, the resulting solution for a fixed $\gamma$ does not depend on the dimension $d$, which makes it easier to deploy.
% with the simplest choice being $\gamma=1$. 
Likewise, our variational formulation with $\Spc X'=\Spc M$ extends the results of Parhi and Nowak \cite{Parhi2021b} by: (i) proving that the neural network parameterization applies to {\em all the extreme points} of the solution set, and (ii) by covering a much broader class of activation functions, including those with polynomial growth (of degree $n_0$). 

%\item The identification of Fourier-based criteria for guiding the choice of RBFs and/or activation functions. The main point is that there is a one-to-one Fourier-domain relation between the activation function %$\rho_{\rm rad}$ 
%(resp., the RBF) %$\rho_{\rm iso}$) 
%and the operator $\Lop$. Our admissibility criterion for $n_0=0$ is compatible with the condition used by Barron to prove the universality of neural networks with sigmoidal activations \cite{Barron1993}.

\item General guarantees of universality, subject to the admissibility condition in Definition \ref{Def:SplineAdNonTrivial}. While the result for $p=2$ is consistent with the known criteria %of universality
 for kernel estimators \cite{Micchelli2006}, its counterpart for neural networks $(\Spc X'=\Spc M)$ brings in a new twist: the addition of a polynomial component. The latter, which is not present in the traditional theory \cite{Barron1993,Pinkus1999}, is necessary to lift the hypothesis of a compact input domain.
The two cases of greatest practical relevance are the sigmoid and the ReLU activations which, in our formulation, require the addition of a bias ($n_0=0$) and an affine term $(n_0=1)$, respectively. Interestingly, the ReLU case yields a neural architecture with a skip connection akin to ResNet \cite{He2016}, which is highly popular in practice.

%\item Generalization of Parhi and Nowak's to a much larger class of ``spline-admissible'' operators with complete proof of existence and characterization of the solution.
%
%\item Description of the whole solution set as the convex hull of its extreme points which are proven to be ridge splines. By contrast, Parhi and Nowak only state the existence of a sparse minimizer that is a ridge spline.
%An important technical point is a complete proof that the sampling functionals are weak* continuous, which has not been available so far.
%

%\item Extension of the framework for a broader class of norms including $L_p$ and $L_2$, with the latter making the connection with classical techniques that rely on radial basis functions.

\end{enumerate}
The paper is organised as follows: We start with mathematical preliminaries in Section \ref{Sec:MathPrelim}. In particular, we state our criteria of admissibility for $\Lop$ and show how to
represent its polynomial null space.
In Section   \ref{Sec:Radon}, we review the main properties of the Radon transform and specify the dual pair $(\Spc X_{\rm Rad}, \Spc X'_{\rm Rad})$ of hyper-spherical Banach spaces that enter the definition of our native spaces. We also provide formulas for the (filtered) Radon transform of RBFs and ridges (the elementary constituents of neural networks). Section
\ref{Sec:VariationalFormulation} is devoted to the description and interpretation of our main result (Theorem
\ref{Theo:RadonSplines}). In particular, we draw a connection with RKHS in Section \ref{Sec:RBF}.
We discuss the issue of universality in Section \ref{Sec:Universal} and show in Section \ref{Sec:Antisym} how our framework can be extended to handle antisymmetric activations, including sigmoids. We complement our exposition in Section \ref{Sec:Examples} with a listing of specific configurations, many of which are intimately connected to splines. % (polynomial as well as fractional). 
The mathematical developments that support our formulation are presented in Section \ref{Sec:Math}. They include the characterization of the kernel of the inverse operator $\Op L^{\ast \dagger}_{\rm R}$% (Theorem \ref{Theo:GenImpulse})
---the enabling ingredient of our formulation---%in Section \ref{Sec:Predual} 
and the construction of the predual Banach space $\Spc X_{\Op L_{\rm R}}(\R^d)$.
% in Section \ref{Sec:Predual}.
%(Theorem \ref{Theo:PredualNative}).
\section{Mathematical Preliminaries}
\label{Sec:MathPrelim}
\subsection{Notations}
\label{Sec:Notations}
We shall consider multidimensional functions $f$ on $\R^d$ that are indexed by the variable $\V x \in \R^d$. To describe their partial derivatives, we use the multi-index 
$\V k=(k_1,\dots,k_d) \in \N^d$ (where $\N$ includes $0$) with the notational conventions $\V k!=\prod_{i=1}^d k_i!$, $|\V k|=k_1+\cdots+k_d$,
$\V x^{\V k}=\prod_{i=1}^d x_i^{k_i}$ for any $\V x \in \R^d$, and $\partial^\V k f(\V x)=\frac{\partial^{|\V k|}f(x_1,\dots,x_d)}{\partial^{k_1}_{x_1} \cdots \partial^{k_d}_{x_d}}$. This allows us to write the multidimensional Taylor expansion around $\V x=\V x_0$ of an analytical function $f: \R^d \to \R$ explicitly as
\begin{align}
\label{Eq:Taylor}
f(\V x)=\sum_{n=0}^\infty \sum_{|\V k|=n} \frac{\partial^\V k f(\V x_0) (\V x- \V x_0)^\V k}{\V k! }
\end{align}
where the internal summation is over all multi-indices $\V k$ such that $k_1+\cdots+k_d=n$.

Schwartz'  space of smooth and rapidly decreasing test functions $\varphi: \R^d \to \R$ equipped with the usual Fr\'echet-Schwartz topology is denoted by $\Spc S(\R^d)$. Its continuous dual is the space $\Spc S'(\R^d)$ of tempered distributions. In this setting, the Lebesgue spaces $L_p(\R^d)$ for $p\in[1,\infty)$ can be specified as the completion of 
$\Spc S(\R^d)$ equipped with the $L_p$-norm $\|\cdot\|_{L_p}$, denoted as $L_p(\R^d)=\overline{(\Spc S(\R^d),\|\cdot\|_{L_p})}$.
For the endpoint $p=\infty$, we have $\overline{(\Spc S(\R^d),\|\cdot\|_{L_\infty})}=C_0(\R^d)$ with $\|\varphi\|_{L_\infty}=\sup_{\V x \in \R^d}|\varphi(\V x)|$, which is the space of continuous functions that vanish at infinity. The continuous dual of $C_0(\R^d)$ is the space $\Spc M(\R^d)=\{f \in \Spc S'(\R^d): \|f\|_{\Spc M}<\infty\}$ of bounded 
Radon measures with
\begin{align}
\|f\|_{\Spc M}=\sup_{\varphi \in \Spc S(\R^d): \|\varphi\|_{L_\infty}\le 1} \langle f, \varphi \rangle.
\end{align}
The latter is a superset of $L_1(\R^d)$, which is isometrically embedded in it, in the sense that $\|f\|_{L_1}=
\|f\|_{\Spc M}$ for all $f \in L_1(\R^d)$.

The %$d$-dimensional 
Fourier transform of a function $\varphi \in L_1(\R^d)$ is defined as
\begin{align}
\widehat \varphi(\bw)\eqdef \Fourier\{\varphi\}(\bw)=\frac{1}{(2 \pi)^d} \int_{\R^d} \varphi(\V x) \ee^{-\jj \langle \bw, \V x \rangle} \dint \V x.\label{Eq:Fourier}
\end{align}
Since the Fourier operator $\Fourier$ continuously maps $\Spc S(\R^d)$ into itself, the transform can be extended by duality to the whole space $\Spc S'(\R^d)$ of tempered distribution. Specifically, $\widehat f=\Fourier\{f\} \in \Spc S'(\R^d)$ is the (unique) {\em generalized Fourier transform} of $f\in \Spc S'(\R^d)$ if and only if $\langle \widehat f, \varphi \rangle=\langle f, \widehat\varphi \rangle$ for all $\varphi  \in \Spc S(\R^d)$, where $\widehat\varphi=\Fourier\{\varphi\}$  is the ``classical'' Fourier transform of $\varphi$ defined by \eqref{Eq:Fourier}.

To control the minimal order $\alpha\ge 0$ of decay (resp., the maximal rate of growth) of functions, we use the dual pair of spaces
$L_{1,\alpha}(\R^d)$ and $L_{\infty,-\alpha}(\R^d)=\big(L_{1,\alpha}(\R^d)\big)'$. These are the Banach spaces associated with the weighted norms
\begin{align}
\label{Eq:L1weight}
%L_{1,\alpha}(\R^d)=\{f: 
\|f\|_{L_{1,\alpha}}\eqdef \int_{\R^d} (1+\|\V x\|)^\alpha |f(\V x)| \dint \V x
%<\infty\},
\end{align}
\begin{align}
%L_{\infty,-\alpha}(\R^d)=\{f: 
\|f\|_{L_{\infty,-\alpha}} \eqdef \esssup_{\V x\in \R} (1+\|\V x\|)^{-\alpha} |f(\V x)|,
%<\infty\},
\end{align}
respectively.
%with $L_{\infty,-\alpha}(\R^d)=\big(L_{1,\alpha}(\R^d)\big)'$. 
Specifically, the inclusion $f\in L_{\infty,-n_0}(\R^d)$ with $n_0 \in \N$ indicates that $f$ cannot grow faster than a polynomial of degree $n_0$, 
while the condition $f \in L_{1,\alpha}(\R^d)$ implies that $f(\V x)$ must be locally integrable and must decay (slightly) faster
than $1/\|\V x\|^{\alpha+d}$ as $\|\V x\| \to \infty$.

\subsection{Admissible Regularization Operators}
\label{Sec:Operators}
The regularization operators $\Lop$ that are of interest to us are linear, shift-invariant (LSI), and isotropic. Consequently, they are characterized by a frequency response
$\widehat L$ that is purely radial, with $\widehat L(\bw)=\widehat L_{\rm rad}(\|\bw\|)$, where the radial profile $\widehat L_{\rm rad}: \R \to \R$ is a continuous, symmetric function.
Our condition for admissibility is that $\Lop$ be invertible in an appropriate sense.

\begin{definition} [Spline-admissible operators with trivial null space]
\label{Def:SplineAdm1}
An isotropic LSI operator $\Lop$ has a trivial null space if its radial frequency profile $\widehat L_{\rm rad}$ does not vanish over $\R$. We then say that it is {\em spline-admissible} if
$1/\widehat L_{\rm rad} \in L_1(\R)$ and $\rho_{\rm rad}=\Fourier^{-1}\{1/\widehat L_{\rm rad}\} \in L_1(\R)$ where  the operator $\Fourier^{-1}: L_1(\R) \to C_0(\R)$ is the classical inverse Fourier transform.
\end{definition}
The typical scenario is $\widehat L(\bw)=(1 + \|\bw\|^2)^{\alpha/2}$ with $\alpha\ge 1$, which results in a stable inverse operator $\Op L^{-1}$ whose radially symmetric impulse response is the Bessel potential of order $\alpha$. These operators play a central in the theory of Sobolev spaces \cite{Grafakos2008}.

Distribution theory allows us to go further and to invert operators with nontrivial null spaces, but only if the zeros of the frequency response are located at isolated points. When the operator is isotropic, this
reduces the options to the cases where $\widehat L(\bw)$ has a (multiple) zero at $\bw=\V 0$.
Specifically, we shall say that $\Lop$ is of order $\gamma_0$ if
$|\widehat L(\bw)|/\|\bw\|^{\gamma_0}=C_0$ as $\|\bw\|\to 0$. The second important parameter is the asymptotic growth exponent of $\widehat L(\bw)$. This is the largest index $\gamma_1$ such that $|\widehat L(\bw)|\ge C_1 \|\bw\|^{\gamma_1},$ for all $\|\bw\|>R$. It determines the smoothness of the Green's function of the operator.
%is not necessary the same as $\gamma_0$. 
%In that respect, we note that $\gamma_1>1$ is a minimal requirement for admissibility, irrespective of wether the null space is trivial or not.

\begin{definition} [Spline-admissible operators with nontrivial null space]
\label{Def:SplineAdNonTrivial}
An isotropic LSI operator $\Lop$ with radial frequency profile $\widehat L_{\rm rad}$ is said to be {\em spline-admissible} with a polynomial null space of degree $n_0$ if the following conditions are satisfied.
\begin{enumerate}
\item The profile $\widehat L_{\rm rad}$ does not vanish over $\R$, except for a zero of order $\gamma_0\in (n_0,n_0+1]$ at the origin; that is, $|\widehat L_{\rm rad}(\omega)|/|\omega|^{\gamma_0}=C_0$ as $\omega \to 0$.
\item There exists an order $\gamma_1> 1$, a constant $C_1>0$, and a radius $R_1>0$ such that $|\widehat L_{\rm rad}(\omega)|\ge C_1 |\omega|^{\gamma_1}$ for all $|\omega|>R_1$ (ellipticity).

\item For all $\varphi \in \Spc S(\R^d)$, $\Lop^{\ast}\{\varphi\} \in L_{1,n_0}(\R^d)$.
%, which tolerates the spoiling of the rapid decay of test functions up to some critical order of algebraic decay.
\end{enumerate}

\end{definition}
The connection between Condition 1 and the null space of $\Lop$ will be explained in Section \ref{Sec:Null}.
Conditions 1 and 2 with $\gamma_1>1$ ensure that $\rho_{\rm rad}=\Fourier^{-1}\{1/\widehat L_{\rm rad}\}$, which is the generalized inverse Fourier transform of  the distribution $1/\widehat L_{\rm rad}$, is identifiable as a continuous function $\R \to \R$. %, which is a prerequisite for an analytical characterization.
The order $\gamma_1$ actually controls the degree of differentiability (Sobolev smoothness) of $\rho_{\rm rad}$. 
Condition 3 is a mild technical constraint on the decay of $\Lop^\ast\varphi$; this constraint has not appeared to be a practical limitation so far. For instance, if $\Lop$ is an ordinary differential operator (an arbitrary polynomial of the Laplace operator $\Delta$) then $\Lop^{\ast}\{\varphi\}\in \Spc S(\R^d)$, which is included in $L_{1,m}(\R^d)$ for any $m\in \Z$. We use this third condition for the handling of fractional operators whose impulse response decays slowly.

An attractive class of admissible operators with $\gamma_0=\gamma_1=\alpha$ and $n_0=\lceil \alpha-1\rceil$ are
the fractional Laplacians $(-\Delta)^{\tfrac{\alpha}{2}}$ with $\alpha\in(1,\infty)$ whose frequency response is $\|\bw\|^{\alpha}$. The inverse of the fractional Laplacian of order $\alpha$, which corresponds to a frequency-domain multiplication by $\|\bw\|^{-\alpha}$, is denoted by $(-\Delta)^{-\tfrac{\alpha}{2}}$. Both operators are part of the same family (isotropic LSI and scale-invariant), their distributional impulse response being given by \begin{align}
k_{\alpha,d}(\V x)=\Fourier^{-1}\left\{\frac{1}{\|\bw\|^{\alpha} } \right\}(\V x)=\begin{cases}
A_{\alpha,d} \; \|\V x\|^{\alpha-d},& \alpha-d, -\alpha \notin 2\N\\
B_{n,d} \; \|\V x\|^{2n} \log(\|\V x\|) , & \alpha-d=2n \in 2\N%=0,2,4, \cdots
\\
(-\Delta)^{n} \{\delta\}, & -\alpha/2=n \in \N,%=0,1,2, \cdots
\end{cases}
\label{Eq:GreenLap}
\end{align}
with %proportionality 
constants
$A_{d,\alpha}=
\frac{\Gamma\big( \tfrac{d-\alpha}{2}\big)}{2^\alpha \pi^{d/2} \Gamma\big( \tfrac{\alpha}{2}\big)}$
and $B_{d,n}=
\frac{(-1)^{1+n} }
{2^{2n+d-1}\pi^{d/2}  \Gamma\big( n+\tfrac{d}{2}\big) n! }$ \cite{Gelfand-Shilov1964,Samko1993}. The 
kernel $k_{\alpha,d}$ can also be interpreted as the Green's function of $(-\Delta)^{\tfrac{\alpha}{2}}$, with the corresponding radial profile in Definition \ref{Def:SplineAdNonTrivial} being
$\rho_{\rm rad}(t)=k_{\alpha,1}(t)$. 
In view of \eqref{Eq:GreenLap}, this means that $(-\Delta)^{\tfrac{\alpha}{2}}$ is admissible for $\alpha>1$.

We note that the impulse response of the filtering operator $\Op K$ in Theorem \ref{Theo:RadonS0} is proportional to $k_{-d+1,d}(\V x)$, which tells us that it decays asymptotically like $1/\|\V x\|^{2d -1}$ when $d$ is even, or is a power of the Laplacian (local operator) otherwise.
Functionally, this means that $\Op K\big(\Spc S(\R^d)\big)=\Spc S(\R^d)$ for even dimensions, and $\Op K\big(\Spc S(\R^d)\big)\subset L_{1, d-1-\epsilon}(\R^d)$ otherwise.
Likewise, the impulse response of the fractional Laplacians (of non-even order) decays asymptotically like
$1/\|\V x\|^{\alpha+d}$, which implies that $(-\Delta)^{\tfrac{\alpha}{2}}\big(\Spc S(\R^d)\big)\subset L_{1, \alpha-\epsilon}(\R^d)$ for %any 
arbitrarily small $\epsilon>0$, so that the third condition in Definition \ref{Def:SplineAdNonTrivial} is met. % (by a comfortable margin ?).

%For the scenarios where $\widehat L(\bw)=C_0 \|\bw\|^{\gamma_0}$ around the origin with $\gamma_0 >0$, it can be shown that the growth-limited null space of $\Lop$ is formed of the polynomials of degree $n_0=\lceil \gamma_0 -1\rceil$; that is,
%$$
%\Spc N_\Lop=\{f \in L_{\infty,-n_0}(\R^d): \Lop\{f\}=0\}=\{\sum_{|\V k|\le n_0} a_\V k \V x^{\V k}: a_{\V k} \in \R\}
%$$
%where the condition $f\in L_{\infty,-n_0}(\R^d) \Leftrightarrow |f(\V x)| \le C  (1+\|\V x\|)^{n_0}$ imposes an explicit limit on the rate of growth of acceptable input functions\footnote{ }.

\subsection{Nontrivial Null Space and Related Projectors}
\label{Sec:Null}%The necessary and 
A sufficient 
condition for an LSI operator $\Lop$ to annihilate the polynomials of degree $n_0$ is \cite[p. 131]{Unser2014book}
\begin{align}
\label{Eq:freqpolnul}
\partial^{\V k} \widehat L(\V 0)=0, \mbox{ for all } \V k \in \N^d \mbox{ with } |\V k|\le n_0.
\end{align}
We also recall that the directional derivative of a function along the direction $\V \xi \in \mathbb{S}^{d-1}$ (i.e.,
$\V \xi \in \R^d$ with $\|\V \xi\|=1$) is given by 
\begin{align}
\Dop_{\V \xi}f=\V \xi^\Top \V \nabla f =\xi_1\partial^{\V e_1}f + \cdots + \xi_d\partial^{\V e_d}f.
\end{align}
The operator $\Dop_{\V \xi}$ is LSI with frequency response  $\widehat{\Dop_{\V \xi}}(\bw)=(\jj \V \xi^\Top \bw)$.
The $n$th iterate of $\Dop_{\V \xi}$ yields the $n$th derivative along $\V \xi$ whose explicit expression in terms of partial derivatives is
\begin{align}
\Dop^n_{\V \xi}f(\V x)=\Fourier^{-1}\{(\jj \V \xi^\Top \bw)^n \hat f(\bw) \}(\V x)=\sum_{|\V k|=n} \frac{n! }{\V k! } \V \xi^{\V k}\partial^\V k f(\V x),
\label{Eq:DirDeriv}
\end{align}
where the right-hand side follows from the application of the multinomial expansion to 
$(\jj\V \xi^\Top \bw)^n=(\xi_1\jj\omega_1 + \dots + \xi_d\jj\omega_d)^n$. 

For isotropic operators, % $\Lop$ of interest to be isotropic.
%In the case of our interest, the operator $\Lop$ is isotropic
%In the present case where the frequency response of $\Lop$ is isotropic---that is, $\widehat L(\bw)=\widehat L_{\rm rad}(\omega)$ where $\omega=\|\bw\|$ is the radial frequency---
%Then, 
the directional derivatives $\Op D^n_{\V \xi}\widehat L(\V 0)$ do not dependent on the direction $\V \xi$ and coincide with the radial derivatives $\widehat L^{(n)}_{\rm rad}(0)$.
In view of \eqref{Eq:DirDeriv} and the identification 
$\partial^{\V k} \widehat L(\V 0)=\Dop^{k_1}_{\V e_1}\cdots \Dop^{k_d}_{\V e_d}\widehat L(\V 0)$, %=\widehat L^{(|\V k|)}_{\rm rad}(0)$, 
\eqref{Eq:freqpolnul} then has the radial equivalent 
\begin{align}
\label{Eq:radfreqpolnul}
\widehat L^{(n)}_{\rm rad}(0)=\frac{\dint^n\widehat L_{\rm rad}(0)}{\dint \omega^n}=0, \mbox{ for } n=0,1,\dots,n_0,
\end{align}
which is much simpler to test. It follows that an operator whose radial frequency profile is
such that $|\widehat L_{\rm rad}(\omega)|/|\omega|^{\gamma_0}=C_0$ as $\omega\to 0$
will annihilate all polynomials up to degree $n_0=\lceil \gamma_0-1\rceil$. 

The null space of a spline-admissible operator $\Lop$ of order $\gamma_0$ therefore consists of the polynomials of degree $n_0=(\gamma_0-1)$ when $\gamma_0$ is an integer and $n_0=\lfloor\gamma_0\rfloor$ otherwise  %(truncation to the integer part) 
when $\gamma_0 \notin \N$. We shall represent these polynomials by expanding them in the monomial/Taylor basis
\begin{align}
m_\V k(\V x)=\frac{\V x^{\V k}}{\V k!}
%\mbox{ and } n_0=\lceil \gamma_0-1\rceil.
\end{align}
with $|\bk|\le n_0$.
We also add a topological structure by equipping the space with the $\ell_2$ norm of the Taylor coefficients, which results in the description
\begin{align}
\label{Eq:PolNullspace}
\Spc P_{n_0}=\big\{p_0=\sum_{|\V k|\le n_0} b_\V k m_{\V k}: \|p_0\|_{\Spc P}%\eqdef \|(b_\V k)_{|\V k|\le n_0}\|_2
<\infty\big\} %\nonumber\\
\mbox{  with  }  \|p_0\|_{\Spc P}\eqdef \|(b_\V k)_{|\V k|\le n_0}\|_2.
\end{align}
To avoid a notational overload, we shall often denote this null space by $\Spc P$, with the convention that $\Spc P=\Spc P_{n_0}=\{0\}$ when $n_0=\lceil \gamma_0-1\rceil<0$ (for the operators $\Lop$ whose null space is trivial). 
The important point here is that  \eqref{Eq:PolNullspace} specifies a finite-dimensional Banach subspace of $\Spc S'(\R^d)$. Its continuous dual $\Spc P'$ is finite-dimensional as well, although it is composed of ``abstract'' elements $p^\ast_0 \in \Spc P'$ that are, in fact, equivalence classes in $\Spc S'(\R^d)$. Yet, it is possible to
 identify every dual element $p^\ast_0\in \Spc P'$ as a true function by selecting a particular dual basis $\{m_{\V k}^\ast\}_{|\V k|\le n_0}$ such that $\langle m^\ast_{\V k}, m_{\V k'}\rangle=\delta_{\V k-\V k'}$ (Kroneker delta). Our specific choice is 
 \begin{align}
 \label{Eq:Dualbasis}
m^\ast_\V k=(-1)^{|\V k|}\partial^{\V k}\kappa_{\rm iso} \in \Spc S(\R^d)
 \end{align}
 with $\V k\in \N^d$,  where $\kappa_{\rm iso}$ is the isotropic function described in Lemma \ref{Theo:IsotropicWindow}. 
 
\begin{lemma}[adapted from \cite{Unser2022_Ridges}]
\label{Theo:IsotropicWindow}
There exists an %entire 
isotropic window $\kappa_{\rm iso} \in \Spc S(\R^d)$ 
%\begin{align}
%\widehat{\kappa}_{\rm iso}(\bw)=1&\quad \|\bw\|\le R_0\\
%\widehat{\kappa}_{\rm iso}(\bw)\le 1&\quad R_0<|\bw\|\le 2R_0\\
%\widehat{\kappa}_{\rm iso}(\bw)=0 &\quad \|\bw\|\ge 2R_0
%\end{align}
such that
\begin{align}
\langle m_{\V k},(-1)^{|\V n|}\partial^{\V n}\kappa_{\rm iso} \rangle=\delta_{\V k -\V n}
%=
%\begin{cases}
%1,&\V k=\V n\\
%0,&\mbox{ otherwise}\end{cases}
\end{align} for all $\V k, \V n \in \N^d$, subject to the spectral constraints $\widehat{\kappa}_{\rm iso}(\bw)=1$ for $\|\bw\|< \frac{1}{2}$, $1\ge \widehat{\kappa}_{\rm iso}(\bw)\ge 0$ for $\frac{1}{2}<\|\bw\|<1$, and $\widehat{\kappa}_{\rm iso}(\bw)=0$ for $\|\bw\|\ge 1$.
% and $\widehat{\kappa}_{\rm iso}(\bw)=0$ for $\|\bw\|\ge R_0$. 
\end{lemma}
This allows us to describe the dual space explicitly as
\begin{align}
\label{Eq:DualNullspace}
\Spc P'=
\Spc P'_{n_0}=\big\{p_0^\ast=\sum_{|\V k|\le n_0} b^\ast_\V k m^\ast_{\V k}: \|p_0^\ast\|_{\Spc P'}%\eqdef \|(b^\ast_\V k)\|_2
<\infty\big\} \mbox{ with } \|p_0^\ast\|_{\Spc P'}\eqdef \|(b^\ast_\V k)\|_2
\end{align}
where each elements $p_0^\ast$ has a unique representation in terms of its coefficients $(b^\ast_\V k)_{|\V k|\le n_0}$.
We %can %also 
use the dual basis $\{m_\V k^\ast\}$ to specify the projection operator
$\Proj_{\Spc P}: \Spc S'(\R^d) \to \Spc P_{n_0}$ as
\begin{align}
\Proj_{\Spc P}\{f\}&=\sum_{|\V k|\le n_0} \langle f,m_\V k^\ast  \rangle \; m_\V k,
\end{align}
which is well-defined for any $f \in \Spc S'(\R^d)$ since $m_\V k^\ast \in \Spc S(\R^d)$. The ``transpose'' of this operator is
\begin{align}
\Proj_{\Spc P'}\{\nu\}&=\sum_{|\V k|\le n_0} \langle m_\V k,\nu  \rangle \; m^\ast_\V k,
\end{align}
%which returns the projection of $\nu
%$ onto $\Spc P'\subseteq \Spc S(\R^d)$ under the implicit assumption that $\nu$ has sufficient decay for $\nu \mapsto \langle m_\V k,\nu  \rangle$ to be well-defined---for instance, $\nu \in 
%L_{1,n_0}(\R^d)$ where
%\begin{align}
%L_{1,\alpha}(\R^d)=\{f: \|f\|_{L_{1,\alpha}}\eqdef \int_{\R^d} (1+\|\V x\|)^\alpha |f(\V x)| \dint \V x<\infty\}.,
%\end{align}
%which is the space of functions with an $\alpha$th order of algebraic decay. 
which returns the projection of $\nu
$ onto $\Spc P'_{n_0}\subseteq \Spc S(\R^d)$ under the implicit assumption that $\nu$ has sufficient decay for $\nu \mapsto \langle m_\V k,\nu  \rangle$ to be well-defined---for instance, $\nu \in 
L_{1,n_0}(\R^d)$.
% where
%\begin{align}
%\label{Eq:L1weight}
%L_{1,\alpha}(\R^d)=\{f: \|f\|_{L_{1,\alpha}}\eqdef \int_{\R^d} (1+\|\V x\|)^\alpha |f(\V x)| \dint \V x<\infty\}.
%\end{align}
%which is the space of functions with an $\alpha$th order of algebraic decay. 
Correspondingly, we also have that
$\Proj_{\Spc P'}\{\Lop^\ast\varphi\}=0$ for all $\varphi$ such that $\Lop^\ast\varphi \in L_{1,n_0}(\R^d)$ since $\langle m_\V k, \Lop^\ast\varphi\rangle=\langle \Lop m_\V k, \varphi\rangle=0$ for $|\V k|\le n_0$. The latter manipulation of the duality product is legitimate in reason of the inclusion $\Spc P_{n_0}\subset L_{\infty,-n_0}(\R^d)=\big(L_{1,n_0}(\R^d)\big)'$.
% where
%\begin{align}
%L_{\infty,-n_0}(\R^d)=\{f: \|f\|_{L_{\infty,-n_0}} \eqdef \sup_{\V x\in \R} (1+\|\V x\|)^{-n_0} |f(\V x)|<\infty\}
%\end{align}
%is the space of measurable functions that do not grow faster than a polynomial of degree $n_0$.

Even though the null space of an admissible operator $\Lop$ may be non-trivial, its intersection with $\Spc S(\R^d)$ is always $\{0\}$.
%Finally, since $\Spc P_{n_0}\bigcap \Spc S(\R^d)=\{0\}$, we note that 
This implies that $\Lop^\ast=\Lop$ is injective on $\Spc S(\R^d)$ 
%(because $\Spc P_{n_0}\bigcap \Spc S(\R^d)=\{0\}$) 
with $\Lop^{-1\ast}\Lop^\ast\varphi=\varphi$ for all $\varphi \in \Spc S(\R^d)$
where $\Lop^{\ast-1}=\Lop^{-1}$ is the LSI operator whose frequency response is $1/|\widehat L|$. %We also note that Condition 4
%The part that is technically challenging is to properly extend this inverse to a much broader class of functions.

  \section{Radon Transform}
  \label{Sec:Radon}
The Radon transform extracts the integrals of a function on $\R^d$ over all 
hyperplanes of dimension $(d-1)$. These hyperplanes are indexed over
$\R \times \mathbb{S}^{d-1}$, where $\mathbb{S}^{d-1}=\{\V \xi \in \R^d: \|\V \xi\|_2=1\}$ is the unit sphere in $\R^d$. The coordinates of a hyperplane associated with an offset $t\in\R$ and a normal vector $\V \xi \in  \mathbb{S}^{d-1}$ satisfy $$
\V \xi^\Top \V x=\xi_1x_1+ \dots + \xi_d x_d = t.$$

Here, we first review the classical theory of the Radon transform \cite{Ludwig1966}, starting with the case of test functions (Section \ref{Sec:RadonSchwartz}), and extending it  by duality to tempered distributions (Section \ref{Sec:RadonDistributions}).
%The transform is first described for ordinary (test) functions and then extended to tempered distributions by duality. 
Then, in Section \ref{Sec:ComplementedSpaces}, we specify the Radon transform and its inverse on an appropriate class of intermediate Banach spaces $\Spc Y$ with $\Spc S(\R^d) \embedD \Spc Y \embedD \Spc S'(\R^d)$ (Theorem \ref{Theo:Complementedspaces}). Finally, in Section \ref{Sec:UsefulRadon}, we provide the (filtered) Radon transforms of the dictionary elements---isotropic kernels and ridges---that are relevant to our investigation.

%The surface measure over this plane is denoted by $s$ with $\dint s(\V x)=$\\ $\sum_{k=1}^d (-1)^k \dint \xi_1 \dots \dint \xi_{k-1}\dint \xi_{k+1}\dots \dint \xi_d ?$.
\subsection{Classical Integral Formulation}
\label{Sec:RadonSchwartz}
The Radon transform of the function $f\in L_1(\R^d)$ is defined as
\begin{align}
\Op R\{ f\}(t, \V \xi)%&=\int_{\V \xi^\Top\V x=t} f(\V x) \dint \V x\nonumber\\
&=\int_{\R^d}\delta(t-\V \xi^\Top\V x)  f(\V x) \dint \V x,\quad (t,\V \xi) \in \R \times \mathbb{S}^{d-1}. \label{Eq:Radon2}
\end{align}
%, under the assumption that the underlying integrals are well-defined---e.g., $f \in L_1(\R^d)$. 
The adjoint of $\Op R$ is the backprojection operator $\Op R^\ast$.
Its action on $g: \R \times \mathbb{S}^{d-1} \to \R$ yields the function
\begin{align}
%\V x \mapsto 
\Op R^\ast \{g\}(\V x)=\int_{\mathbb{S}^{d-1}} g(\underbrace{\V \xi^\Top\V x}_{t}, \V \xi)\dint \V \xi, \quad\V x\in \R^d.
\label{Eq:Backprojection}
\end{align}
%with $\V x \inR^d$. 

Given the $d$-dimensional Fourier transform $\hat f$ %=\Fourier\{f\}$ %\in C_0(\R^d)$ 
of  $f \in L_1(\R^d)$, one can calculate $ \Op R \{f\}(\cdot,\V \xi_0)$ for any fixed  $\V \xi_0 \in \mathbb{S}^{d-1}$ through the relation
\begin{align}
\Op R\{ f\}(t, \V \xi_0)=\frac{1}{2 \pi} \int_{\R} \hat f(\omega\V \xi_0) \ee^{\jj \omega t} \dint \omega= \Fourier^{-1}\{  \hat f(\cdot\V \xi_0) \}\{t\},
\label{Eq:CentralSliceTheo}
\end{align}
In other words, the restriction of $\hat f: \R^d \to \C$ along the ray $\{\bw=\omega \V \xi_0: \omega \in \R\}$ coincides with the 1D Fourier transform of $\Op R\{ f\}(\cdot, \V \xi_0)$, a property that is referred to as the {\em Fourier-slice theorem}.

%A very useful relation for calculating the Radon transform is the central slice theorem which relates the 1D Fourier transform of $\Op R\{ f\}(\cdot, \V \xi_0)$ with $\V \xi_0 \in \mathbb{S}^{d-1}$ fixed to the restriction of the $d$-dimensional Fourier transform $\hat f=\Fourier\{f\}$ along the ray $\{\bw=\omega \V \xi_0: \omega \in \R\}$. Specifically,
%\begin{align}
%\hat f(\omega\V \xi_0)=\int_\R \Op R\{ f\}(t, \V \xi_0) \ee^{-\jj \omega t}\dint t=\widehat{\Op R f(\cdot,\V \xi_0)}\}(\omega).
%\end{align}
%kown as the central slice theorem.
%

To describe the functional properties of the Radon transform, one needs the 
 (hyper)spherical (or Radon-domain) counterparts of the spaces described in Section \ref{Sec:Notations}. There, the Euclidean indexing with $\V x \in \R^d$ must be replaced by $(t, \V \xi) \in \R \times  \mathbb{S}^{d-1}$.
 % where $\mathbb{S}^{d-1}=\{\V \xi \in \R^d: \|\V \xi\|=1\}$ is the unit  sphere in $d$ dimensions.
The spherical counterpart of $\Spc S(\R^d)$ is $\Spc S(\R \times \mathbb{S}^{d-1})$. Correspondingly, an element $g \in \Spc S'(\R \times \mathbb{S}^{d-1})$ is a continuous linear functional on $\Spc S(\R \times \mathbb{S}^{d-1})$ whose action on the test function $\phi%(t,\V \xi)
$ is represented by the duality product $g: \phi \mapsto \langle g,\phi \rangle_{\rm Rad}$. When $g$ can be identified as an ordinary function $g: (t,\V \xi) \mapsto g(t,\V \xi)\in \R$, one can write that
\begin{align}
\langle g,\phi\rangle_{\rm Rad} = \int_{\mathbb{S}^{d-1}} \int_{\R} g(t, \V \xi) \phi(t, \V \xi) \dint t \dint \V \xi 
\end{align}
where $\dint \V \xi$ stands for the surface element on the unit sphere $\mathbb{S}^{d-1}$.
% with $\|\V \xi\|=1$. 
%For instance, for $d=2$, we can parametrize $\mathbb{S}^{1}$ %explicitly 
%by setting $\V \xi=(\cos \theta, \sin \theta)$ with $\dint \V \xi=\dint \theta$ for $\theta\in [0,2 \pi]$, which then yields
%\begin{align}
%\langle g,\phi \rangle_{\rm Rad} = \int_{0}^{2 \pi} \int_{\R} g(t, \theta) \phi(t, \theta)\dint t  \dint \theta.
%\end{align}
%Such explicit representations are also available in higher dimensions using hyper-spherical polar coordinates.
%Of special importance to us is the translated and rotated hyper-spherical Dirac distribution 
%$\delta\big(\cdot-(t_0, \V \xi_0)\big)=\delta(\cdot-t_0)\delta(\cdot-\V \xi_0)\in \Spc S'(\R \times \mathbb{S}^{d-1})$, which is defined as
%$\langle \delta\big(\cdot-(t_0, \V \xi_0)\big),\phi \rangle_{\rm Rad}\eqdef \phi(t_0,\V \xi_0)$ for all
%$\phi \in \Spc S(\R \times \mathbb{S}^{d-1})$ and any offset $(t_0,\V  \xi_0) \in \mathbb{S}^{d-1}$.
%This Dirac impulse, which is separable in the index variables $t$ and $\V \xi$, is included in the Banach space $\Spc M (\R \times \mathbb{S}^{d-1})$ (hyper-spherical Radon measures) with the property that
%$\|\delta\big(\cdot-(t_0, \V \xi_0)\big)\|_{\Spc M}=1$.

The key property for analysis is that the Radon transform is continuous on $\Spc S$ and invertible  \cite{Ludwig1966,Helgason2011,Ramm2020}.

\begin{theorem}[Continuity and invertibility of the Radon transform on $\Spc S(\R^d)$]
\label{Theo:RadonS0}
The Radon operator $\Op R$ continuously maps $\Spc S(\R^d) \to \Spc S(\R \times \mathbb{S}^{d-1})$. Moreover, $\Op R^\ast \Op K_{\rm rad} \Op R=\Op K \Op R^\ast \Op R=\Op R^\ast \Op R\Op K =\Identity \mbox{ on }\Spc S(\R^d)
$,
where $\Op K=(\Op R^\ast \Op R)^{-1}=c_d(-\Delta)^{\tfrac{d-1}{2}}$ with $c_d=\frac{1}{2(2\pi)^{d-1}}$ is the so-called ``filtering'' operator, %which is proportional to the Laplacian of order $\gamma=\tfrac{d-1}{2}$, 
and $ \Op K_{\rm rad}$ its one-dimensional radial counterpart that acts along the Radon-domain variable $t$. These %filtering 
operators are characterized by their frequency response 
$\widehat K(\bw)=c_d\|\bw\|^{d-1}$ and
$\widehat K_{\rm rad}(\omega)=c_d |\omega|^{d-1}$, respectively. \end{theorem}

It should be noted, however, that the operator $\Op R: \Spc S(\R^d) \to \Spc S(\R \times \mathbb{S}^{d-1})$ is not surjective, which means that not very hyper-spherical test function $\phi \in \Spc S(\R \times \mathbb{S}^{d-1})$ can be written as $\phi =\Op R\{\varphi\}$ with $\varphi \in \Spc S(\R^d)$. A necessary condition is that
$\phi$ be even, but this is not sufficient \cite{Gelfand1966,Ludwig1966,Helgason2011}. The good news, however,
is that the range of $\Op R$  on $\Spc S(\R^d)$ is closed
in $\Spc S_{\rm Rad}=\Op R\big(\Spc S(\R^d) \big)$ equipped with the nuclear topology of $\Spc S(\R \times \mathbb{S}^{d-1})$ \cite[p. 60]{Helgason2011}. Since the underlying spaces are both Fréchet, this translates into the transform $\varphi \mapsto \Op R\{\varphi\}$ being a homeomorphism of $\Spc S(\R^d)$ onto $\Spc S_{\rm Rad}$. 
%Consequently, we can define
%the continuous projection operator
%\begin{align}
%\Op P_{\rm Rad} &: \Spc S(\R \times \mathbb{S}^{d-1}) \to \Spc S_{\rm 
%Rad},% =\Op R\big(\Spc S(\R^d)\big)\label{Eq:PRad0},
%%\\
%%\Op P_{\rm Rad}^\ast&: \Spc S'(\R \times \mathbb{S}^{d-1}) \to \Spc S'_{\rm Rad} =\Op K_{\rm rad}\Op R\big(\Spc S'(\R^d)\big).
%%\label{Eq:PRad2} 
%\nonumber \end{align}
%which allows us to decompose the space of hyper-spherical test functions as $\Spc S(\R \times \mathbb{S}^{d-1})=\Spc S_{\rm Rad} \oplus \Spc S_{\rm Rad^\perp}$
%where $\Spc S_{\rm Rad}=\Op R\big(\Spc S(\R^d)\big)=\Op P_{\rm Rad}\big(\Spc S(\R \times \mathbb{S}^{d-1})\big)$ and $\Spc S_{\rm Rad^\perp}=(\Identity-\Op P_{\rm Rad})\big(\Spc S(\R \times \mathbb{S}^{d-1})\big)$.
%With the help of this operator, %, which factors out the null space of $\Op R^\ast$, 
%we can refine
%the statement on the invertibility of $\Op R$ in Theorem \ref{Theo:RadonS0} as follows: 
\begin{corollary}
\label{Coro:Sinvert}
The operator $\Op R: \Spc S(\R^d) \to \Spc S_{\rm Rad}$ is a continuous bijection, with a continuous inverse given by $ \Op R^{-1}=(\Op R^\ast\Op K_{\rm rad}): \Spc S_{\rm Rad}\to \Spc S(\R^d)$. 
\end{corollary}

\subsection{Distributional Extension}
\label{Sec:RadonDistributions}
To extend the framework to distributions, one proceeds by duality. By invoking the property that 
 $\Op R^\ast \Op K_{\rm rad}\Op R=\Identity$
on $\Spc S(\R^d)$, we make the manipulation
\begin{align}
\forall \varphi \in \Spc S(\R^d)\quad 
\langle f,\varphi \rangle&=\langle f,\Op R^\ast \Op K_{\rm rad}\Op R\{\varphi\} \rangle\nonumber\\
&= \langle \Op R\{f\},\Op K_{\rm rad}\Op R\{ \varphi\}\rangle_{\rm Rad}
=\langle \Op R\{f\}, \phi\rangle_{\rm Rad},\label{Eq:dualKR}
\end{align}
with $\phi=\Op K_{\rm rad}\Op R\{\varphi\} \in \Op K_{\rm rad}\Op R\big(\Spc S(\R^d)\big)\subset  \Spc S(\R \times   \mathbb{S}^{d-1})$  and $\varphi=\Op R^\ast\{\phi\}$. Relation \eqref{Eq:dualKR}, which is valid in the classical sense for $f \in L_1(\R^d)$, is then used as definition to extend the scope of $\Op R$ for $f\in \Spc S'(\R^d)$. 
%The intrinsic limitation, however, is that this gives a consistent definition of $\Op R\{f\}$ for $f \in \Spc S'(\R^d)$ {\em only if} the range of the operator is limited to $\big( \Op K_{\rm rad}\Op R\big( \Spc S(\R^d)\big)'$: the subspace of $\Spc S'(\R \times   \mathbb{S}^{d-1})$
%of Radon-compatible distributions \cite{Ludwig1966}. The same constraint applies for the definition of
%the filtered Radon transform $\Op K_{\rm rad}\Op R: \Spc S(\R^d) \to \Spc S_{\rm Rad}'$
%where the proper distributional range $\Spc S_{\rm Rad}'$ is the topological dual of $\Spc S_{\rm Rad}=\Op R\big(\Spc S(\R^d)\big)$ in Corollary \ref{Coro:Sinvert}.

\begin{definition}
\label{Def:GeneralizedRadon}
The distribution $g=\Op R\{f\} \in \big( \Op K_{\rm rad}\Op R\big( \Spc S)\big)'$ 
is the (formal) Radon transform of $f \in \Spc S'(\R^d)$ if
\begin{align}
\forall \phi \in \Op K_{\rm rad}\Op R \big( \Spc S(\R^d)\big):\quad \langle g,\phi \rangle_{\rm Rad}
=\langle f, \Op R^\ast \{\phi\} \rangle.
\label{Eq:RDist}
\end{align}
Likewise, $\tilde g=\Op K_{\rm rad}\Op R\{f\} \in \Spc S_{\rm Rad}'$ is the (formal) filtered projection of $f \in \Spc S'(\R^d)$ if 
\begin{align}
\forall \phi \in \Spc S_{\rm Rad}: \quad \langle \tilde g,\phi \rangle_{\rm Rad}=\langle f, \Op R^\ast\Op K_{\rm rad} \{\phi\} \rangle.
\label{Eq:KRDist}
\end{align}
Finally, $f=\Op R^\ast \{g\} \in \Spc S'(\R^d)$ is the backprojection of $g%=\Op R\{ f \}
\in \Spc S'(\R \times   \mathbb{S}^{d-1})$ %the tempered distribution 
if %and only if
\begin{align}
%\langle \Op R^\ast\{ g\}, \varphi \rangle=\langle g, \Op R \{\varphi\}\rangle\\
\forall \varphi \in \Spc S(\R^d): \quad \langle \Op R^\ast\{ g\}, \varphi \rangle
%=\int_{\R^d}   \Op R^\ast \{g\} \varphi \dint \V x= \int_{\mathbb{S}^{d-1}} \int_{\R} g \;\Op R\{ \varphi\}  \;\dint t \dint \V \xi
=\langle g, \Op R \{\varphi\}\rangle_{\rm Rad}. \label{Eq:RadjDis}
\end{align}
\end{definition}

%While \eqref{Eq:RadjDis} yields a unique definition of the linear operator $\Op R^\ast: \Spc S'(\R \times   \mathbb{S}^{d-1}) \to \Spc S'(\R^d)$, this is not so for \eqref{Eq:KRDist} (resp., \eqref{Eq:RDist}) taken on its own because of the additional condition $\tilde g \in \Spc S_{\rm Rad}'$ (resp., $g\in \big( \Op K_{\rm rad}\Op R\big( \Spc S)\big)'$), which is not straightforward to check in practice. To ensure that the latter membership condition is satisfied, we shall rely on the projection operator $\Op P_{\rm Rad}^\ast: \Op K_{\rm rad}\Op R\Op R^\ast: \Spc S'(\R \times   \mathbb{S}^{d-1})\to \Spc S'_{\rm Rad}$  \cite{Unser2022_Ridges}, which is the adjoint of $\Op P_{\rm Rad}=\Op R\Op R^\ast\Op K_{\rm rad}: 
%\Spc S_{\rm Rad} \to \Spc S_{\rm Rad} \embedC \Spc S(\R \times   \mathbb{S}^{d-1})
%$ in Corollary \ref{Coro:Sinvert}.

While \eqref{Eq:RadjDis} identifies $\Op R^\ast\{g\}$ as a single, unique distribution in $\Spc S'(\R^d)$, this is not so for \eqref{Eq:KRDist} (resp., \eqref{Eq:RDist}), as
the members of $\Spc S'_{\rm Rad}$ (resp., of $\big( \Op K_{\rm rad}\Op R\big( \Spc S)\big)'$) are equivalence classes in $\Spc S'(\R \times   \mathbb{S}^{d-1})$.
%---hence, the term `formal''.
To make this explicit, we take advantage of the equivalence $\Op R^\ast\{g\}=0 \Leftrightarrow   \langle g, \phi  \rangle_{\rm Rad}=0$ to identity the null space of the backprojection operator as being
\begin{align}
\Spc N_{\Op R^\ast}=\{g \in \Spc S'(\R \times \mathbb{S}^{d-1}): %\Op R^\ast\{g\}=0\  \Leftrightarrow \ 
 \langle g, \phi  \rangle_{\rm Rad}=0, \forall \phi \in \Spc S_{\rm Rad}\},
\end{align}
which is a closed subspace of  $\Spc S'(\R \times \mathbb{S}^{d-1})$. It is then possible to describe $\Spc S'_{\rm Rad}$ %(the distributional range of $\Op K_{\rm rad}\Op R$) 
as the abstract quotient space $\Spc S'(\R \times \mathbb{S}^{d-1})/\Spc N_{\Op R^\ast}$. In other words, if we find a hyper-spherical distribution $g_0\in \Spc S'(\R \times \mathbb{S}^{d-1})$ such that
\eqref{Eq:KRDist} is met for a given $f \in \Spc S'(\R^d)$, then, strictly speaking, $\Op K_{\rm rad}\Op R\{f\} \in \Spc S'_{\rm Rad}$ is the equivalence class (or coset) given by
\begin{align}
\label{Eq:FProjEquivalenceClass}
\Op K_{\rm rad}\Op R\{f\}=[g_0]=\{g_0 + h: h \in \Spc N_{\Op R^\ast}\}.
\end{align}
Since $[g_0]=[g]$ for any $g\in\Op K_{\rm rad}\Op R\{f\}$, we refer to the members of $\Op K_{\rm rad}\Op R\{f\}$ as ``formal'' filtered projections of $f$ to remind us of this lack of unicity.

Based on those definitions, one obtains the classical result on the invertibility of the (filtered) Radon transform on $\Spc S'(\R^d)$ \cite{Ludwig1966}, which is the dual of 
 Corollary \ref{Coro:Sinvert}.

\begin{theorem}[Invertibility of the Radon transform on $\Spc S'(\R^d)$]
\label{Theo:InvertRadonDist} It holds that
$\Op R^\ast \Op K_{\rm rad} \Op R%=\Op R^\ast (\Op R\, \Op K)
=\Identity$ on $\Spc S'(\R^d)$. 
More precisely, the filtered Radon transform $\Op K_{\rm rad}\Op R: \Spc S'(\R^d) \to \Spc S'_{\rm Rad}$ %=\Op P^\ast_{\rm Rad}\big(\Spc S'(\R \times   \mathbb{S}^{d-1})\big)$ 
is a continuous bijection, with a continuous inverse given by %, which yields a unique distributional definition of $g=\Op K_{\rm rad}\Op R\{ f\}=\Op P^\ast_{\rm Rad}(\Op K_{\rm rad}\Op R)\{f\}$ 
$(\Op K_{\rm rad}\Op R)^{-1}=\Op R^\ast: \Spc S'_{\rm Rad} \to 
\Spc S'(\R^d)$.
\end{theorem}

To illustrate the fact that \eqref{Eq:KRDist} does not identify a single distribution, 
we consider the Dirac ridge $\delta(\V \xi_0\V x - t_0) \in \Spc S'(\R^d)$ and refer to the definition \eqref{Eq:Radon2} of the Radon transform to deduce that, for all $\phi=\Op R\{\varphi\} \in \Spc S_{\rm Rad}$ with $\varphi \in \Spc S(\R^d)$,
\begin{align*}
\langle \delta(\V \xi_0^\Top\cdot - t_0),\Op R^\ast \Op K_{\rm rad}\{\phi\}\rangle&=\langle \delta(\V \xi_0^\Top\cdot - t_0), \overbrace{\Op R^\ast\Op K_{\rm rad}\Op R}^{\Identity}\{\varphi\}\rangle %=\langle \delta(\V \xi_0^\Top\cdot - t_0), \varphi\}\rangle
\\
&= \int_{\R^d}\delta(\V \xi_0^\Top\V x-t_0) \varphi(\V x)  \dint \V x=\Op R\{\varphi\}(-t_0,-\V \xi_0)\\
&=\langle \delta\big(\cdot+(t_0,\V \xi_0)\big),\Op R\{\varphi\}\rangle_{\rm Rad}=\langle \delta\big(\cdot+(t_0,\V \xi_0)\big),\phi\rangle_{\rm Rad},
\end{align*}
which shows that $\delta\big(\cdot+\V z_0\big)$ with $\V z_0=(t_0,\V \xi_0)$ is a formal filtered projection of $\delta(\V \xi_0^\Top\V x - t_0)$. Moreover, since 
$\delta(\V \xi_0^\Top\V x - t_0)=\delta(-\V \xi_0^\Top\V x +t_0)$, the same holds true for $\delta(\cdot-\V z_0)$, as well as for 
$\delta_{{\rm Rad},\V z_0}\eqdef\frac{1}{2} \big(\delta(\cdot-\V z_0)+\delta(\cdot+\V z_0)\big)$, which has the advantage of being symmetric. While the general solution in $\Spc S'(\R \times \mathbb{S}^{d-1})$ is $\Op K_{\rm rad}\Op R\{\delta(\V \xi_0^\Top\cdot - t_0)\}=[\delta\big(\cdot\pm\V z_0\big)]$, we shall see that there also exists a way to specify a representer that is unique (i.e., $\delta_{{\rm Rad},\V z_0})$ by restricting the
range of $\Op K_{\rm rad}\Op R$ to a suitable space of measures.

The distributional extension of the Radon transform inherits most of the properties of the ``classical'' operator defined in \eqref{Eq:Radon2}.
%  (see list in Theorem \ref{Theo:RadonS}).
Of special relevance to us is the quasi-commutativity of $\Op R$ with convolution, also known as the {\em intertwining property}. Specifically, let $h,f \in \Spc S'(\R^d)$ be two distributions whose convolution $h \ast f$ is well-defined in $\Spc S'(\R^d)$. Then,
\begin{align}
\Op R\{ h \ast f\}=\Op R\{ h\} \astrad  \Op R\{ f\}\end{align}
where the symbol  ``$\astrad$'' denotes the 1D convolution along the radial variable $t \in\R$ with $(u \astrad g)(t,\V \xi)=\langle u(\cdot,\V \xi),g(t-\cdot,\V \xi) \rangle$. In particular, when $h=\Lop\{\delta\}$ is the (isotropic) impulse response of an LSI operator whose frequency response $\widehat L(\bw)=\widehat L_{\rm rad}(\|\bw\|)$ is purely radial, we get that
\begin{align}
\Op R\{ h \ast f\}=\Op R\Lop\{f\}=\Lop_{\rm rad}\Op R\{f\},
\end{align}
where $\Lop_{\rm rad}$ is the radial convolution operator whose 1D frequency response is $\widehat L_{\rm rad}$. Likewise, by duality, for $g \in \Spc S'(\R \times   \mathbb{S}^{d-1})$ we have that
\begin{align}
\Lop \Op R^\ast\{g\}=\Op R^\ast \Lop_{\rm rad}\{g\},
\end{align}
under the implicit assumption that $\Lop \{\Op R^\ast g\}$ and $\Lop_{\rm rad}\{g\}$ are well-defined distributions. 
By taking inspiration from Theorem 1, we can then use these relations for $\Lop=\Op K=(\Op R^\ast \Op R)^{-1}$ to show that $\Op R^\ast\Op K_{\rm rad} \Op R\{f\}=\Op R^\ast \Op R\Op K\{f\}=\Op K\Op R^\ast\Op R\{f\}=f$ for a broad class of distributions. The first form is valid for all $f \in \Spc S'(\R^d)$ (Theorem \ref{Theo:InvertRadonDist}), but there is a slight restriction with the
second form (resp., third form), which requires that $\Op K\{f\}$ \big(resp., $\Op K\{g\}$ with $g=\Op R^\ast\Op R\{f\} \in \Spc S'(\R^d)$\big) be well-defined in $\Spc S'(\R^d)$.  While the latter condition is always met when $d$ is odd, it may fail\footnote{For $d=2n$ even, $\widehat K(\bw) \propto \|\bw\|^{2n-1}$ which is $C^\infty$ everywhere, except at the origin where it is only $C^{2n-2}$. This means that $\Op K$ can properly handle (and annihilate) polynomials only up to degree $2n-2$.} in even dimensions with distributions (e.g., polynomials) whose Fourier transform is singular at the origin.
The good news for our regularization framework is that these problematic distributions are either excluded from the native space or annihilated by $\Lop$, so that it is legitimate to write
that $\Lop_{\rm R}=\Op K_{\rm rad}\Op R\Lop=\Op R\Op K\Lop$, where the second form has the advantage that $\Op K$ and $\Lop$ can be pooled into a single augmented operator $(\Op K\Op L)$. An alternative form is $\Lop_{\rm R}=\Op Q_{\rm rad}\Op R$, where $\Op Q_{\rm rad}=\Op K_{\rm rad}\Op L_{\rm rad}$ is the radial Radon-domain operator whose frequency response is $\widehat Q_{\rm rad}(\omega)=c_d |\omega|^{d-1} \widehat L_{\rm rad}(\omega)$.
\subsection{Radon-Compatible Banach Spaces}
\label{Sec:ComplementedSpaces}
Our formulation requires the identification of Radon-domain Banach spaces over which the backprojection operator $\Op R^\ast$ is invertible. This is a non-trivial point because the extended operator $\Op R^\ast: \Spc S'(\R \times   \mathbb{S}^{d-1}) \to \Spc S'(\R^d)$ in Definition \ref{Def:GeneralizedRadon} is not injective. In fact, it has the highly non-trivial null space
${\rm ker}(\Op R^\ast)=\Spc S^\perp_{\rm Rad}$, which is a superset of the odd Radon-domain distributions \cite{Gelfand1966}. Yet, $\Op R^\ast$ is invertible on $\Spc S'_{\rm Rad}$ and surjective on $\Spc S'(\R^d)$ (Theorem \ref{Theo:InvertRadonDist}). 
%These properties translate into the direct sum decomposition 
%\begin{align}
%\label{Eq:DirectSum}
%\Spc S'(\R \times   \mathbb{S}^{d-1})=\Spc S'_{\rm Rad} \oplus \Spc S'_{\rm Rad^\perp}
%\end{align} with
%$\Spc S'_{\rm Rad}=\Op P^\ast_{\rm R}\big(\Spc S'(\R \times   \mathbb{S}^{d-1})\big)$ and
%$\Spc S'_{\rm Rad^\perp}=(\Identity- \Op P^\ast_{\rm Rad})\big(\Spc S'(\R \times   \mathbb{S}^{d-1})\big)$, which is the dual counterpart of
%$\Spc S(\R \times \mathbb{S}^{d-1})=\Spc S_{\rm Rad} \oplus \Spc S_{\rm Rad^\perp}$ discussed in Section \ref{Sec:RadonSchwartz}. The functional relevance of \eqref{Eq:DirectSum} is that
%$\Spc S'_{\rm Rad^\perp}$ (the annihilator of $\Spc S_{\rm Rad}$) is precisely the (distributional) null space of $\Op R^\ast$, while $\Spc S'_{\rm Rad}$ delineates the subset of distributions that are in the range of $\Op K_{\rm rad}\Op R$. The latter is the largest possible subspace of Radon-domain distributions over which the inversion (or projection) formula $\Op K_{\rm rad}\Op R\Op R^\ast\{g\}=\Op P^\ast_{\rm Rad}\{g\}=g$ holds. 

To ensure invertibility, we therefore need to restrict ourselves to Banach spaces that are embedded in $\Spc S'_{\rm Rad}$.
%The main outcome of Section \ref{Sec:Radonprojector} is that $\Spc S(\R \times   \mathbb{S}^{d-1})=\Spc S_{\rm Rad} \oplus \Spc S_{\rm Rad^\perp}$ and $\Spc S'(\R \times   \mathbb{S}^{d-1})=\Spc S'_{\rm Rad} \oplus \Spc S'_{\rm Rad^\perp}$ where $\Spc S'_{\rm Rad^\perp}$ is the distributional range of $\Op K_{\rm rad}\Op R$ and $\Spc S'_{\rm Rad^\perp}$ is the null space of $\Op R^\ast$. 
To identify such objects, we consider a generic Banach space 
%$\Spc X=\overline{(\Spc S(\R \times   \mathbb{S}^{d-1}),\|\cdot\|_{\Spc X})}$.
$\Spc X=(\Spc X, \|\cdot\|_{\Spc X})$ such that
$\Spc S(\R \times   \mathbb{S}^{d-1}) \embedD \Spc X \embedD \Spc S'(\R \times   \mathbb{S}^{d-1})$.
This dense-embedding hypothesis has several implications:
\begin{enumerate}
\item The space $\Spc X$ is the completion of $\Spc S(\R \times   \mathbb{S}^{d-1})$ in the $\|\cdot\|_{\Spc X}$ norm; i.e., 
\begin{align}
\label{Eq:Xspace}
\Spc X=\overline{\big(\Spc S(\R \times   \mathbb{S}^{d-1}), \|\cdot\|_{\Spc X}\big)}.
\end{align}
\item The dual space $\Spc X'\embedC \Spc S'(\R \times   \mathbb{S}^{d-1})$ is equipped with the norm
\begin{align}
\label{Eq:Dualnorm}
\|g\|_{\Spc X'}=\sup_{\phi \in \Spc X:\; \|\phi\|_{\Spc X}\le 1} \langle g, \phi \rangle
=\sup_{\phi \in \Spc S(\R \times   \mathbb{S}^{d-1}):\; \|\phi\|_{\Spc X}\le 1} \langle g, \phi \rangle,
\end{align}
where the restriction of $\phi \in \Spc S(\R \times   \mathbb{S}^{d-1})$ on the right-hand side of \eqref{Eq:Dualnorm} is justified by the denseness of
$\Spc S(\R \times   \mathbb{S}^{d-1})$ in $\Spc X$.
\item The definition of $\|g\|_{\Spc X'}$ given by the right-hand side of \eqref{Eq:Dualnorm} is valid for 
any distribution $g\in \Spc S'(\R \times   \mathbb{S}^{d-1})$ with $\|g\|_{\Spc X'}=\infty$ for $g \notin \Spc X'$. Accordingly, we can specify the topological dual of $\Spc X'$ as
\begin{align}
\label{Eq:DualX}
\Spc X'=\big\{g \in  \Spc S'(\R \times   \mathbb{S}^{d-1}): \|g\|_{\Spc X'}< \infty\big\}.
\end{align}
\end{enumerate}

Likewise, based on the pair $(\Spc S_{\rm Rad},\Spc S'_{\rm Rad})$, we specify the Radon-compatible Banach subspaces
\begin{align}
\label{Eq:Xrad}
\Spc X_{\rm Rad}&=\overline{(\Spc S_{\rm Rad},\|\cdot\|_{\Spc X})}\\
%\Spc X_{\rm Rad^\perp}&=\overline{(\Spc S_{\rm Rad^\perp},\|\cdot\|_{\Spc X})}\\
\Spc X'_{\rm Rad}&=\big(\Spc X_{\rm Rad}\big)'=\big\{g \in  \Spc S'_{\rm Rad}: \|g\|_{\Spc X'_{\rm Rad}}< \infty\big\}
\label{Eq:XradP}
%\\ \Spc X'_{\rm Rad^\perp}&=\big(\Spc X_{\rm Rad^\perp}\big)'=\big\{g \in  \Spc S'_{\rm Rad^\perp}: \|g\|_{\Spc X'_{\rm Rad^\perp}}< \infty\big\}
\end{align}
where the underlying dual norms have a definition that is analogous to \eqref{Eq:Dualnorm} with
$\Spc S_{\rm Rad}$ and $\Spc X_{\rm Rad}$ substituting for $\Spc S(\R \times   \mathbb{S}^{d-1})$ and $\Spc X$.
%(resp., by $\Spc S_{\rm Rad^\perp}$ and $\Spc X_{\rm Rad^\perp}$).

%The next result presents a sufficient condition for the dual pair of Banach spaces $(\Spc X_{\rm Rad},\Spc X'_{\rm Rad})$ to meet our requirements.

\begin{theorem}[adapted from \cite{Unser2022_Ridges}]
\label{Theo:Complementedspaces}
Let $(\Spc X_{\rm Rad},\Spc X'_{\rm Rad})$ be the dual pair of spaces specified by \eqref{Eq:Xrad} and \eqref{Eq:XradP}. Then,
\begin{enumerate}
\item the map $\ \Op R^\ast\Op K_{\rm rad}: \Spc X_{\rm Rad} \to  \Spc Y=\Op R^\ast\Op K_{\rm rad}\big(\Spc X_{\rm Rad}\big)$ is an isometric bijection, with $\Op R\Op R^\ast\Op K_{\rm rad}=\Identity$ on $\Spc X_{\rm Rad}$;
\item the map $\Op R^\ast: \Spc X'_{\rm Rad} \to  \Spc Y'=\Op R^\ast\big(\Spc X'_{\rm Rad}\big)$ is an isometric bijection,
with $\Op K_{\rm rad}\Op R\Op R^\ast=\Identity$ on $\Spc X'_{\rm Rad}$.
% or, equivalently, $\Spc X_{\rm Rad}=\Op R(\Spc Y)$.
\end{enumerate}
Moreover, if there exists a complementary Banach space $\Spc X^{\rm c}_{\rm Rad}$ such that $\Spc X=\Spc X_{\rm Rad}\oplus\Spc X^{\rm c}_{\rm Rad}$, then 
$\Spc X'=\Spc X'_{\rm Rad}\oplus(\Spc X^{\rm c}_{\rm Rad})'$ where $(\Spc X^{\rm c}_{\rm Rad})'$ can be identified as the null space of the backprojection operator $\Op R^\ast: \Spc X'\to\Spc Y' \embedC \Spc S(\R^d)$.
%\item $\Spc X'_{\rm Rad^\perp}=(\Identity-\Op P^\ast_{\rm Rad})\big(\Spc X'\big)=\{g \in 
%\Spc X': \langle g, \phi\rangle_{\rm Rad}=0, \ \ \forall \phi \in \Spc S_{\rm Rad}\}$
\end{theorem}
The prototypical examples where those properties are met are $(\Spc X, \Spc X')=\big(L_p(\R \times   \mathbb{S}^{d-1}),L_q(\R \times   \mathbb{S}^{d-1})\big)$ with $p\in[1,\infty)$ and $q=p/(p-1)$ (conjugate exponent), as well as
$(\Spc X, \Spc X')=\big(C_0(\R \times   \mathbb{S}^{d-1}),\Spc M(\R \times   \mathbb{S}^{d-1})\big)$. In fact, those hyper-spherical spaces have the convenient feature of admitting a decomposition in their even and odd components.
\begin{lemma} 
\label{Lem:Even}
Let $\Spc Z =\R \times   \mathbb{S}^{d-1}$.
Then, for $\Spc X=L_p(\Spc Z)$ with $p \in [1,\infty)$ 
and $\Spc X=C_0(\Spc Z)$ for $p=\infty$, we have that $\Spc X=\Spc X_{\rm Rad}\oplus\Spc X^{\rm c}_{\rm Rad}$ where
\begin{align}
\Spc X_{\rm Rad}&=\Spc X_{\rm even}=\{g\in \Spc X: g(\V z)=g(-\V z), \forall \V z \in \Spc Z\}=\Spc X(\mathbb{P}^d)\\
\Spc X^{\rm c}_{\rm Rad}&=\Spc X_{\rm odd}=\{g\in \Spc X: g(\V z)=-g(-\V z), \forall \V z \in \Spc Z\}.
\end{align}

\end{lemma}
\begin{proof} To establish this result directly is tricky because the characterization of $\Spc S_{\rm Rad}$ involves some general moment conditions \cite{Gelfand1966,Helgason2011,Ludwig1966}. Instead, we consider the smaller space of even Radon-domain Lizorkin test functions \cite{Kostadinova2014} described by
\begin{align}
\Spc S_{\rm Liz, Rad}=\{\phi \in \Spc S_{\rm even}(\Spc Z): \int_\R t^k\phi(t,\V \xi)\dint t=0,\forall \V \xi \in \mathbb{S}^{d-1}, k \in \N\},
\end{align}
which is such that $\Spc S_{\rm Liz, Rad}\subset \Spc S_{\rm Rad} \subset \Spc S_{\rm even}(\Spc Z)$. We then invoke a general result by Samko \cite{Samko1982denseness} that implies that
$\overline{(\Spc S_{\rm Liz, Rad},\|\cdot\|_{L_p})}=L_{p,{\rm even}}(\Spc Z)\supset {(\Spc S_{\rm Rad},\|\cdot\|_{L_p})}$ for $p\in[1,\infty)$ and $\overline{(\Spc S_{\rm Liz, Rad},\|\cdot\|_{L_\infty})}=C_{0,{\rm even}}(\Spc Z)$ otherwise \cite{Neumayer2022}. The claim then follows from the observation that
$L_{p}(\Spc Z)=L_{p,{\rm even}}(\Spc Z)\oplus L_{p,{\rm odd}}(\Spc Z)$ with $L_{p,{\rm even}}(\Spc Z)=\overline{(\Spc S_{\rm Rad},\|\cdot\|_{L_p})}$ 
(because the completion is unique) and suitable adaptation for $p=\infty$.
\end{proof}
Correspondingly, we get that $\Spc X'_{\rm Rad}=\Op P_{\rm even}(\Spc X')=\Spc X'_{\rm even}$ and 
$(\Spc X^{\rm c}_{\rm Rad})'=(\Identity-\Op P_{\rm even})(\Spc X')=\Spc X'_{\rm odd}$, with the cases of greatest interest to us being $\Spc M_{\rm Rad}=\Spc M_{\rm even}(\R \times   \mathbb{S}^{d-1})=\Spc M(\mathbb{P}^{d})$ and
$L_{2,\rm Rad}=L_{2,{\rm even}}(\R \times   \mathbb{S}^{d-1})=L_2(\mathbb{P}^{d})$.

\subsection{Specific Radon Transforms}
\label{Sec:UsefulRadon}
The Fourier-slice theorem expressed by \eqref{Eq:CentralSliceTheo} remains valid for tempered distributions \cite{Ramm2020} and therefore also yields a %unique 
%(Fourier-based) 
characterization of $\Op R\{ f\}$ that is compatible 
%with Definition \ref{Def:GeneralizedRadon} \cite{Ramm2020}, as well as 
 with the %restricted 
 Banach framework of Theorem \ref{Theo:Complementedspaces}. It is especially helpful %easy to deploy 
when the underlying function $\rho_{\rm iso}$ is isotropic with a known radial frequency profile $\widehat \rho_{\rm rad}$ such that $\Fourier\{ \rho_{\rm iso}\}(\bw)=\widehat \rho_{\rm rad}(\|\bw\|)$.
%; that is, 
%$\rho_{\rm iso}(\V x)=\rho(\|\bx\|)$ where $\rho: \R_{\ge0} \to \R$ is the radial profile of $\rho_{\rm iso}$. 
%The frequency domain counterpart of this characterization is 
%$\widehat \rho_{\rm iso}(\bw)=\widehat \rho_{\rm rad}(\|\bw\|)$ where the radial frequency profile 
%can be computed as
%$$\widehat \rho_{\rm rad}(\omega)= \frac{(2\pi)^{d/2}}{|\omega|^{d/2-1} }\int_{0}^{+\infty} \rho(t) t^{d/2-1} J_{d/2-1}(\omega t) t \dint t$$
%with  $J_{\nu}$ the Bessel function of the first kind of order $\nu$.
\begin{proposition} [Radon transform of isotropic distributions] 
\label{Prop:IsoRad}
Let $\rho_{\rm iso}$ be an isotropic distribution whose radial frequency profile is $\widehat \rho_{\rm rad}: \R \to \R$. Then,
\begin{align}
\Op R\{\rho_{\rm iso}(\cdot -\V x_0)\}(t,\V \xi\}&=\rho_{\rm rad}(t-\V \xi^\Top \V x_0)%=(\tilde \rho_{\rm rad} \ast q_d)(t)
\\
\Op K_{\rm rad}\Op R\{ \rho_{\rm iso}(\cdot -\V x_0)\}(t,\V \xi\}&=\nu_{\rm rad}(t-\V \xi^\Top \V x_0)\end{align} 
with $\rho_{\rm rad}(t)=\Fourier^{-1}\{\widehat \rho_{\rm rad}(\omega)\}(t)$ and 
$\nu_{\rm rad}(t)=\tfrac{1}{2(2\pi)^{d-1}}\Fourier^{-1}\{|\omega|^{d-1}\widehat\rho_{\rm rad}(\omega)\}(t)$.
\end{proposition}

%Let us note that both $\rho_{\rm rad}(t)$ and  $\tilde\rho_{\rm rad}(t)$, as inverse Fourier transform of a real-valued function, are symmetric, which is consistent with the symmetry of the Radon transform and its filtered version. 

%\begin{proposition} [Radon transform of derivatives of isotropic distributions] 
%\label{Prop:IsoRad}
%Let $\rho_{\rm iso}$ be an isotropic distribution whose radial frequency profile is $\widehat \rho_{\rm rad}(\omega)$. Then,
%\begin{align}
%\Op R\{\partial^{\V m}\rho_{\rm iso}\}(t,\V \xi\}&=\V \xi^{\V m} \Op D^{|\V m|}\{\rho_{\rm rad}\}(t)
%\end{align}
%with $\rho_{\rm rad}(t)=\Fourier^{-1}\{\widehat \rho_{\rm rad}(\omega)\}(t)$ and $\Op D^m=\frac{\dint^m}{\dint t^m}$.
%\end{proposition}
%\begin{proof}By setting $\bw =\omega \V \xi$ in the Fourier transform of $\partial^{\V m}\rho_{\rm iso}$, we get
%\begin{align}
%\widehat{\partial^{\V m}\rho_{\rm iso}}(\omega \V \xi)%=(\jj \bw)^{\V m}\widehat \rho_{\rm rad}(\|\bw\|)&
%=(\jj \omega\V \xi)^{\V m}\widehat \rho_{\rm rad}(\omega)=\V \xi^{\V m} (\jj \omega)^{|\V m|} \widehat \rho_{\rm rad}(\omega).
%\end{align}
%The result is obtained by taking the inverse 1D Fourier transform (Fourier slice theorem).
%\end{proof}

The other important building blocks for representing functions are ridges. Specifically, a ridge is a multidimensional function 
\begin{align}
r_{\V \xi_0}:\R^d \to \R: \V x \mapsto r(\V \xi_0^\Top\V x)
\end{align}
that is characterized by a profile $r: \R \to \R$ and a direction $\V \xi_0 \in \mathbb{S}^{d-1}$. In effect, $r_{\V \xi_0}$ varies along the axis specified by $\V \xi_0$
and is constant within any hyperplane perpendicular to $\V \xi_0$. The connection between ridges and the Radon transform is given by the %so-called 
{\em ridge identity}
\begin{align}
\label{Eq:Ridges}
\forall \varphi \in \Spc S(\R^d): \quad \langle r_{\V \xi_0}, \varphi\rangle=\langle r, \Op R\{\varphi\} (\cdot,\V \xi_0)\rangle,
\end{align}
where the right-hand side duality product (1D) 
 is well-defined for any $r \in \Spc S'(\R)$ because $\Op R\{\varphi\} (\cdot,\V \xi_0) \in \Spc S(\R)$ (by Theorem \ref{Theo:RadonS0}).
When the profile $r: \R \to \R$ is locally integrable, \eqref{Eq:Ridges} is a simple consequence of Fubini's theorem. For more general distributional profiles $r \in \Spc S'(\R)$, we use the ridge identity as definition, which then leads to the following characterization \cite{Unser2022_Ridges}. 
\begin{theorem}[Filtered Radon transform of a ridge]
\label{Theo:Rad1DprofilesContinuous}
The filtered Radon transform of the (generalized) ridge $r_{\V \xi_0}$ %=r(\V \xi_0^\Top\V x)$ 
with profile $r\in \Spc S'(\R^d)$ and direction $\V \xi_0 \in \mathbb{S}^{d-1}$ is given by
\begin{align}
\label{Eq:RidgeGeneral}
\Op K_{\rm rad}\Op R \{r_{\V \xi_0} \}(t,\V \xi)=%\Op P^\ast_{\rm Rad}\{r(t)\delta(\V \xi-\V \xi_0)\}
[r(t)\delta(\V \xi-\V \xi_0)],
\end{align}
where $[r(t)\delta(\V \xi-\V \xi_0)]\in \Spc S_{\rm Rad}'$ is the equivalence class of distributions specified by \eqref{Eq:FProjEquivalenceClass}.
%where $\Op P_{\rm Rad}^\ast=\Op R\Op K_{\rm rad}\Op R^\ast: \Spc S'(\R \times \mathbb{S}^{d-1})\to \Spc S'_{\rm Rad}$.
If $r \in \Spc M(\R)$, then the latter has the unique, concrete representer %as
\begin{align}
\label{Eq:RidgeContinuous}
\Op K_{\rm rad}\Op R \{r_{\V \xi_0} \}(t,\V \xi)&=\frac{1}{2} \big(r(t)\delta(\V \xi-\V \xi_0) + r(-t)\delta(\V \xi+\V \xi_0)\big) % \in \Spc M(\R \times \mathbb{S}^{d-1})
\end{align}
in $\Spc M_{\rm Rad}=\Spc M_{\rm even}(\R \times   \mathbb{S}^{d-1})$. 
\end{theorem}
%Eq. \eqref{Eq:RidgeContinuous} is consistent with \eqref{Eq:RidgeGeneral}
%because the restriction of the projector on the space of Radon-domain measures (see explanation at the end of Section \ref{Sec:ComplementedSpaces}) is $\Op P^\ast_{\rm Rad}\vert_{\Spc M}=\Op P_{\rm even}$, which is the self-adjoint projector that extracts the even part of a function space \cite{Unser2022_Ridges}.
An important special case of \eqref{Eq:RidgeContinuous} is the Radon transform of a Dirac ridge: $\Op K_{\rm rad}\Op R \{\delta(\V \xi_0^\Top\cdot-t_0) \}=
\delta_{{\rm Rad},(t_0, \V \xi_0)}= \frac{1}{2} \big(\delta(\cdot-t_0)\delta(\cdot-\V \xi_0) + \delta(\cdot+t_0)\delta(\cdot+\V \xi_0)\big)$, which has already been mentioned in Section \ref{Sec:RadonDistributions} (see also \cite[Example 1]{Ongie2020b}). 
%The latter is an {\em even} Radon-domain Dirac distribution with its mass evenly split and localized at $(t,\V \xi)=\pm(t_0,\V \xi_0)$.

\section{Unifying Variational Formulation}
\label{Sec:VariationalFormulation}

\subsection{Representer Theorem for Radon-Domain Regularization}
From now on, we shall use the generic symbol $\Spc X$ to designate
the hyper-spherical Banach space $L_q(\R \times \mathbb{S}^{d-1})$ with $q=(1,\infty)$ or 
$C_0(\R \times \mathbb{S}^{d-1})$ for $q=\infty$, which fall into the category described by \eqref{Eq:Xspace} with $\|\cdot\|_{\Spc X}=\|\cdot\|_{L_q}$.

The formulation of Theorem \ref{Theo:RadonSplines} below requires the specification 
of a native space $\Spc X_{\Op L_{\rm R}}'(\R^d)$ that is tied to a Radon-domain norm $\|\cdot\|_{\Spc X'}$ and an admissible regularization operator $\Lop_{\rm R}$. 
To properly identify $\Spc X_{\Op L_{\rm R}}'$, we first need to restrict the dual pair $(\Spc X,\Spc X')$ to the range of the (filtered) Radon transform. This yields the Banach spaces
%---that is, within the range of the (filtered) Radon transform---
$(\Spc X_{\rm Rad},\Spc X'_{\rm Rad})$, as defined by \eqref{Eq:Xrad} and \eqref{Eq:XradP}, with the pairs of interest being
$(C_{0,{\rm Rad}},\Spc M_{0,{\rm Rad}})$ and $(L_{q,{\rm Rad}},L_{p,{\rm Rad}})$ with $\frac{1}{p}+\tfrac{1}{q}=1$ and $p\in[1,\infty)$.
%The interaction between those spaces and the various Radon-related operators is
%covered by Theorem \ref{Theo:Complementedspaces}.
Given some spline-admissible operator $\Lop$ (Definition \ref{Def:SplineAdm1} and \ref{Def:SplineAdNonTrivial}), we then define our regularization operator and its adjoint as
\begin{align*}
\Lop_{\rm R}&\eqdef\Op K _{\rm rad}\Op R \Op L: \Spc X'_{\Op L_{\rm R}}(\R^d) \to \Spc X'_{\rm Rad}%(\R \times \mathbb{S}^{d-1})
,\\
\Lop^\ast_{\rm R}&=\Op L^\ast \Op R^\ast\Op K _{\rm rad}: \Spc X_{\rm Rad}%(\R \times \mathbb{S}^{d-1}) 
\to \Spc X_{\Op L_{\rm R}}(\R^d)
\end{align*}
where $\Op R$ denotes the Radon transform and $\Op K _{\rm rad}$ is the (self-adjoint) filtering operator such that  $\Op K_{\rm rad}\Op R\Op R^\ast=\Identity$ on $\Spc X'_{\rm Rad}$ (Theorem \ref{Theo:Complementedspaces}).
%In essence, the native space $\Spc X'_{\Op L_{\rm R}}(\R^d)$, which is tied to the norm $\|\cdot\|_{\Spc X'}$ and the regularization operator $\Lop_{\rm R}$, is the largest space of functions $f:\R^d \to \R$ such that 
%\Lop_{\rm R} f \in \Spc X'_{\rm Rad}$. 
In order to establish isometries, one needs to be able to invert $\Lop$ (resp., $\Lop_{\rm R}$), which is feasible if one factors out the null space $\Spc P$, which is common to both. This motivates us to  define the directed inverse operators
\begin{align}
\Lop^{-1}_\Spc P&\eqdef 
(\Identity -\Proj_{\Spc P})\Op L^{-1}: \Op R^{\ast}\big(\Spc X'_{\rm Rad}
%(\R \times \mathbb{S}^{d-1})
\big) \to  \Spc X'_{\Op L_{\rm R}}(\R^d) \tag{Right inverse of $\Op L$} \nonumber\\
%\Lop^{-1\ast}_{\Spc P}&\eqdef
\Lop^{-1\ast}_\Spc P&=\Op L^{-1\ast} (\Identity -\Proj_{\Spc P'})
:\Spc X_{\Op L_{\rm R}}(\R^d) \to \Op K_{\rm rad}\Op R(\Spc X_{\rm Rad})
%(\R \times \mathbb{S}^{d-1})
 \tag{Left inverse of $\Op L^\ast$}\nonumber\\
\Lop_{\rm R}^{\dagger}&\eqdef  
\Lop^{-1}_\Spc P \Op R^\ast%\Op  K
%(\Identity -\Proj_{\Spc P})\Op L^{-1} \Op R^\ast
: \Spc X'_{\rm Rad} \to \Spc X'_{\Op L_{\rm R}}(\R^d)\\%(\R \times \mathbb{S}^{d-1})
\Lop_{\rm R}^{\ast\dagger}&= \Op R \Lop^{-1\ast}_\Spc P
%\Op L^{-1\ast} (\Identity -\Proj_{\Spc P'})
: \Spc X_{\Op L_{\rm R}}(\R^d)  \to \Spc X_{\rm Rad}%(\R \times \mathbb{S}^{d-1})
\end{align}
where the operators $\Lop_{\rm R}^{\dagger}$ and $\Lop_{\rm R}^{\ast\dagger}$ are generalized inverses\footnote{The precise properties of these inverse operators are stated in Theorem \ref{Theo:Native} and \ref{Theo:PredualNative} .} of $\Lop_{\rm R}$ and $\Lop_{\rm R}^{\ast}$, respectively. % in the sense that $\Lop_{\rm R}^{\ast}\Lop_{\rm R}^{\ast\dagger}\Lop_{\rm R}^{\ast}=\Lop_{\rm R}^{\ast}$. 
We now have all the ingredients to specify our native space as
\begin{align}
\Spc X'_{\Op L_{\rm R}}(\R^d)&=\Op L^{\dagger}_{\rm R}\big(\Spc X'_{\rm Rad}
%(\R \times \mathbb{S}^{d-1})
\big)\oplus \Spc P \nonumber\\
&=\{ f \in L_{\infty,-n_0}(\R^d): \|\Op L_{\rm R}\{f\}\|_{\Spc X'}+ \|\Proj_\Spc P\{f\}\|_{\Spc P}<\infty\}\nonumber\\
&=\{%f=
\Op L^{\dagger}_{\rm R}\{w\} + p_0:\quad  (w, p_0)\in \Spc X_{\rm Rad}'%(\R \times \mathbb{S}^{d-1}) 
\times \Spc P\},
\label{Eq:Nativespace}
\end{align}
which is isometrically isomorphic to $\Spc X_{\rm Rad}'%(\R \times \mathbb{S}^{d-1})
 \times \Spc P
$, as expressed by %the last line of 
\eqref{Eq:Nativespace}. The key property there is that 
$\Lop_{\rm R}\Op L^{\dagger}_{\rm R}%=\Op R \Op K \Op  L\Op L^{-1}_{\Spc P}\Op R^{\ast}
=\Identity$
on $ \Spc X_{\rm Rad}'
%(\R \times \mathbb{S}^{d-1})
$, while $\Lop_{\rm R}\{p_0\}=0$ for all $p_0\in \Spc P$ (Theorem  \ref{Theo:Native}). Moreover, $\Spc X'_{\Op L_{\rm R}}(\R^d)$ is the topological dual of the predual space
\begin{align}
%\Spc X_{\Op L_{\rm R}^{\ast\dagger}}(\R^d)=
\Spc X_{\Op L_{\rm R}}(\R^d)&=\Lop_{\rm R}^{\ast}\big(\Spc X_{\rm Rad}
%(\R \times \mathbb{S}^{d-1})
\big)\oplus \Spc P' %N_{\V p^\ast}
\nonumber\\
%&=
%\overline{\big(\Spc S_{\rm Rad},\|\varphi\|_{\Spc X_{\Lop_{\rm R}}}= \max(\|\Lop_{\rm R}^{\ast\dagger}\{\varphi\}\|_{\Spc X},  \|\Proj_{\Spc P'}\{\varphi\}\|_{\Spc P'})\big)}&=
&=\{\nu \in \Spc S'(\R^d): \|\nu\|_{\Spc X_{\Lop_{\rm R}}}= \max(\|\Lop_{\rm R}^{\ast\dagger}\{\nu\}\|_{\Spc X},  \|\Proj_{\Spc P'}\{\nu\}\|_{\Spc P'})<\infty\}\nonumber\\
&=
\{ %\nu=%\Lop^\ast\Op R^{\ast}\Op K _{\rm rad}
\Op L_{\rm R}^\ast\{v\} + p_0^\ast:\quad  (v,p_0^\ast)\in \Spc X_{\rm Rad}%(\R \times \mathbb{S}^{d-1}) 
\times \Spc P'\},\label{Eq:Predualspace}
\end{align}
which is a bona fide Banach space, as shown in  Theorem \ref{Theo:PredualNative}.
The validity of this dual pairing can be checked formally in the absence of null space components: For any $(f,\nu) \in \Spc X'_{\Op L_{\rm R}}(\R^d) \times \Spc X_{\Op L_{\rm R}}(\R^d)$ with $\Proj_\Spc P f=0$ and $\Proj_{\Spc P'} \nu= 0$, we have that
\begin{align*}
\langle f, \nu \rangle&=\langle \Op L^{\dagger}_{\rm R}\{w\}, \Op L^\ast_{\rm R} \{v \}\rangle=\langle  \Op L_{\rm R}\Op L^{\dagger}_{\rm R}\{w\}, v\rangle_{\rm Rad}
%\langle  \overbrace{\Op L_{\rm R}\Op L^{\dagger}_{\rm R}}^{\Identity}\{w\}, v\rangle_{\rm Rad}
%\langle\Op K _{\rm rad}\Op R \overbrace{\Lop (\Identity-\Proj_{\Spc P})\Lop^{-1}}^{\Identity}\Op R^{\ast}\{w\}, v\rangle\\
%&=\langle \Op K _{\rm rad}\Op R\Op R^{\ast} \{w\}, v\rangle_{\rm Rad}=\langle \Op P^\ast_{\rm Rad} \{w\}, v \rangle_{\rm Rad}
=\langle w, v\rangle_{\rm Rad}
\end{align*}
with $(w,v) \in \Spc X'_{\rm Rad}
%(\R \times \mathbb{S}^{d-1})
 \times \Spc X_{\rm Rad}
 %(\R \times \mathbb{S}^{d-1})
$. Finally, since
$\Proj_{\Spc P}$ continuously maps $\Spc X'_{\Op L_{\rm R}}(\R^d)\to \Spc P\embedIso\Spc X'_{\Op L_{\rm R}}(\R^d)$, we can identify $\Proj_{\Spc P'}$ as its adjoint
$(\Proj_{\Spc P})^\ast=\Proj_{\Spc P'}: \Spc X_{\Op L_{\rm R}}(\R^d)\to \Spc P'\embedIso\Spc X_{\Op L_{\rm R}}(\R^d)$.

We now  state our primary theorem. %, as announced in the title.
\begin{theorem} 
\label{Theo:RadonSplines}
Let $E: \R \times \R \to \R$ be a strictly convex loss function, $\Lop$ an isotropic, spline-admissible operator
with polynomial null space $\Spc P_{n_0}$ of degree $n_0$, possibly trivial with the convention that $\Spc P_{-1}=\{0\}$, and let $\psi: \R^+ \to \R^+$ be a convex, {strictly} increasing function.  
Then, for any given $(\V x_m,y_m)\in \R^d \times \R, m=1,\dots,M$ with $M\ge {\rm dim}\Spc P_{n_0}$,  the
generic functional-optimization problem
\begin{align}
\label{Eq:RadonEnergy}
S=\arg \min_{f \in \Spc X'_{\Op L_{\rm R} }(\R^d)} \left( \sum_{m=1}^M E(y_m ,f(\V x_m)) +  
\psi(\|\Op L_{\rm R} f\|_{\Spc X'}%(\R \times \mathbb{S}^{d-1})}
)\right),
\end{align}
with $\Lop_{\rm R}=\Op K _{\rm rad}\Op R \Op L$ and $\Spc X'$ as stated below,
%$\Spc X=\Spc M, L_p$ for $p\in(1,\infty)$ or $\Spc X'=\Spc M$ 
always has a solution. %, weak-* closed and convex.
\begin{enumerate}

\item When $\Spc X=\Spc X'=L_2(\R \times \mathbb{S}^{d-1})$, the solution of \eqref{Eq:RadonEnergy} is unique and representable by the linear kernel expansion
\begin{align}
\label{Eq:RBF}
f(\V x)=p_0(\V x) +\sum_{m=1}^M a_m \rho_{\rm iso}(\V x - \V x_m),
\end{align}
where $\rho_{\rm iso}=2(2\pi)^{d-1}\Fourier^{-1}\{1/(|\widehat L(\bw)|^2 \|\bw\|^{d-1})\}$ is a radial basis function,  $(a_m) \in \R^{M}$ an adequate set of coefficients, and $p_0 \in \Spc P_{n_0}$ a polynomial that lies in the null space of $\Op L_{\rm R}$.
\item When $\Spc X'=L_{p}(\R \times \mathbb{S}^{d-1})$ with $p\in(1,2]$, the solution is unique and admits the parametric representation
\begin{align}
\label{Eq:RadLp}
f(\V x)= p_0(\V x) + \Lop_{\rm R}^{\dagger} \circ \Op J_q\{\sum_{m=1}^M a_m \nu_{\V x_m}\}(\V x) 
\end{align}
with basis functions $\nu_{\V x_1},\ldots,\nu_{\V x_M} \in L_q(\R \times \mathbb{S}^{d-1})$ specified by \eqref{Eq:Radonbasis} and parameters $(a_m)\in \R^M$, $p_0 \in \Spc P_{n_0}$, and where $\Op J_q$ is the pointwise nonlinearity given by \eqref {Eq:LpDual} with $q=p/(p-1)$. (The latter is the duality map $\Op J_{\Spc X}: \Spc X \to \Spc X'$ with $\Spc X=L_{q}(\R \times \mathbb{S}^{d-1})$---see Appendix \ref{App:DualityMap} for explanations). 

\item When $\Spc X'=\Spc M(\R \times \mathbb{S}^{d-1})$, the solution set is the weak$\ast$-closed convex hull of its extreme points, which are all of the form
\begin{align}
f_{\rm ext}(\V x)= p_0(\V x) + \sum_{k=1}^{K_0} a_k 
%\left(
\rho_{\rm rad}(\V \xi_k^\Top \V x - \tau_k) %  + \tilde \rho_\Lop(-\V \xi_k^\Top \V x + \tau_k)\right)
\label{Eq:Extremespline}
\end{align}
with activation function $\rho_{\rm rad}=\Fourier^{-1}\{1/\widehat L_{\rm rad}\}$, for some $K_0\le M-{\rm dim}\Spc P_{n_0}$, $(a_k,\V \xi_k,\tau_k) \in \R \times \mathbb{S}^{d-1} \times \R$ for $k=1,\dots,K_0$, and a null-space component $p_0\in \Spc P_{n_0}$. The optimal regularization cost associated with \eqref{Eq:Extremespline} is $\|\Op L_{\rm R}f_{\rm ext}\|_\Spc M=\sum_{k=1}^{K_0}|a_k|$ and is shared by all solutions.
\end{enumerate}
\end{theorem}

\begin{proof}[Proof of Theorem \ref{Theo:RadonSplines}]
Since $\Lop_{\rm R}^{\ast}$ is injective on $\Spc S_{\rm Rad}%(\R \times \mathbb{S}^{d-1})
$ and, by extension, on 
the completed space $\Spc X_{\rm Rad}
%(\R \times \mathbb{S}^{d-1})
$, the image space $\Spc U=\Lop_{\rm R}^{\ast}\big(\Spc X_{\rm Rad} %(\R \times \mathbb{S}^{d-1})
\big)$ is a bona fide Banach space (see proof of Theorem \ref{Theo:PredualNative} for the details of the construction of $\Spc U$). Its continuous dual is given by
$\Spc U'=\Op L^{-1}_{\Spc P}\Op R^{\ast}\big(\Spc X'_{\rm Rad}
%(\R \times \mathbb{S}^{d-1})
\big)$, in reason of the identities 
$\Op R\Op L^{-1\ast}_{\Spc P}\Lop_{\rm R}^{\ast}= \Op R\Op L^{-1\ast}_{\Spc P}\Op L^\ast  \Op R^\ast\Op K_{\rm rad}=
\Op R \Op R^\ast\Op K_{\rm rad}=\Identity$ on $\Spc X_{\rm Rad}$. %(\R \times \mathbb{S}^{d-1})$. 
Likewise, $\Spc P'$, as identified by \eqref{Eq:DualNullspace}, is a finite-dimensional Banach space.
Its dual is simply $(\Spc P')'=\Spc P$ (the null space of both $\Lop$ and $\Lop_{\rm R}$), owing to the property that all finite-dimensional spaces are reflexive.
Using the notation for direct-sum topologies of \cite{Unser2022}, we then observe that $\Spc X_{\Op L_{\rm R}}(\R^d)=(\Spc U \oplus \Spc P')_{\ell_\infty}$, whose formal dual
$(\Spc U \oplus \Spc P' )'_{\ell_\infty}=(\Spc U' \oplus \Spc P)_{\ell_1}$  is precisely the native space $\Spc X'_{\Op L_{\rm R}}(\R^d)$ described by \eqref{Eq:Nativespace}.
By writing $f=\Lop_{\Spc P}^{-1} \Op R^\ast\{ w \}+ p_0$ and recalling that $\Lop\Lop_{\Spc P}^{-1}=\Identity$ (right-inverse property), we then identify the regularization functional as 
\begin{align*}
\|\Op L_{\rm R} f\|_{\Spc X'%(\R \times \mathbb{S}^{d-1})
}&=\| \Op K_{\rm rad} \Op R \{\Op L \Lop^{-1}_\Spc P \Op R^\ast w + \Op L p_0\}\|_{\Spc X'}\\
&=\| \Op K_{\rm rad} \Op R  \Op R^\ast \{w\} + \Op K_{\rm rad}\Op R \{0\}\|_{\Spc X'}\\
&=\| w \|_{\Spc X'_{\rm Rad}}=\|\Proj_{\Spc U'} f\|_{\Spc U'}
\end{align*} which, in effect, converts the  seminorm over $\Spc X'_{\Op L_{\rm R}}(\R^d)$ into a norm over $\Spc U'$ by factoring out the null space of $\Lop_{\rm R}$. 
The other mathematical ingredient for the optimization problem \eqref{Eq:RadonEnergy} to be well-posed is the weak$\ast$ continuity of the sampling functionals 
$\delta(\cdot-\V x_m): f \mapsto f(\V x_m)$ in the underlying topology. This is equivalent to
$\delta(\cdot-\V x_m) \in  \Spc X_{\Op L_{\rm R}}(\R^d)$ for any $\V x_m \inR^d$. In the present context, this condition can be reframed as $\nu_{\V x_m}=\Lop^{\ast\dagger}_{\rm R}\{\delta(\cdot-\V x_m)\} \in \Spc X_{\rm Rad}(\R \times \mathbb{S}^{d-1})$, which is a property that
is established in Theorem \ref{Theo:GenImpulse} for the cases $\Spc X=C_0$ and $\Spc X=L_q$ for any $q\ge 2$.

The existence of the solution and the parametric descriptions \eqref{Eq:RBF}, \eqref{Eq:RadLp}, and \eqref{Eq:Extremespline} then follows from the three cases of the abstract representer theorem for direct sums \cite[Theorem 3]{Unser2022}. The duality map that is required for the first and second scenarios is $\Op J_\Spc U=  \Lop^{\dagger}_{\rm R} \circ \Op J \circ \Lop^{\ast\dagger}_{\rm R} : \Spc U \to \Spc U'$  (see Appendix \ref{App:DualityMap}, Proposition \ref{Prop:DualMaps} with $\Op T=\Lop^\ast_{\rm R}$,
$\Op T^{-1}=\Lop_{\rm R}^{\ast\dagger}=\Op R^{\ast}\Op L^{\ast-1}_{\Spc P}$,
and $(\Op T^\ast)^{-1}= \Lop^{\dagger}_{\rm R} =\Op L^{-1}_{\Spc P}\Op R^{\ast}$). 

To describe the solution set for the non-reflexive case $\Spc X'=\Spc M$, we invoke the third case of \cite[Theorem 3]{Unser2022}, which tells us that the extreme points of $S$ can all be expressed as the sum of a null-space component plus a linear combination of $K_0\le M-{\rm dim}(\Spc P)$ atoms
$e_k$ that are selected adaptively within a dictionary consisting of the extreme points of the unit ball in $\Spc U'$: $B(\Spc U')=\{u \in \Spc U': \|u\|_{\Spc U'}\le 1\}$. 
Because of the linear isomorphism $\Spc U'=\Op L^{\dagger}_{\rm R}\big(\Spc M_{\rm Rad}%(\R \times \mathbb{S}^{d-1})
\big)$, ${\rm ext}B(\Spc U)=\Op L^{\dagger}_{\rm R}\big({\rm ext}B(\Spc M_{\rm Rad})\big)$. 
%Accordingly, we shall deduce the generic form of the $e_k$ from the extreme points of $\Spc M_{\rm Rad}%(\R \times \mathbb{S}^{d-1})=\Op P_{\rm Rad}\big(\Spc M (\R \times \mathbb{S}^{d-1})\big)$. 
Next, we use the property that $\Spc M_{\rm Rad}=\Spc M(\mathbb{P}^{d})$ whose extreme points 
are $\{\pm \delta_{{\rm Rad}, \V z}\}_{\V z\in \mathbb{P}^{d}}$.
%_{\V z_0=(t_0,\V \xi_0) \in \R \times \mathbb{S}^{d-1}}$ where $\delta_{{\rm Rad}, \V z_0}=\Op P^\ast_{\rm Rad}\{\delta(\cdot-\V z_0)\}
Each %extreme point 
$\delta_{{\rm Rad}, \V z_k}\in 
{\rm ext}B(\Spc M_{\rm Rad})$ then bijectively maps into an extreme point $e_{k}=
\Op L^{\dagger}_{\rm R}\{\delta_{{\rm Rad}, \V z_k}\} \in 
{\rm ext}B(\Spc U)$, and vice versa. 
Finally, by recalling that $\Op L^{\dagger}_{\rm R}=(\Identity-\Proj_{\Spc P})\Op L^{-1}\Op R^{\ast}$
and by invoking Theorem \ref{Theo:Rad1DprofilesContinuous} to show that $\Op L^{-1}\Op R^{\ast}\{\delta_{{\rm Rad}, (t_k,\V \xi_k)}\}=\Op L^{-1}\Op R^{\ast}\Op K_{\rm rad}\Op R \{\delta(\V \xi_k^\Top \V x -t_k)\}=\Op L^{-1} \{\delta(\V \xi_k^\Top \V x -t_k)\}=\rho_{\rm rad}(\V \xi_k^\Top \V x -t_k)$, we find that
\begin{align}
\label{Eq:extremM}
e_k=\pm \rho_{\rm rad}(\V \xi_k^\Top \V x -t_k) \mp p_{0,k},
\end{align}
where $p_{0,k}=\Proj_{\Spc P}\{\rho_{\rm rad}(\V \xi_k^\Top \V x -\tau_k)\}\in \Spc P$. Since every extreme point of $B(\Spc U')$ is necessarily of the form
\eqref{Eq:extremM}, we can substitute this expression in the generic expansion $f_{\rm ext}= \tilde p_0 +\sum_{k=1}^{K_0} a_k e_k $ which, upon collection of all null-space components, yields \eqref{Eq:Extremespline}. \end{proof}

The two cases in Theorem \ref{Theo:RadonSplines} that are of direct practical relevance to machine learning are Items 1 and 3.  The first scenario
yields a learning architecture that is 
a linear expansion of RBFs, which also has a classical RKHS interpretation, as explained in Section 
\ref{Sec:RBF}.

The form of the solution in Item 3 is equivalent to a shallow network with the weights ($\V \xi_k$) of the hidden layer  being normalized.
It actually turns out that this normalization is inconsequential when the activation is a homogeneous function. This happens for instance when the regularization operator $\Lop=(-\Delta)^{\gamma_0/2}$ is a fractional Laplacian which maps into $\rho_{\rm Rad}(t)\propto|t|^{\gamma_0-1}$. Indeed, for any $(\V w_k,b_k)\in \R^d \times \R$, we have that 
\begin{align}
|\V w_k^\Top \V x -b_k|^{\gamma_0-1}= \|\V w_k\|^{-\gamma_0+1} |\V \xi_k^\Top \V x -t_k|^{\gamma_0-1}
\end{align}
with $\V \xi_k=\V w_k/\|\V w_k\|$ and $t_k=b_k/\|\V w_k\|$, which indicates that the normalization (or lack thereof) can be absorbed in the weights $(a_k)$ of the output layer. The case $\gamma_0=1$ with $\Lop=\Delta$ (Laplacian) is particularly attractive, as %\eqref{Eq:CNNaffine} 
it nicely maps into a ReLU network with one hidden layer and a skip connection to implement the affine component \cite{Parhi2021b}. 

The form of the solution in Item 2 is more involved, but still useful to %Even though it does not appear to have any direct application, 
get insight into the transition with $p$ from RBFs to neural nets. The equivalence with Item 1 for $p=2$ is demonstrated in the Section \ref{Sec:RBF}. As for the behavior as $p\to 1$, we  observe that the effect of the duality map $\Op J_{q}$ as $q\to\infty$ is to pull a few maximal values to infinity, while attenuating all other (non-supremum) values towards zero. In effect, this means that the hyper-spherical function $\Op J_q\{\sum_{m=1}^M a_m \nu_{\V x_m}\}$ will exhibit peaks that become more and more pronounced, and eventually converge to a sum of impulses as $p\to 1$, which is consistent with the limit form given by \eqref{Eq:Extremespline}. 

\subsection{Connection with RKHS Methods}
\label{Sec:RBF}
%The form of the solution \eqref{Eq:RBF} 
The scenario $\Spc X=L_2$ in Theorem \ref{Theo:RadonSplines} is compatible with the kernel models of ``classical'' machine learning \cite{Poggio1990b,Scholkopf1997}. This is because the underlying native space is a reproducing-kernel Hilbert space whose topological structure is now made explicit.

\begin{proposition}[Characterization of RKHS for $\Spc X=L_2$]
\label{Theo:RKHSspace}
Let $\Lop$ be a spline admissible operator with a polynomial null space of degree $n_0$ and consider the self-adjoint operator $\Op A=(\Op L^\ast \Op K \Lop)^{-1}$. Then, the native space $\Spc H=L'_{2,\Lop_{\rm R}}(\R^d)$, defined by \eqref{Eq:Nativespace} with $\Spc X'=\Spc X=L_2$, is the reproducing-kernel Hilbert space $\Spc H=\Spc U' \oplus \Spc P$
associated with the composite inner product
\begin{align}\
%\|f\|_{\Spc H}&=\left(\|\Lop_{\rm R} f\|^2_{L_2}+ \sum_{|\V k|\le n_0} |\langle m_\V k^\ast, f \rangle|^2\right)^{1/2}=\sqrt{\langle f, f \rangle_{\Spc H}}
%\end{align}
\langle f_1, f_2 \rangle_{\Spc H}&=\langle\Lop_{\rm R} \{f_1\}, \Lop_{\rm R} \{f_2\} \rangle_{\rm Rad}+
\sum_{|\V k|\le n_0} \langle m_\V k^\ast, f_1 \rangle \langle m_\V k^\ast, f_2 \rangle \nonumber\\
&=\langle(\Lop^\ast \Op K\Lop) \{f_1\}, f_2 \rangle+
\sum_{|\V k|\le n_0} \langle m_\V k^\ast, f_1 \rangle \langle m_\V k^\ast, f_2 \rangle
\label{Eq:Hinner}
%&=\langle(\Lop^\ast \Op K\Lop) \Proj_{\Spc U'}\{f_1\}, \Proj_{\Spc U'}f_2 \rangle+
%\sum_{|\V k|\le n_0} \langle m_\V k^\ast, f_1 \rangle \langle m_\V k^\ast, f_2 \rangle
\end{align}
whose leading term can also be written as
\begin{align}
\langle(\Lop^\ast \Op K\Lop) \{f_1\}, f_2 \rangle&=\langle(\Lop^\ast \Op K\Lop) \Proj_{\Spc U'}\{f_1\},\Proj_{\Spc U'}\{f_2 \}\rangle \nonumber\\
&=\langle\Proj_{\Spc U'}\{f_1\}, \Proj_{\Spc U'}\{f_2\} \rangle_{\Spc U'},\label{Eq:InnerUP}
\end{align}
where $\Proj_{\Spc U'}=\Identity- \Proj_{\Spc P}: \Spc H \to \Spc U'$.
The topological dual of $\Spc H$ is the Hilbert space $\Spc H'=L_{2,\Lop_{\rm R}}(\R^d)=\Spc U \oplus \Spc P'$ equipped with the
inner product
\begin{align}\
\langle g_1, g_2 \rangle_{\Spc H'}&=\langle\Lop_{\rm R}^{\ast\dagger} g_1, \Lop_{\rm R}^{\ast\dagger} g_2 \rangle_{\rm Rad}+
\sum_{|\V k|\le n_0} \langle m_\V k, g_1\rangle \langle m_\V k, g_2 \rangle \nonumber \\
&=\langle \Lop_{\rm R}^{\dagger}\Lop_{\rm R}^{\ast\dagger} g_1, g_2 \rangle+
\sum_{|\V k|\le n_0} \langle m_\V k, g_1\rangle \langle m_\V k, g_2 \rangle \nonumber\\
&=\langle \Op A\Proj_{\Spc U}\{g_1\}, \Proj_{\Spc U}\{g_2\} \rangle+
\sum_{|\V k|\le n_0} \langle m_\V k, g_1\rangle \langle m_\V k, g_2 \rangle,
\label{Eq:Hinner2}
\end{align}
where $\Proj_{\Spc U}=\Identity- \Proj_{\Spc P'}: \Spc H' \to \Spc U$.
 The corresponding linear isometries (Riesz maps) that map $\Spc U\to \Spc U'$,
 $\Spc P' \to \Spc P$, and 
 $\Spc H\to \Spc H'$  are
\begin{align*}
\Op J_{\Spc U}&= \Op A :  \Spc U  \to \Spc U'\\
&=(\Identity- \Proj_{\Spc P})\Op A (\Identity- \Proj_{\Spc P'}) : \Spc U \oplus \Spc P' \to \Spc U',\\
\Op J_{\Spc P'}&=\Proj_{\Spc P}: \Spc P' \to \Spc P,\\
&=\Proj_{\Spc P}\Proj_{\Spc P'}: \Spc U \oplus \Spc P' \to \Spc P,\\
\Op J_{\Spc H}&=\left((\Identity- \Proj_{\Spc P})\Op A (\Identity- \Proj_{\Spc P'}) +  \Proj_{\Spc P}\Proj_{\Spc P'}\right): \Spc H\to \Spc H',
\end{align*}
where the second, alternative forms of $\Op J_{\Spc U}$ and $\Op J_{\Spc P'}$ that include projectors are the proper extension of those operators to the whole space $\Spc H=\Spc U'\oplus \Spc P$. %
%whose inverse is 
%$$\Op J_{\Spc H}^{-1}=\Op J_{\Spc H'}=\Op (\Lop^\ast\Op K \Lop +  \Proj_{\Spc P'}\Proj_{\Spc P}): \Spc H \to \Spc H'.$$
\end{proposition}

Proposition \ref{Theo:RKHSspace} is obtained as a corollary of Theorem \ref{Theo:PredualNative}  with $\Spc X=L_2$, with the help of the intertwining property $\Op L_{\rm R}=\Op K_{\rm rad}\Op R \Op L= \Op R \Op K\Op L$ (see the discussion at the end of Section \ref{Sec:RadonDistributions}).
The technical part is to establish the completeness of $\Spc U$ (resp., 
$\Spc H'=\Spc U\oplus \Spc P'$), which then implies that of $\ \Spc U'$ (resp., 
$\Spc H=\Spc H''=\Spc U'\oplus \Spc P$) by duality. 
It then suffices to verify that the semi-inner products associated with each individual term in \eqref{Eq:Hinner} and \eqref{Eq:Hinner2} induce the four component norms $\|w\|_{\Spc U'}$, $\|p_0\|_{\Spc P}$, $\|v\|_{\Spc U}$, $\|p_0^\ast\|_{\Spc P'}$ for $\Spc X=\Spc X'=L_2$ that appear in the definitions \eqref{Eq:Nativespace} and \eqref{Eq:Predualspace} of the native space and of its predual. %$(\Spc H')'=\Spc H$ and of its predual.
% $\Spc H'$. 
The denomination RKHS applies to any Hilbert space $\Spc H$ of functions on $\R^d$ such that $\delta(\cdot-\V x_0) \in \Spc H'$ for any $\V x_0 \in \R^d$. In our case, this property is equivalent to $\Lop^{\ast\dagger}_{\rm R}\{\delta(\cdot-\V x_0)\} \in L_{2, {\rm Rad}}
%(\R  \times  \mathbb{S}^{d-1})
$, which follows from Theorem \ref{Theo:GenImpulse}.

%The property that 
%The underlying (semi-)reproducing kernel for the Hilbert subspace $\Spc U'$ is 
%$$
%(\V x, \V x_0)\mapsto \Op A\{\delta(\cdot- \V x_0)\}(\V x)=\rho_{\rm iso}(\V x -\V x_0)$$
%It is such that, for all $w \in \Spc U'$ and any  $\V x_0\in\R^d$,
%\begin{align}
%\langle \Proj_{\Spc U'} \{\rho_{\rm iso}(\cdot-\V x_0)\}, v\rangle_{\Spc U'}&=
%\langle (\Lop^\ast \Op K\Lop) \rho_{\rm iso}(\cdot-\V x_0 ),  v\rangle \nonumber\\
%&=\langle \delta(\cdot-\V x_0),w\rangle=w(\V x_0).\label{Eq:RHKSsampling}
%\end{align}
%Problem: $\rho_{\rm iso} \notin \Spc H$\\
%where we have made use of \eqref{Eq:InnerUP} and the property that $ (\Lop^\ast \Op K\Lop)\Proj_{\Spc P}\{\rho_{\rm iso}(\cdot-\V x_0 )\}=0$. 
%In other words, $\Proj_{\Spc U'} \{\rho_{\rm iso}(\cdot-\V x_0 )\}$  (only OK if this is the extended projector in $\Spc S'$) is the (unique) ``representer'' in $\Spc U'$ of the Dirac sampling functional 
%$\delta(\cdot- \V x_0)$ in the sense that the inner product in \eqref {Eq:RHKSsampling}
% returns the value of $v$ at $\V x=\V x_0$ (reproducing kernel property).
 
To get further insight, we now show that the RBF solution \eqref{Eq:RBF} is  a particular case of the $L_p$ solution \eqref{Eq:RadLp} with $p=2$.
Since $L_2=(L_2)'$ is its own dual, $\Op J=\Identity$, which allows us to manipulate  \eqref{Eq:RadLp} as 
\begin{align}
f&=p_0 + \Lop_{\Spc P}^{-1}\Op R^\ast\Op R \Lop_{\Spc P}^{-1\ast}\{\sum_{m=1}^{M} a_m\delta(\cdot - \V x_m)\} \nonumber \\
&=  p_0 +\Lop_{\Spc P}^{-1}\Op K^{-1}\Lop_{\Spc P}^{-1\ast}\{\sum_{m=1}^{M} a_m\delta(\cdot - \V x_m)\}\nonumber \\
&= p_0 + (\Identity- \Proj_{\Spc P})\Op A (\Identity- \Proj_{\Spc P'})\{\sum_{m=1}^{M} a_m\delta(\cdot - \V x_m)\}\label{Eq:RBFRad2}\\
&=p_0 + \Op A\{\sum_{m=1}^{M} a_m\delta(\cdot - \V x_m)\}
=  p_0 +\sum_{m=1}^{M} a_m \rho_{\rm iso}(\V x- \V x_m), \label{Eq:RBFRad}
\end{align}
where $\rho_{\rm iso}=(\Lop^\ast \Op K \Lop)^{-1}\{\delta\}=\Op A\{\delta\}: \R^d \to \R$ is the Green's function of $(\Lop^\ast \Op K \Lop)$  and  $p_0\in \Spc P$. 
%The non-obvious part in this derivation is t
The nonobvious 
simplification from \eqref{Eq:RBFRad2}
to \eqref{Eq:RBFRad} results from two crucial observations: (1) the ``orthogonality-to-the-null-space'' condition $\sum_{m=1}^{M} a_m\delta(\cdot - \V x_m) \in \Spc U$ is necessary for optimality; and (2) the availability of the identity
\begin{align*}
 \forall u\in \Spc U: \quad  (\Identity- \Proj_{\Spc P})\Op A (\Identity- \Proj_{\Spc P'})\{ u\} =\Op A\{ u\} =u^\ast \in \Spc U',
% \forall g \in \Spc H',\quad (\Identity- \Proj_{\Spc P})\Op A (\Identity- \Proj_{\Spc P'})\{ g\}=\Op A\{\Proj_{\Spc U'}g\}=(\Proj_{\Spc U'}\{g\})^\ast \in \Spc U'
 \end{align*}
%where $u^\ast=\Op J_{\Spc U}\{u\}$ the unique Hilbert conjugate of $u$, 
which is tightly linked to the specification of the underlying spaces in Proposition \ref{Theo:RKHSspace}.
%
%\begin{align*}\Lop_{\Spc P}^{-1}\Op K^{-1}\Lop_{\Spc P}^{-1\ast}&=
%(\Identity -\Proj_{\Spc P})\Op L^{-1}\Op K^{-1}\Op L^{-1\ast} (\Identity -\Proj_{\Spc P'})\\
%&= (\Identity -\Proj_{\Spc P}) \Op A
%%(\Op L^{\ast}\Op K\Op L)^{-1}
%(\Identity -\Proj_{\Spc P'}).
%\end{align*}
%Now, $\delta(\cdot-\V x_m) \in \Spc X_{\Lop_{\rm R}}$ has a unique decomposition
%as $v=(\Identity -\Proj_{\Spc P'})\{\delta(\cdot-\V x_m)\} \in \Spc U$ and $p_0^\ast=\Proj_{\Spc P'}\{\delta(\cdot-\V x_m)\}$ ...
Likewise, we find that the quadratic regularization cost associated with the linear model \eqref{Eq:RBFRad} is $\V a^\Top \M G \V a$, where 
$\M G\in \R^{M \times M}$ is a symmetric, conditionally positive-definite matrix whose entries are calculated as follows:
\begin{align*}
[\M G]_{m,n}&=\langle \Lop_{\rm R}\rho_{\rm iso}(\cdot-\V x_m), \Lop_{\rm R}\rho_{\rm iso}(\cdot-\V x_n)\rangle_{\rm Rad}\\
&=\langle \Op R \Op K \Op L \{\rho_{\rm iso}(\cdot-\V x_m)\}, \Op R \Op K \Op L \{\rho_{\rm iso}(\cdot-\V x_n)\}\rangle_{\rm Rad}\\
&=\langle \Op K \Op L \{\rho_{\rm iso}(\cdot-\V x_m)\}, \Op R^\ast \Op R \Op K \Op L \{\rho_{\rm iso}(\cdot-\V x_n)\}\rangle\\
&=\langle  \rho_{\rm iso}(\cdot-\V x_m), \Op L^\ast \Op K\underbrace{\Op R^\ast \Op R \Op K}_{\Identity} \Op L \{\rho_{\rm iso}(\cdot-\V x_n)\}\rangle\\
&=
\langle \rho_{\rm iso}(\cdot - \V x_m),\delta(\cdot-\V x_n)\rangle=\rho_{\rm iso}(\V x_n-\V x_m)=\rho_{\rm iso}(\V x_m-\V x_n).
\end{align*}

As variant of the $L_2$ result in Theorem \ref{Theo:RadonSplines}, we may also consider the modified regularization operator $\tilde \Lop=\Op K^{-\frac{1}{2}} \Op L$ whose frequency response is
$\widehat{\tilde L}(\bw)=\sqrt{2}(2\pi)^{(d-1)/2}\widehat L(\bw)/\|\bw\|^{(d-1)/2}$. For this particular setting, $\tilde{\Op L}_{\rm R}=\Op R\Op K^{\frac{1}{2}}\Op L$, which translates into
$$\|\tilde{\Op L}_{\rm R}f\|^2_{L_2(\R \times \mathbb{S}^{d-1})}=\|\Op R\Op K^{\frac{1}{2}}\Op Lf\|^2_{L_2(\R \times \mathbb{S}^{d-1})}=\|\Op L f\|^2_{L_2(\R^d)},$$
owing to the property that $\Op R\Op K^{\frac{1}{2}}$ is an $L_2$ isometry \cite{Ludwig1966}.
The proposed Radon-domain regularization therefore reduces to the standard energy functional associated with (semi-)reproducing-kernel Hilbert spaces. The resulting basis function 
is $\tilde \rho_{\rm iso}(\V x)=\Fourier^{-1}\{1/|\widehat L|^2\}(\V x)$, which is the same as the one encountered in the classical formulation that does not involve the Radon transform.
While this may suggest that the two formulations are equivalent, there is an important difference
that concerns the dimension of the null space.

As a matter of illustration, we now compare two schemes that utilize the same ``linear'' radial basis $\rho_{\rm iso} (\V x)\propto  \|\V x\|$. Based on \eqref{Eq:GreenLap}, we deduce that this corresponds to the choice $|\widehat L(\bw)|^2=\|\bw\|^{d+1}$ in the classical formulation, which induces the regularization norm $\|(-\Delta)^{(d+1)/2}\{f\}\|_{L_2}$ with a polynomial null space of degree $n_0=\lceil (d+1)/2\rceil> d/2$ \cite{Duchon1977,Wendland2005}.
In our proposed Radon-domain variant, the appropriate regularization 
norm is $\|\Op R\Op K (-\Delta)^{1/2}\{f\}\|_{L_2}$ with a null space of polynomial degree $n_0=0$. This solution is attractive because it does not depend on the dimensionality of the data.
\subsection{Universal Approximation Properties}
\label{Sec:Universal}
The universal-approximation properties of the supervised-learning scheme specified by \eqref{Eq:RadonEnergy} are supported by Theorem \ref{Theo:Native} below,
which summarizes the properties of the native space $\Spc X'_{\Lop_{\rm R}}$ in relation to the regularization operator $\Lop_{\rm R}$ and its generalized inverse $ \Lop^{\dagger}_{\rm R}$.
This result is a direct corollary (dual counterpart) of Theorem \ref{Theo:PredualNative}, whose proof is given in Section \ref{Sec:Predual}.

\begin{theorem}[Properties of the native space $\Spc X'_{\Lop_{\rm R}}$]
\label{Theo:Native} 
Let $\Lop$ be an isotropic LSI operator that is spline-admissible with a polynomial null space $\Spc P$ (possibly trivial) of degree $n_0$.
%In addition to the admissibility conditions in Definition 2, we assume that the operator $\Lop$ is such that the image of $\Spc S(\R^d)$ through $\Lop^\ast$ is included in $L_{1,n_0}(\R^d)$ with $n_0=\lceil\gamma_0-1\rceil \ge \gamma_0-1$.
Then, the operators $\Lop_{\rm R}=\Op K_{\rm rad}\Op R\Op L: \Spc X'_{\Lop_{\rm R}} \to \Spc X'_{\rm Rad}$ and  $\Lop_{\rm R}^{\dagger}=(\Identity -\Proj_{\Spc P})\Op L^{-1}\Op R^\ast: \Spc X'_{\rm Rad} \to L_{\infty,-n_0}(\R^d)$ (the adjoint of $\Lop_{\rm R}^{\ast\dagger}$ in Theorem \ref{Theo:PredualNative}) are continuous and have the properties
\begin{align}
\forall w\in \Spc X'_{\rm Rad}%(\R \times \mathbb{S}^{d-1})
: &\quad \Lop_{\rm R}\Lop_{\rm R}^{\dagger}
\label{Eq:Native1}
\{w\}=w\\
\forall p_0\in \Spc P: & \quad \Lop_{\rm R}\{p_0\}=0 \label{Eq:PrenativeNull}\\
\forall f \in \Spc X'_{\Lop_{\rm R}}(\R^d): & \quad \Lop_{\rm R}^{\dagger}\Lop_{\rm R}\{f\}
=(\Identity-\Proj_{\Spc P})\{f\}=\Proj_{\Spc U'}\{ f\},
\label{Eq:isoNative2}
\end{align}
where $\Spc X'_{\Lop_{\rm R}}(\R^d)\eqdef\Lop^{\dagger}_{\rm R}(\Spc X'_{\rm Rad}) \oplus \Spc P$ is 
equipped with the composite norm $\|f\|_{\Spc X'_{\Lop_{\rm R}}}=\|\Lop_{\rm R}\{f\}\|_{\Spc X'_{\rm Rad}}+ \|\Proj_{\Spc P}\{f\}\|_{\Spc P}$. The space $\Spc X'_{\Lop_{\rm R}}$ is complete and
 isomorphic to $\Spc X'_{\rm Rad} \times \Spc P$
with $f =\Lop^{\dagger}_{\rm R}\{u\} + p_0 \mapsto (w,p_0)=(\Lop_{\rm R}\{f\}, \Proj_{\Spc P}\{f\})$. Moreover, $\Spc S(\R^d) \embedC \Spc X'_{\Op L_{\rm R}}(\R^d) \embedC L_{\infty,-n_0}(\R^d)\embedD \Spc S'(\R^d)$.
\end{theorem}

To explain how Theorem \ref{Theo:Native}  relates to universality, let us first consider the case $\Spc X'=L_2$ for which we have just shown that $\Spc X'_{\Lop_{\rm R}}=\Spc H$ is a RKHS whose (semi-)reproducing kernel 
is %precisely 
$\rho_{\rm iso}=(\Lop^\ast \Op K \Lop)^{-1}\{\delta\}$. From the general properties of
RKHS \cite{Aronszajn1950,Wendland2005}, we know that $\Spc H$ (as a set) can be specified as 
$\Spc H=\overline{{\rm span}\{\rho_{\rm iso}(\cdot-\V y)\}}_{\V y\in \R^d} + \Spc P$, which tells us that the class of RBF estimators of the form given by \eqref{Eq:RBF} is dense in $\Spc H$. This means that such estimators can yield an approximation of any $f \in \Spc H$ to an arbitrary degree of precision. In particular, this applies to any $f \in \Spc S(\R^d)$, due to the inclusion $\Spc S(\R^d) \subset \Spc X'_{\Op L_{\rm R}}(\R^d)$, as guaranteed by Theorem \ref{Theo:Native}. Now, the key to universal approximation is that $\Spc S(\R^d)$ is dense in most of the classical function spaces \cite{Schwartz:1966}, in particular, $C_0(\R^d)$. We then readily deduce that \eqref{Eq:RBF}, for $M$ sufficiently large and a suitable choice of the $\V x_m$, has the ability to reproduce any continuous function $f: \R^d \to \R$. This deduction, of course, is consistent with the classical theory of kernel estimators: Our admissibility conditions for $\widehat{ L}_{\rm rad}$ in Definition \ref{Def:SplineAdm1} (resp., in Definition \ref{Def:SplineAdNonTrivial}) ensure that the function
$\rho_{\rm iso}: \R^d \to \R$ is strictly positive-definite (resp., strictly conditionally positive-definite), which is the standard criterion for universality \cite{Micchelli1986, Micchelli2006,Wendland2005}.

The same kind of density argument can also be made for $\Spc X'=\Spc M$.
The relevant basis elements there are the atomic Radon-compatible Dirac measures $\delta_{{\rm Rad},(t_k,\V \xi_k)}\in \Spc M_{\rm Rad}$ with $(t_k,\V \xi_k)\in \R \times \mathbb{S}^{d-1}$. These get mapped
into $e_k=\Op L^{\dagger}_{\rm R}\{\delta_{{\rm Rad},(t_k,\V \xi_k)}\} \in \Spc U'$, which are essentially ridges (up to some polynomial) characterized by $e_k=\rho_{\rm rad}(\V \xi_k^\Top \cdot -t_k)- p_{0,k}$ with %$\V z_k=(t_k,\V \xi) \in \R \times \mathbb{S}^{d-1}$ and 
$p_{0,k}=\Proj_{\Spc P}\{\rho_{\rm rad}(\V \xi_k^\Top \cdot -t_k)\}$. Thus, by setting $w\approx \sum_k w_k \delta_{{\rm Rad},(t_k,\V \xi_k)}$, we can interpret the generative formula 
$f=\Op L^{\dagger}_{\rm R}\{w\}+ p_0$ in Theorem \ref{Theo:Native} as a linear superposition of ridges plus a global polynomial trend, which is compatible with the form of the estimator in  \eqref{Eq:Extremespline}. 
We then invoke the property that $\Spc S(\R^d) \subset \Spc M_{\Op L_{\rm R}}(\R^d)$,
%, 
%we then have the garantee that this kind of representation can approximate any continuous function as closely as desired.
%Let us note that the latter chain of (continuous) embeddings ensures that the approximation process described by Theorem \ref{Theo:RadonSplines} is universal. In particular, Theorem \ref{Theo:Native} implies that $\Spc M_{\Op L_{\rm R}}(\R^d)$ includes $\Spc S(\R^d)$, which is dense in all classical function spaces \cite{Schwartz:1966}.
%Practically, this
which implies that any continuous function can be approximated as closely as desired by a member of 
$\Spc M_{\Op L_{\rm R}}(\R^d)$. As it turns out, the latter is representable by a superposition of 
ridges plus a polynomial of degree $n_0$. We emphasize that the presence of the polynomial term---the guarantor of stability for Theorem \ref{Theo:Stability}---is essential to counterbalance the growth of the individual atoms at infinity. This is an important aspect where our analysis and conclusions deviate from those of \cite{Sonoda2017}.

%Since $\rho_{\rm rad}(\V \xi_k^\Top \cdot -t_k)=\Op L^{\dagger}_{\rm R}\{\delta_{{\rm Rad},\V z_k}\}- p_{0,k}$ with %$\V z_k=(t_k,\V \xi) \in \R \times \mathbb{S}^{d-1}$ and 
%$p_{0,k}=\Proj_{\Spc P}\{\rho_{\rm rad}(\V \xi_k^\Top \cdot -t_k)\}$,
%the interpretation of \eqref{Eq:Nativespace} is that any $f \in \Spc X'_{\Op L_{\rm R}}$ can be expressed as a linear superposition of ridges plus a suitable polynomial. 
%%Another way to put is that ${\rm span}\{\rho_{\rm rad}(\V \xi^\Top \cdot -t)\}_{(t,\V \xi) \in \mathbb{P}^d}+\Spc P$ is (weak$\ast$) dense in $\Spc M_{\Op L_{\rm R}} % \supset \Spc S(\R^d)$. 
%We can then invoke the continuous embedding $\Spc S(\R^d) \embedC \Spc X'_{\Op L_{\rm R}}$ in Theorem XXX to deduce that any continuous function $f: \R^d \to \R$ can be approximated to any desired level of accuracy by such a superposition of basis functions, which is a statement on universal approximation.
%
%While different notation of a universal approximator have been described in the literature, we can  which ensures that
%
%
%Speak about Universality !!!

\subsection{Regularization Operators for Antisymmetric Activations}
\label{Sec:Antisym}
While the framework that has been discussed so far is very powerful, it has one shortcoming: it yields canonical activation functions $\rho_{\rm rad}$ that are necessarily symmetric. In some cases such as $\Lop_{\rm R}=\Op K_{\rm rad} \Op R(-\Delta)$, these can be converted into one-sided functions such as the ReLU by doctoring the  polynomial term.
% and relations such as $t_+=\frac{1}{2}(|t|+ t)$.
That said, the original scheme that is described by Theorem \ref{Theo:RadonSplines} does not allow for sigmoids, which are frequently used for classification \cite{Bishop2006}. This is the reason why we now introduce a %slight 
variant that systematically produces anti-symmetric activations, including sigmoids for $n_0=0$. 

The idea is to substitute the (symmetric) filtering operator $\Op K_{\rm rad}$ by its anti-symmetric counterpart $\tilde{\Op K}_{\rm rad}$, which includes an additional Hilbert transform. Specifically, $\tilde{\Op K}_{\rm rad}$ is the hyper-spherical radial filter whose frequency response is
$\widehat{\tilde{K}}_{\rm rad}(\omega)=-\jj \,{\rm sign}(\omega) c_p|\omega|^{-d-1}$ and whose adjoint is $\tilde{\Op K}^\ast_{\rm rad}=-\tilde{\Op K}_{\rm rad}$. Since the action of the (radial) Hilbert transform $\tilde{\Op H}_{\rm rad}: \phi(\cdot,\V \xi) \mapsto \phi \astrad 1/(\pi t)$ is well-defined on $\Spc S(\R \times \mathbb{S}^{d-1})$
with $\tilde{\Op H}^\ast_{\rm rad}=-\tilde{\Op H}_{\rm rad}=\tilde{\Op H}^{-1}_{\rm rad}$, we have the identity $\Op R^\ast \Op K_{\rm rad} \Op R=
\Op R^\ast \Op K_{\rm rad}\tilde{\Op H}_{\rm rad}^\ast \tilde{\Op H}_{\rm rad}\Op R=\Op R^\ast \tilde{ \Op K}_{\rm rad}^\ast \tilde{\Op R}=\Identity$ with $\tilde{\Op R}\eqdef\tilde{\Op H}_{\rm rad}\Op R$. Accordingly, we can essentially replicate the whole mechanism of construction of spaces in Section \ref{Sec:ComplementedSpaces} by subtituting $\Spc S_{\rm Rad}$ by $\tilde{\Spc S}_{\rm Rad}=\tilde{\Op R}(\Spc S(\R \times \mathbb{S}^{d-1})=\tilde{\Op H}_{\rm rad}(\Spc S_{\rm Rad})$, which is a space of odd functions that are smooth ($C^\infty$) and included in $L_{p}(\R \times \mathbb{S}^{d-1})$ for all $p\ge 1$. While the members of $\tilde{\Spc S}_{\rm Rad}$ do not necessarily decay rapidly, the mapping $\Op R^\ast\tilde{\Op K}^\ast_{\rm rad}: \tilde{\Spc S}_{\rm Rad} \to \Spc S(\R^d)$ is still guaranted to be an isomorphism (see \cite[Theorem 3.3.1, p. 83]{Ramm2020} where similar arguments are used). 
Under the assumption that the hyper-spherical norm $\|\cdot\|_{\Spc X}$ is continuous on $\tilde{\Spc S}_{\rm Rad}$, we can readily adapt the proof
of Theorem \ref{Theo:Complementedspaces} to establish the following.

\begin{proposition}[Odd Radon-compatible Banach spaces]
\label{Theo:Complementedspaces2}
Consider the Banach space $\tilde{\Spc X}_{\rm Rad}=\overline{(\tilde{\Spc S}_{\rm Rad},\|\cdot\|_{\Spc X})}$ of odd hyper-spherical functions. Then, the following holds.
\begin{enumerate}
\item The map $\Op R^\ast\tilde{\Op K}^\ast_{\rm rad}: \tilde{\Spc X}_{\rm Rad} \to  \tilde{\Spc Y}=\Op R^\ast\tilde{\Op K}^\ast_{\rm rad}\big(\tilde{\Spc X}_{\rm Rad}\big)$ is an isometric bijection, with $\tilde{\Op R}\Op R^\ast\tilde{\Op K}^\ast_{\rm rad}=\Identity$ on $\tilde{\Spc X}_{\rm Rad}$.
\item The map $\tilde{\Op R}^\ast: \tilde{\Spc X}'_{\rm Rad} \to  \tilde{\Spc Y}'=\tilde{\Op R}^\ast\big(\tilde{\Spc X}'_{\rm Rad}\big)$ is an isometric bijection,
with $\tilde{\Op K}_{\rm rad}\Op R\tilde{\Op R}^\ast=\Identity$ on $\tilde{\Spc X}'_{\rm Rad}$.
% or, equivalently, $\Spc X_{\rm Rad}=\Op R(\Spc Y)$.
\end{enumerate}
\end{proposition}

The underlying definition of the ``oddified'' backprojection operator $\tilde{\Op R}^\ast: \tilde{\Spc X}'_{\rm Rad} \to  \tilde{\Spc Y}'$ for $g \in \tilde{\Spc X}_{\rm Rad}'$ is 
\begin{align}
\langle \tilde{\Op R}^\ast\{g\}, \varphi \rangle=\langle g, \tilde{\Op R}\{\varphi\}\rangle_{\rm Rad}%=\langle g, \tilde{\Op H}_{\rm rad}\Op R\{\varphi\}\rangle_{\rm Rad}
=\langle \tilde{\Op H}^\ast_{\rm rad}\{g\}, \Op R\{\varphi\}\rangle_{\rm Rad}
\label{Eq:Rtildeast}
\end{align}
for all $\varphi \in  \tilde{\Spc Y}$ or, equivalently,
$\varphi \in  \Spc S(\R^d)$ since $\Spc S(\R^d)$ is dense in $\Spc Y$ by construction.
Likewise, by using the property that the Hilbert transform is a homeomorphism from $\Spc S_{\rm Liz, even}$ onto $\Spc S_{\rm Liz, odd}$ \cite{Samko1993} with the underlying Lizorkin spaces being included in
$\Spc S_{\rm Rad}$ and $\tilde{\Spc S}_{\rm Rad}\subset L_{p, {\rm odd}}$, respectively, we can adapt the proof of Lemma \ref{Lem:Even} to show that
$\tilde L_{q,{\rm Rad}}=L_{q, {\rm odd}}$ for $q\in(1,\infty)$,  
$\tilde{C}_{0,{\rm Rad}}=C_{0, {\rm odd}}$ for $p=\infty$, and $\tilde{\Spc M}_{{\rm Rad}}=(\tilde{C}_{0,{\rm Rad}})'=\Spc M_{{\rm odd}}$.

The bottom line is that the whole argumentation, including the critical Fourier-based proofs of Section \ref{Sec:Math}, applies in this modified setting as well. Accordingly, all theorems that mention the regularization operator $\Lop_{\rm R}=\Op K_{\rm rad} \Op R \Lop$
and the radial profile $\rho_{\rm rad}=\Fourier^{-1}\{ 1/\widehat{L}_{\rm rad}\}$ are also valid for the odd setting where these quantifies are substituted by
\begin{align}
\tilde{\Op L}_{\rm R}&\eqdef \tilde{\Op K}_{\rm rad} \Op R \Op L=\tilde{\Op K}_{\rm rad}  \Op L_{\rm rad}\Op R, \\
\tilde{\rho}_{\rm rad}(t)& \eqdef \Fourier^{-1}\left\{ \frac{\jj\,{\rm sign}(\omega)}
{ \widehat{L}_{\rm rad}(\omega) } \right\} (t), \label{Eq:Oddactivation}
\end{align}
where \eqref{Eq:Oddactivation} directly follows from \eqref{Eq:Rtildeast} and 
Theorem \ref{Theo:Rad1DprofilesContinuous}. The conditions for admissibility in Definitions \ref{Def:SplineAdm1} and \ref{Def:SplineAdNonTrivial}, which are all Fourier-based, remain the same, while the adjusted activation $\tilde{\rho}_{\rm rad}$ (the 1D Hilbert transform of $\rho_{\rm rad}$) is real-valued and anti-symmetric
because the original Fourier profile $\widehat{L}_{\rm rad}: \R \to \R$ is symmetric.

We note that our admissibility condition with $\gamma_0=1$ for odd variant ($n_0=0$) is compatible with the condition used by Barron to prove the universality of neural networks with sigmoidal activations \cite{Barron1993}.
It is also worth mentioning that the substitution of $\Op L_{\rm R}$ by $\tilde{\Op L}_{\rm R}$ has no incidence on the form of the RBF in \eqref{Eq:RBF} because of the unitary nature of the Hilbert transform.

\subsection{Specific Configurations}
\label{Sec:Examples}
The proposed framework encompasses a wide variety of kernels and activation functions, with minimal restrictions. For instance, one can start with any strictly positive-definite function $\rho_{{\rm rad},0}$ whose Fourier transform is strictly positive, and construct some higher-order variants by (fractional) integration. The variants are such that $\rho_{{\rm rad},\gamma_0}(t)=\Fourier^{-1}\{\frac{\widehat{\rho}_{{\rm rad},0}(\omega)}{|\omega|^{\gamma_0}}\}(t)$ with suitable $\gamma_0>0$ and are guaranteed to meet the requirements in Definition \ref{Def:SplineAdNonTrivial}. The simplest scenario is to set $\rho_{{\rm rad},0}=\delta$, which maps into a Laplacian-type regularization. 

In Table 1, we provide examples of admissible operators together with their corresponding symmetric and anti-symmetric activations. It is noteworthy that the two most popular sigmoids (${\rm tanh}$ and ${\rm arctan}$) are part of the framework. We can determine the explicit frequency response of their regularization filter and show that it is first-order-admissible with a null space that consists of the polynomials of degree $n_0=0$ (the constants). The symmetric spline activations of odd degree $m-1$ and the anti-symmetric ones of even degree are also known: they coincide with the ridge splines
of Parhi and Novak, which are tied to the Radon-domain operator $\Lop_{\rm rad}=\frac{\partial^{m}}{\partial t^{m}}$  \cite{Parhi2021b}. With the present formulation, we can seamlessly extend this family to fractional orders, in direct analogy with \cite{Unser2000a}, by setting $\Lop=(-\Delta)^{\frac{\alpha+1}{2}}$.
\begin{table}\label{tab:activation}
\begin{tabular}{p{100pt} p{100pt} p{60pt} p{100pt}}
\hline \hline &\\[-1.8ex]$\rho_{\rm rad}(t)$ & $\tilde\rho_{\rm rad}(t)$ & $\widehat{L}_{\rm rad}$ or $\widehat{\tilde L}_{\rm rad}$& $n_0$ (null space)\\[1ex]
\hline\\[-2ex]
{\bf Exponential}\\[1ex]
$\ee^{-|t|}$ & N/A & $1+|\omega|^2$ & $-1$ (trivial)\\[2ex]
& {\bf Classical sigmoids}\\
$ $ & $\frac{\tanh(\tfrac{t}{2})}{2}=\tfrac{1}{2} + \frac{1}{1 + \ee^{-t}}$ & $\frac{\sinh(\pi \omega)} {\pi}$ & $0$ (bias)\\[2ex]
$ $ & $\frac{\arctan(t)}{\pi}$ & $\omega \ee^{|\omega|}$ & $0$\\[2ex]
{\bf Ridge splines} & (of degree $n \in \N$)\\[1ex]
$\tfrac{1}{2}|t|$ (or ReLU) & ${\color{black} t \log |t|}$ & $|\omega|^2$ & $1$ (affine)\\[2ex]
$\propto {\color{black} t^{2n} \log |t|}$ & $\frac{{\rm sign}(t) |t|^{2n}} {(2n)!}$& $|\omega|^{2n+1}$ & $2n\ge 2$ (even)\\[2ex]
$\tfrac{1}{2}\frac{|t|^{2n +1}}{(2n +1) !}$ & $\propto t^{2n+1} \log| t|$ & $|\omega|^{2n+2}$ & $2n+1\ge1$ (odd)\\[2ex]
{\bf Fractional splines} & (degree $\alpha \in \R^+\backslash \N$)\\[1ex]
$\color{black} \frac{|t|^\alpha \sin(\frac{\alpha \pi}{2})}{\pi \Gamma(\alpha)}$ & $\color{black} \frac{ {\rm sign}(t) |t|^{\alpha}\cos(\frac{\alpha \pi}{2})}{\pi \Gamma(\alpha)}$& $|\omega|^{\alpha+1}$ & $\lceil\alpha\rceil$\\
%Exponenial splines & ($\alpha \in \R_{>0}\backslash 2 \N$)\\
%$e^{-|t|}\frac{|t|^{\alpha}}{2 \Gamma(\alpha+1)}$ & N/A & $|\omega|^{\alpha+1}(1+\omega^2)$ & $\lceil\alpha \rceil$\\
\hline\hline
\end{tabular}
\caption{Examples of admissible symmetric and anti-symmetric activation functions with their corresponding regularization operator. The anti-symmetric activation is given by \eqref{Eq:Oddactivation} and requires the use of the modified filter $\tilde{\Op K}_{\rm rad}$ in the statement of Theorem \ref{Theo:RadonSplines}.}
\end{table}

\section{Supporting Mathematical Results}
\label{Sec:Math}
To answer the fundamental issue of the existence of a solution in Theorem \ref{Theo:RadonSplines}, we need to 1) prove that the ``predual'' space
 $\Spc X_{\Lop}(\R^d)$ is a proper Banach space (Theorem \ref{Theo:PredualNative}); and 2) establish the weak* continuity
 of the sampling functionals $\delta(\cdot-\V x_k)$.
As we shall see, both aspects largely rest upon the functional characterization of the Schwartz kernel of the pseudoinverse operator $\Op L_{\rm R}^{\ast\dagger}$ provided in Theorem \ref{Theo:GenImpulse}.

%, meaning that it is generally not enough to derive continuity bounds. 
\subsection{Kernel and Stability of Generalized Inverse Operators} 
Let $\nu: f \mapsto \langle \nu, f\rangle$ be a linear functional that is acting on some Banach space $\Spc X'$. We recall that $\nu$ is weak*-continuous if and only if
$\nu \in \Spc X$, where $\Spc X$ is the predual of $\Spc X'$ \cite{Reed1980}. The condition that $\Spc X'$ is reflexive (i.e., $\Spc X''=\Spc X$) is equivalent to the continuity of $\nu$ on $\Spc X$, which is the standard condition for analysis. However, when $\Spc X'$ is not reflexive (e.g., $\Spc X'=\Spc M=(C_0)'$), the constraint of weak* continuity is stronger than continuity, contrary to what could be suggested by %the usage of 
the qualifier ``weak." In that case, the predual space $\Spc X \subseteq \Spc X''$, which is continuously embedded in $\Spc X''$, turns out to be smaller than $\Spc X''$. Therefore, in order to establish the weak* continuity of the Dirac functionals $\delta(\cdot-\V x_0)$ for the scenarios in Theorem \ref{Theo:RadonSplines}, we need to show that $\delta(\cdot-\V x_0) \in \Spc X_{\Op L_{\rm R}}$, which can be reduced to proving that 
$%\nu_{\V x_0}(t,\V \xi)=
\Lop_{\rm R}^{\ast\dagger}\{\delta(\cdot - \V x_0)\}%(t,\V \xi) 
\in 
\Spc X_{\rm Rad}
%(\R \times \mathbb{S}^{d-1})
$.
\begin{theorem}[Properties of the impulse response of $\Op L_{\rm R}^{\ast\dagger}=\Op R \Op L_\Spc P^{-1\ast}$]
\label{Theo:GenImpulse}
Let $\Lop$ be an isotropic operator such that
$\widehat L(\bw)=\widehat L_{\rm rad}(\|\bw\|)$ where $\widehat L_{\rm rad}: \R \to \R$ is a continuous function and $\rho_{\rm rad}(t)=\Fourier^{-1}\{1/\widehat L_{\rm rad}\}(t)$. We consider two cases:
\begin{enumerate}
\item Trivial null space: If $\Lop$ satisfies the admissibility conditions in Definition \ref{Def:SplineAdm1}, 
%(We further assume that the underlying profile is invertible with
%$1/\hat L_{\rm rad}(\omega) \in L_1(\R)$ and $\Fourier^{-1}\{1/\hat L_{\rm rad}(\omega)\}(x)=\rho(x)$ with 
%$\rho\in L_1(\R)$.)
then $\Op L_{\rm R}^{\ast\dagger}=\Op R \Op L^{-1\ast}$ and
\begin{align}
\nu_{\V x_0}(t,\V \xi)&=\Op R \Op L^{-1\ast}\{\delta(\cdot - \V x_0)\}(t,\V \xi) =\rho_{\rm rad}(t-\V \xi^\Top \V x_0)
\label{Eq:Radonbasis1}
\end{align}
with $\V x_0 \in \R^d$ and $(t,\V \xi)\in \R \times \mathbb{S}^{d-1}$.
Moreover, $\nu_{\V x_0}%(t,\V \xi) 
\in \Spc X_{\rm Rad}
%(\R \times \mathbb{S}^{d-1})
$ for $\Spc X=C_0$ as well as $\Spc X=L_q$ with $q\in[1,\infty]$.

\item Nontrivial null space: If $\Lop$ satisfies the admissibility conditions in Definition \ref{Def:SplineAdNonTrivial} with a polynomial null space of degree $n_0$,
%with
%$\widehat L(\bw)=C_0\|\bw\|^{\gamma_0}$ as $\|\bw\|\to 0$, $|\widehat L(\bw)|\ge C_1 \|\bw\|^{\gamma_1}$ with $\gamma_1>1$ for all $\|\bw\|>R$, and $n_0=\lceil \gamma_0-1\rceil$, 
then
\begin{align}
\nu_{\V x_0}(t,\V \xi)&= \Lop_{\rm R}^{\ast\dagger}\{\delta(\cdot - \V x_0)\}(t,\V \xi) \nonumber\\
&=\rho_{\rm rad}(t-\V \xi^\Top \V x_0) - \sum_{n=0}^{n_0} \frac{(-\V \xi^\Top \V x_0)^n}{n!}\big(\kappa_{\rm rad} \ast \partial^n\rho_{\rm rad}\big)(t)
\label{Eq:Radonbasis}
\end{align}
with $\nu_{\V x_0}%(t,\V \xi) 
\in \Spc X_{\rm Rad}
%(\R \times \mathbb{S}^{d-1})
$ for the same spaces as in Item 1, but with $q\in[2,\infty]$.
%\\
%is continuous, bounded and decreasing for $t\to\pm \infty$ and the sampling functional $\delta(\cdot-\V x_0)$ is weak-* continuous on $\Spc  X'_{\Op L_{\rm R}}$ for any $\V x_0 \in\R^d$ for all cases considered in Theorem \ref{Theo:RadonSplines}.
Moreover, %if $\gamma_1\ge \gamma_0$, then
\begin{align}
\label{Eq:Continuitybound}
\sup_{(\V x_0,\V \xi) \in \R^d \times \mathbb{S}^{d-1}} (1+|\V \xi^\Top \V x_0|)^{-n_0} \|\nu_{\V x_0}(\cdot,\V \xi)\|_{L_q(\R)}  < \infty
\end{align}
for any $q\in[2,\infty]$. % and all $(\V x_0,\V \xi) \in \R^d \times \mathbb{S}^{d-1}$.
\end{enumerate}
\end{theorem}
\begin{proof} To show that $\nu_{\V x_0}\in \Spc X_{\rm Rad}$, it is sufficient to prove that $\nu_{\V x_0}\in \Spc X(\R^d \times \mathbb{S}^{d-1})$ since $\nu_{\V x_0}$ is in the range of the Radon transform by construction. %; i.e, $\nu_{\V x_0}=\Op P_{\rm Rad}\{\nu_{\V x_0}\}$.
%
%Clearly, $\Op P_{\rm Rad}\{\nu_{\V x_0}\}=\Op R \overbrace{\Op R^\ast\Op K_{\rm rad}\Op R}^{\Identity} \Op L^{-1\ast}(\Identity-\Proj_{\Spc P'})\{\delta(\cdot - \V x_0)\}=\nu_{\V x_0}$, which confirm that $\nu_{\V x_0}$ is in the range of the Radon transform.

When the null space of $\Lop$ is trivial, $\Lop$ has a stable convolution inverse so that it suffices to show that $\nu_{\V x_0}%=\Op R \Op L^{-1\ast} \{\delta(\cdot-\V x_0)\} 
\in L_q(\R \times \mathbb{S}^{d-1}) \cap C_0(\R \times \mathbb{S}^{d-1})$. % for all $q\ge 1$.
To that end, we formally identify the isotropic distribution $\rho_{\rm iso}=\Op L^{-1\ast} \{\delta\}=\Op L^{-1} \{\delta\}=\Fourier^{-1}\{1/\widehat L_{\rm rad}(\|\bw\|)\}$ and apply Proposition \ref{Prop:IsoRad}, which yields
$$
\nu_{\V x_0}(t,\V \xi)=\Op R \{\rho_{\rm iso}(\cdot -\V x_0)\}(t, \V \xi)=\rho_{\rm rad}(t-\V \xi^\Top \V x_0).
$$
Due to our assumptions, this 
function is such that $\|\nu_{\V x_0}(\cdot,\V \xi)\|_{L_1}=\|\rho_{\rm rad}\|_{L_1}<\infty$ for any fixed $\V \xi \in \mathbb{S}^{d-1}$.
Moreover, since $1/\hat L_{\rm rad} \in L_1(\R)$, $\rho_{\rm rad}$ is bounded, continuous, and vanishing at infinity (by the Riemann-Lebesgue lemma), which gives $\rho_{\rm rad}(\cdot-\V \xi^\Top \V x_0) \in C_0(\R)$ for any $\V x_0 \in \R^d$ and $\V \xi \in \mathbb{S}^{d-1}$. Since
$\rho_{\rm rad} \in L_\infty(\R) \cap L_1(\R)$, we readily deduce that $\rho_{\rm rad}(\cdot-\V \xi^\Top \V x_0) \in L_q(\R)$ for all intermediate $q\ge 1$, which then also yields $\rho_{\rm rad}(t-\V \xi^\Top \V x_0) \in L_q(\R \times \mathbb{S}^{d-1})$ because the surface of the unit sphere $\mathbb{S}^{d-1}$ is bounded.
Finally, since $\rho_{\rm rad}: \R \to \R$ is continuous and vanishing at infinity, $\nu_{\V x_0}(t,\V \xi)$ is jointly continuous in
$(t,\V \xi)$ and vanishing at $t \to \pm \infty$ for all $\V \xi \in \mathbb{S}^{d-1}$, which implies that $\nu_{\V x_0} \in C_0(\R \times \mathbb{S}^{d-1})$.

%The important quantity to study weak-$\ast$ continuity (as well as the case 2 in rep theorem) is
%$\nu_{\V x_0}(t,\V \xi)=\Lop_{\rm R}^{\ast\dagger}\{\delta(\cdot-\V x_0)\}(t,\V \xi)=\Op R %\Op  K
%\Op L^{-1\ast}_{\Spc P}\{\delta(\cdot-\V x_0)\}(t,\V \xi)$; that is, the (generalized) impulse response of $\Lop_{\rm R}^{\ast\dagger}$. 
For the more difficult case of a nontrivial null space, we invoke the Fourier-slice theorem to evaluate the 1D Fourier transform of $\nu_{\V x_0}(\cdot,\V \xi)$ with $\V \xi$ fixed as 
\begin{align}
\hat \nu_{\V x_0}(\omega,\V \xi)&=\frac{\Fourier\{\delta(\cdot-\V x_0) - \sum_{|\V k|\le n_0} \langle m_\V k, \delta(\cdot-\V x_0) \rangle \;  m_\V k^\ast\}(\omega \V \xi) }
{\widehat L_{\rm rad}(\omega)} \nonumber\\
&=\frac{\ee^{-\jj \omega \V \xi^\Top \V x_0} - \sum_{|\V k|\le n_0}\frac{\V x_0^\V k}{\V k!} \; \widehat m_\V k^\ast(\omega \V \xi) }
{\widehat L_{\rm rad}(\omega)}\nonumber\\
&=\frac{\ee^{-\jj \omega \V \xi^\Top \V x_0} - \sum_{|\V k|\le n_0}\frac{\V x_0^\V k}{\V k!} \; (-\jj \omega \V \xi)^{\V k} \widehat \kappa_{\rm rad}(\omega)}
{\widehat L_{\rm rad}(\omega)}\nonumber\\
&=\frac{\ee^{-\jj \omega \V \xi^\Top \V x_0} - \sum_{n=0}^{n_0}\frac{(-\V \xi^\Top \V x_0)^n}{n!} (\jj \omega)^n \widehat \kappa_{\rm rad}(\omega)}
{\widehat L_{\rm rad}(\omega)}, \label{Eq:Radspline}\end{align}
where the simplification of the summation over $\V k$ results from the multinomial expansion
$(-\jj \omega \V \xi^\Top\V x_0)^n=(y_1+ \cdots + y_d)^n= \sum_{|\V k|= n} \frac{n!}{\V k!} \V y^{\V k}$ with
$\V y={\big(-\jj \omega \xi_i x_{0,i}\big)_{i=1}^d}$. The delicate point here is that  \eqref{Eq:Radspline} is potentially singular because it has a pole of multiplicity $\gamma_0$ at $\omega=$0. Fortunately, the condition $\gamma_0\le n_0+1$ ensures that there is a proper pole-zero cancellation: By recalling that $\widehat \kappa_{\rm rad}(\omega)=1$ for $\omega\in\Omega_0=[-R_0,R_0]$ with $R_0=\frac{1}{2}$ and setting $t_0=\V \xi^\Top \V x_0$, we identify the numerator as the remainder of the Maclaurin series of $\ee^{-\jj \omega t_0}$, % around the origin, 
which is bounded by
\begin{align}
\left|\ee^{-\jj \omega t_0} - \sum_{n=0}^{N}\frac{(-\jj t_0)^n}{n!} \omega^n\right|& \le \sup_{\omega \in \R} \left|(-\jj t_0)^{N+1} \ee^{-\jj\omega t_0}\right| \;
\frac{ |\omega|^{N+1}}{(N+1)!} \nonumber\\
&\le \frac{ |t_0|^{N+1}}{(N+1)!}\; | \omega|^{N+1}.
\label{Eq:Remainder}
\end{align}
%In fact, since $\widehat \kappa_{\rm rad}(\omega)\le 1$, t
This then yields the estimate
\begin{align}
 |\hat \nu_{\V x_0}(\omega,\V \xi)|&\le
 \frac{|\V \xi^\Top \V x_0|^{n_0+1}} {C_0(n_0+1) !} |\omega|^{\epsilon} \quad \mbox{ as } \omega\to 0 \nonumber \\
\mbox{ with } \epsilon=n_0+1-\gamma_0& =\begin{cases}
0,& \gamma_0 \in \N\\
1-(\gamma_0-\lfloor \gamma_0\rfloor) \in (0,1),& \mbox{otherwise}.
\end{cases}
\label{Eq:nuatzero}
\end{align}
Since the denominator $\widehat L_{\rm rad}$ in \eqref{Eq:Radspline} is continuous and non-vanishing away from the origin,
%this implies that 
$\hat \nu_{\V x_0}(\cdot,\V \xi)$ is bounded on $\Omega_0$, and, by extension, %(and uniformly continuous) 
on any compact subset of $\R$.
Moreover, since 
$|\ee^{-\jj \omega \V \xi^\Top \V x_0}|=1$ and $\widehat \kappa_{\rm rad}$ %(\omega)$ 
is rapidly decreasing, there exists a constant $C$ such that $|\hat \nu_{\V x_0}(\omega,\V \xi)|\le C |\omega|^{-\gamma_1}$ for $|\omega|>R_1$. This leads to several conclusions.
First, if $\gamma_1> 1$, then $\hat \nu_{\V x_0}(\cdot,\V \xi) \in L_1(\R)$ so that
$\nu_{\V x_0} \in C_{0,\rm Rad}(\R \times \mathbb{S}^{d-1})$ by the same argument as in the nonsingular case. Second, if $\gamma_1> \alpha + \tfrac{1}{2}$, then
$\nu_{\V x_0}(\cdot,\V \xi) \in W_2^{\alpha}(\R)$, where
$W_2^{\alpha}(\R)=\{f: \int_{\R} (1 + |\omega|^2)^\alpha|\hat f(\omega)|^2 \dint \omega<\infty\}$ is the Sobolev space of functions with finite-energy derivatives up to order $\alpha$. Therefore, $\nu_{\V x_0} \in L_{q,\rm Rad}(\R \times \mathbb{S}^{d-1})$ for
all $q\in[2,\infty]$ provided that 
$\gamma_1> 1$. The explicit Radon-domain formula \eqref{Eq:Radonbasis} with $\kappa_{\rm rad}(t)=\Fourier^{-1}\{ \widehat \kappa_{\rm rad}(\omega)\}(t)$ is obtained by taking the 1D inverse transform of \eqref{Eq:Radspline}.

\begin{figure}

\includegraphics[width=12.5cm]{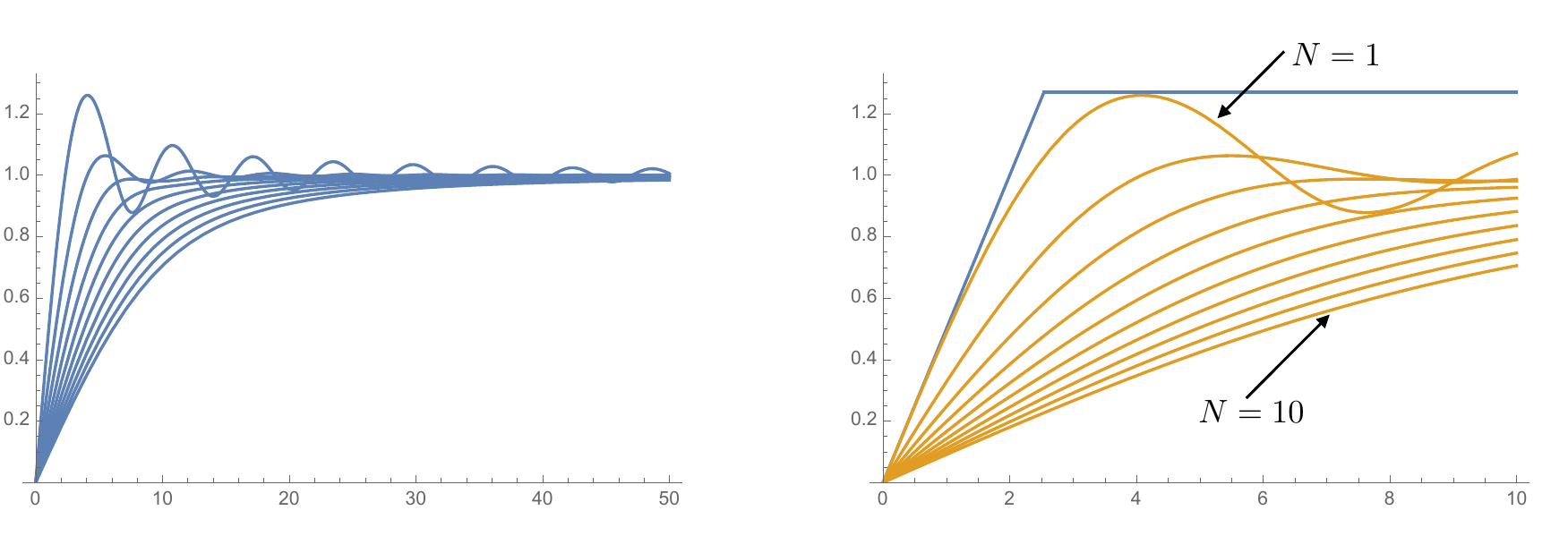}
\caption{Functions $\omega \mapsto |r_{N }(\omega)|$ for $N=1,\dots,10$. The version on the right includes the upper bound specified by \eqref{Eq:uBound} as an overlay.}
\label{Fig:Ncurve}

\end{figure}

To refine our characterization of $\hat \nu_{\V x_0}$, %(\omega,\V \xi)$, 
we introduce the function
\begin{align}
\label{Eq:Radialfunction}
r_{N }(\omega)=\frac{\ee^{-\jj \omega} - \sum_{n=0}^{N}\frac{(- \jj\omega)^n}{n!}}{\frac{(\jj \omega)^N }{N !}}
\end{align}
whose modulus is plotted in Figure \ref{Fig:Ncurve}. % with $N$ ranging from $0$ to $10$. 
Some of the remarkable properties of $r_N$ are
\begin{align}
r_{N }(\omega)= \frac{-\jj \omega}{N+1}\ \mbox{ as } \ \omega\to 0\\
%\forall \omega \in \R, N\ge 1:\quad|r_{N+1 }(\omega)|\le |r_{N}(\omega)|\\
%\forall \omega \in \R:\quad \left|r_{N }(\omega)\right|\le 1.3\\
\forall \omega \in \R:\quad \left|r_{N }(\omega)\right|\le \min(|\omega|/2, 1.27)\label{Eq:uBound}\\
\lim_{\omega \to \infty}  \left|r_N(\omega)\right|=1,
\end{align}
with the global bound \eqref{Eq:uBound} being overlaid on the graph to demonstrate its sharpness. This function will allows us to control the behavior of
%$\hat \nu_{\V x_0}(\omega,\V \xi)$ 
%specified by \eqref{Eq:Radspline}.
%within the frequency range where $\widehat \kappa_{\rm rad}$ is non-zero.
%
\begin{align}
\hat \nu_{\V x_0}(\omega,\V \xi)& = \frac{\ee^{-\jj  t_0 \omega} - \sum_{n \le n_0}\frac{(-\jj t_0\omega)^n}{n!} \kappa_{\rm rad}(\omega) }{\widehat L_{\rm rad}(\omega)} 
 \nonumber\\
& = 
\frac{
\widehat \kappa_{\rm rad}(\omega) r_{n_0}(t_0\omega)
\frac{(-\jj t_0\omega)^{n_0}}{n_0!}
+ \big(1-\widehat \kappa_{\rm rad}(\omega)\big)\ee^{-\jj  t_0 \omega}
}
{\widehat L_{\rm rad}(\omega)} 
\nonumber
\end{align}
%When analysing \eqref{Eq:Radspline}, 
by splitting the frequency domain in three regions. First, for $\omega\in \Omega_0=[0,R_0]$,  where $\widehat \kappa_{\rm rad}(\omega)=1$, we find that
\begin{align}
|\hat \nu_{\V x_0}(\omega,\V \xi)|& =\left|
\frac{r_{n_0}(t_0\omega)\frac{(-\jj t_0\omega)^{n_0}}{n_0!}}{\widehat L_{\rm rad}(\omega)} \right|
\le  |t_0|^{n_0} \min{(|\omega|,2)} \frac{2\frac{|\omega|^{n_0}}{n_0!}}
{|\widehat L_{\rm rad}(\omega)|}.
\label{vbound1}
\end{align}
For the transition region $\omega\in \Omega_{01}=[R_0, \tilde R_{1}]$ with $\tilde R_{1}=\max(2R_0,R_1)$, where $0\le \widehat \kappa_{\rm rad}(\omega)\le 1$, we  bound the numerator by its maximum, which yields
 \begin{align}
|\hat \nu_{\V x_0}(\omega,\V \xi)|& \le 
\frac{
2 \frac{|t_0 \tilde R_1 |^{n_0}}{n_0!}
+1}
{|\widehat L_{\rm rad}(\omega)|}.
\label{vbound2}
\end{align}
Finally, for $\omega\in\Omega_1=[\tilde R_1,\infty]$, where $\widehat \kappa_{\rm rad}(\omega)=0$, we get the expected tail behavior
\begin{align}
|\hat \nu_{\V x_0}(\omega,\V \xi)|& \le 
\frac{
1}
{|\widehat L_{\rm rad}(\omega)|}\le \frac{
1}{C_1|\omega|^{\gamma_1}}.
\label{vbound3}
\end{align}
Based on those bounds with $t_0=1$, we define the auxiliary function
\begin{align}
u(\omega)=\begin{cases}
 \min{(|\omega|,2)} \frac{2\frac{|\omega|^{n_0}}{n_0!}}
{|\widehat L_{\rm rad}(\omega)|},&|\omega|< R_0 \\
\frac{
2 \frac{\tilde R_1^{n_0}}{n_0!}
+1}{|\widehat L_{\rm rad}(\omega)|},
&R_0\le |\omega|\le \tilde R_1\\
\frac{1}
{|\widehat L_{\rm rad}(\omega)|},&\tilde R_1< |\omega|
\end{cases}
\end{align}
which, by construction, %, because of the reasons listed before, 
is such that $\|u\|_{L_p}<\infty$ for any $p\ge 1$.
%With the help of this auxiliary function, 
We can now use \eqref{vbound1},  \eqref{vbound2}, and \eqref{vbound3}  to bound the norm of $\hat \nu_{\V x_0}(\cdot,\V \xi)$ by distinguishing between two cases. For $t_0\le 1$, we have 
$\|\hat \nu_{\V x_0}(\cdot,\V \xi)\|_{L_p} \le \|u\|_{L_p}$, while, for $t_0\ge 1$, we get 
$\|\hat \nu_{\V x_0}(\cdot,\V \xi)\|_{L_p} \le |t_0|^{n_0} \|u\|_{L_p}$.
By combining these two inequalities, we obtain the universal norm estimate
\begin{align}
 \|\hat \nu_{\V x_0}(\cdot,\V \xi)\|_{L_p(\R)} &\le (1+ |\V \xi^\Top \V x_0|)^{n_0} \|u\|_{L_p} < + \infty,
\end{align}
which holds for all $(\V x_0,\V \xi) \in \R^d \times \mathbb{S}^{d-1}$ and $p\ge 1$. 
Finally, we invoke the boundedness of the (inverse) Fourier transform $\Fourier^{-1}: L_p(\R) \to L_q(\R)$ for $p\in[1,2]$ to obtain the generic bound
\begin{align}
\sup_{(\V x_0,\V \xi) \in \R^d \times \mathbb{S}^{d-1}} (1+|\V \xi^\Top \V x_0|)^{-n_0} \|\nu_{\V x_0}(\cdot,\V \xi)\|_{L_q(\R)}    < \infty,
\label{Eq:Boundspace1}
\end{align}
which is the desired result.

\end{proof}

As complement to the proof of Theorem \ref{Theo:GenImpulse}, we make the following remarks.

1. The function $\omega\mapsto \hat \nu_{\V x_0}(\omega,\V \xi)$ in \eqref{Eq:Radspline}
is $\epsilon$-H\"older continuous around the origin for $\gamma_0 \notin 2 \N$
and as smooth as $\widehat L$ when $\gamma_0=2n$. For instance, the case where  $\Lop=(-\Delta)^{n}$ is a non-fractional iterate of the Laplacian corresponds to $\widehat L_{\rm rad}(\omega)=\omega^{2n}$ and  $\hat \nu_{\V x_0}(\cdot,\V \xi) \in C^\infty(\R)$, which then translates into $t\mapsto \nu_{\V x_0}(t,\V \xi)$ being rapidly decreasing, but with a limited order of differentiability controlled by $\gamma_0=\gamma_1=2n$.

2. The proof can be readily adapted to characterize the partial derivatives of $\nu_{\V x_0}$ by replacing $\delta(\cdot-\V x_0)$ by $\partial^{\V n}\delta(\cdot-\V x_0)$ with $|\V n|< n_0$. These distributions are such that 
$\langle \partial^{\V n}\delta(\cdot-\V x_0), m_\V k\rangle=(-1)^{|\V n|}\partial^{\V n}m_\V k(\V x_0)=(-1)^{|\V n|}m_{\V k-\V n}(\V x_0)$, with the convention that $m_{\V k-\V n}=0$ if $k_m<n_m$ for any $m\in \{1,\dots,d\}$.

3. The leading term in \eqref{Eq:Radonbasis} is $\rho_{\rm rad}(t-\tau_0)$ with $\tau_0=\V \xi^\Top \V x_0$, which, in the non-trivial scenario, typically grows as $O(|t|^{\gamma_0-1})$. Our analysis shows that the second correction term in \eqref{Eq:Radonbasis}, which is unbounded at infinity as well, essentially neutralizes this growth. It is tempting to call this a miraculous cancellation.

The bottom line is that, for any $(\V x_0, \V \xi_0) \in \R^d \times \mathbb{S}^{d-1}$, the function $t\mapsto \nu_{\V x_0}(t,\V \xi_0)$ is continuous, bounded with a maximum that is proportional to $|\V \xi_0^\Top \V x_0|^{n_0}$ (see \eqref{Eq:Boundspace1} with $p=\infty$), and vanishing at infinity.
This is a remarkable property that also guarantees the boundedness of $\Lop_{\rm R}^{\ast\dagger}: L_{1,n_0}(\R^d) \to \Spc X_{\rm Rad}$, which is not obvious a priori.
The enabling ingredient there is \eqref{Eq:Continuitybound}, which ensures that the corresponding bounding constant in Theorem \ref{Theo:Stability} is finite.
Indeed, since $\|\V x\|\ge |\V \xi^\Top \V x|$ with equality if and only if  $\V \xi$ and $\V x$ are colinear, we have that
\begin{align*}
\|\Lop_{\rm R}^{\ast\dagger}\|&\le \sup_{\V x\in\R^d, \; \V \xi \in \mathbb{S}^{d-1}} (1+\|\V x\|)^{-n_0} \|\nu_{\V x}(\cdot,\V \xi)\|_{\Spc X(\R)}\\
&\le \sup_{\V x\in\R^d,\; \V \xi \in \mathbb{S}^{d-1}} (1+|\V \xi^\Top \V x|)^{-n_0} \|\nu_{\V x}(\cdot,\V \xi)\|_{\Spc X(\R)} < \infty.
\end{align*}

\begin{theorem}[Stability of Cartesian-to-Radon-domain mappings]
\label{Theo:Stability}
Let $h_{\V x}(t,\V \xi)=\Op T\{\delta(\cdot-\V x)\}(t,\V \xi)\}$ denote the generalized impulse response of the operator $\Op T: L_{1,\alpha}(\R^d) \to \Spc X_{\rm Rad}
(\R \times \mathbb{S}^{d-1})
$ and let
\begin{align}
C_{p,\alpha}=\sup_{\V x\in\R^d,\; \V \xi \in \mathbb{S}^{d-1}} (1+\|\V x\|)^{-\alpha} \|h_{\V x}(\cdot,\V \xi)\|_{L_p(\R)}.
\end{align}
\begin{enumerate}

\item $\Spc X=C_0$: If %$h_{\V x} \in C_{0,\rm Rad}(\R \times \mathbb{S}^{d-1})$ for all $\V x\in \R^d$ and 
$C_{\infty,\alpha}<\infty$, then $\Op T: L_{1,\alpha}(\R^d) \to C_{0,\rm Rad}
%(\R \times \mathbb{S}^{d-1})
$ is bounded with $\|\Op T\|\le C_{\infty,\alpha}$.

\item $\Spc X=L_p$ with $p\in(1,\infty)$: If $C_{p,\alpha}<\infty$, then $\Op T: L_{1,\alpha}(\R^d) \to L_{p, \rm Rad}
%(\R \times \mathbb{S}^{d-1})
$ %for $q=\tfrac{p}{p-1}$ 
is bounded with 
$
\|\Op T\|\le \frac{2 \pi^{d/2}}{\Gamma(d/2)} C_{p,\alpha}$. 
\end{enumerate}
The same holds true for the adjoint $\Op T^\ast: \Spc X'_{\rm Rad} 
%(\R \times \mathbb{S}^{d-1}) 
\to 
L_{\infty,-\alpha}(\R^d)$ with $\|\Op T^\ast\|=\|\Op T\|$.

\end{theorem}
\begin{proof} 
%The key is to recognize that
%\begin{align*}
%C_{\Spc X,\alpha}&=\sup_{\V x\in\R^d,\; \V \xi \in \mathbb{S}^{d-1}} (1+|\V \xi^\Top \V x|)^{-\alpha} \|h_{\V x}(\cdot,\V \xi)\|_{\Spc X(\R)}\\
%&\ge\sup_{\V x\in\R^d, \; \V \xi \in \mathbb{S}^{d-1}} (1+\|\V x\|)^{-\alpha} \|h_{\V x}(\cdot,\V \xi)\|_{\Spc X(\R)}
%\end{align*}
%since $\|\V x\|\ge |\V \xi^\Top \V x|$ with equality iff $\V \xi$ and $\V x$ are collinear.
The function $\big((t, \V \xi),\V x\big) \mapsto h_{\V x}(t,\V \xi)$ is the Schwartz kernel of the operator $\Op T$, so that
\begin{align*}
|g(t,\V \xi)|&=\big|\Op T\{f\}(t,\V\xi)\big| = \big|\int_{\R^d} h_{\V x}(t,\V \xi) \;  f(\V x)\dint \V x\big|\\
&\le \sup_{\V x\in\R^d,\; \V \xi \in \mathbb{S}^{d-1}}\left((1+\|\V x\|)^{-\alpha} \|h_{\V x}(\cdot,\V \xi)\|_{L_\infty}\right)\;\int_{\R^d}(1+\|\V x\|)^{\alpha}|f(\V x)| \dint \V x.
\end{align*}
Consequently,%where the right-hand side integral is identified as $\|f\|_{L_{1,\alpha}}$. It then follows that
\begin{align*}
\|g\|_{L_{\infty}}=\sup_{(t,\V \xi) \in \R \times \mathbb{S}^{d-1}}|g(t,\V \xi)|&=C_{\infty,\alpha} \|f\|_{L_{1,\alpha}},
\end{align*}
which is the desired bound for $\Spc X=C_0$.

To handle the reflexive case $\Spc X=L_p$, we consider the adjoint operator
$\Op T^\ast$ whose Schwartz kernel $\big(\V x,(t, \V \xi)\big) \mapsto h_{\V x}(t,\V \xi)$ is obtained by transposition. We now show that $\Op T^\ast: L_q(\R \times \mathbb{S}^{d-1}) \to L_{\infty,-\alpha}(\R^d)$ with $q=\tfrac{p}{p-1}$ is bounded which, by duality, implies that the same holds true for $(\Op T^\ast)^\ast=\Op T: L_{1,\alpha}(\R^d) \toC L_p(\R \times \mathbb{S}^{d-1})$
since $L_{1,\alpha}(\R^d)$ is isometrically embedded in its bidual $\big(L_{1,\alpha}(\R^d)\big)''=\big(L_{\infty,-\alpha}(\R^d)\big)'$.
The action of the adjoint operator is described as
\begin{align*}
f(\V x)&=\Op T^\ast\{g\}(\V x) = \int_{\R} \int_{\mathbb{S}^{d-1}} h_{\V x}(t,\V \xi) \;  g(t,\V \xi)\dint \V \xi\dint t
\end{align*}
which, with the help of H\"older's inequality, yields the bound
\begin{align*}
\big| (1+\|\V x\|)^{-\alpha} f(\V x)\big|&\le
(1+\|\V x\|)^{-\alpha} \|h_{\V x}\|_{L_p(\R \times \mathbb{S}^{d-1})} \; \|g\|_{L_q(\R \times \mathbb{S}^{d-1})}\\
& \le \sup_{_{\V x\in\R^d,\; \V \xi \in \mathbb{S}^{d-1}} } \left( (1+\|\V x\|)^{-\alpha} S_{d} \;\|h_{\V x}(\cdot,\V \xi) \|_{L_p(\R)}\right) \|g\|_{L_q(\R \times \mathbb{S}^{d-1})}\\
& \le S_{d}\; C_{p,\alpha} \; \|g\|_{L_q(\R \times \mathbb{S}^{d-1})},
\end{align*}
where $S_d=\frac{2 \pi^{d/2}}{\Gamma(d/2)} $ is the surface of the unit hypersphere $\mathbb{S}^{d-1}$.
This proves the desired result with $\|\Op T\|=\|\Op T^\ast\|\le S_d C_{p,\alpha}$.

\end{proof}
%\begin{proposition}[Stability of shift-invariant Radon projectors] Let 
%$h_{\V x_0}(t,\V \xi)=h(t-\V \xi^\Top \V x_0)$ denote the impulse response of the operator $\Op T: L_{1,\alpha}(\R^d) \to \Spc X_{\rm Rad}(\R \times \mathbb{S}^{d-1})$. Then
%
%\end{proposition}
%\begin{align*}
%|g(t,\V \xi)|&=\big|\Op T\{f\}(t,\V\xi)\big| = \big|\int_{\R^d} h(t-\V \xi^\Top\V x) \;  f(\V x)\dint \V x\big|\\
%&\le S_n \|h\|_{L_1(\R)} \|f\|_{L_\infty}
%\end{align*}
%Dual operator:
%\begin{align*}
%f(\V x)&=\Op T^\ast\{g\}(\V x) = \int_{\mathbb{S}^{d-1}}  \int_{\R} h(t-\V \xi^\Top\V x) \;  g(t,\V \xi)\dint t \dint \V \xi
%\end{align*}
%Argument inspired by Young's inequality for convolution
%\begin{align*}
%\|f\|_{L_p}&\le  S_n  \|h\|_{L_1(\R)} \;  \|g\|_{L_\infty(\R \times  \mathbb{S}^{d-1})}
%\end{align*}
%
\subsection{Characterization of the Predual Space $\Spc X_{\Op L_{\rm R}}$}
\label{Sec:Predual}
The application of the general representer in \cite{Unser2022} requires that
the native space $\Spc X'_{\Lop_{\rm R}}=(\Spc X_{\Lop_{\rm R}})'$ be identifiable as the topological dual of 
some primary space $\Spc X_{\Lop_{\rm R}}$ endowed with the Banach property. We achieve this constructively through a completion process that ensures that $\Spc X_{\Lop_{\rm R}}$ is a complete normed space.

\begin{theorem}[Construction of the predual space $\Spc X_{\Lop_{\rm R}}$]
\label{Theo:PredualNative} 
Let $\Lop$ be an isotropic LSI operator that is spline-admissible with a polynomial null space $\Spc P$ of degree $n_0$.
%In addition to the admissibility conditions in Definition 2, we assume that the operator $\Lop$ is such that the image of $\Spc S(\R^d)$ through $\Lop^\ast$ is included in $L_{1,n_0}(\R^d)$ with $n_0=\lceil\gamma_0-1\rceil \ge \gamma_0-1$.
Then, $\Lop_{\rm R}^{\ast\dagger}=\Op R \Op L^{-1\ast} (\Identity -\Proj_{\Spc P'}): L_{1,n_0}(\R^d) \toC \Spc X_{\rm Rad}
%(\R \times \mathbb{S}^{d-1})
$ is bounded,
 and admits the unique extension
$\Spc \Lop_{\rm R}^{\ast\dagger}: \Spc X_{\Lop_{\rm R}}(\R^d) \toC \Spc X_{\rm Rad}
%(\R \times \mathbb{S}^{d-1})
$ with the properties
\begin{align}
\forall v\in \Spc X_{\rm Rad}%(\R \times \mathbb{S}^{d-1})
: &\quad \Lop_{\rm R}^{\ast\dagger}
\label{Eq:isoPrenative1}
\Lop_{\rm R}^{\ast}\{v\}=v\\
\forall p_0^\ast\in \Spc P': & \quad \Lop_{\rm R}^{\ast\dagger}\{p_0^\ast\}=0 \label{Eq:PrenativeNull}\\
\forall f \in \Spc X_{\Lop_{\rm R}}(\R^d): & \quad \Lop_{\rm R}^{\ast}\Lop_{\rm R}^{\ast\dagger}\{f\}
=(\Identity-\Proj_{\Spc P'})\{f\}=\Proj_\Spc U\{ f\},
\label{Eq:isoPrenative2}
\end{align}
where $\Lop_{\rm R}^{\ast}=\Lop^\ast \Op R^\ast\Op K_{\rm rad}$ and $\Spc X_{\Lop_{\rm R}}(\R^d)=\Lop^\ast_{\rm R}(\Spc X_{\rm Rad}) \oplus \Spc P'$ is 
equipped with the norm $\|f\|_{\Spc X_{\Lop_{\rm R}}}\eqdef\max(\|\Lop_{\rm R}^{\ast\dagger}\{f\}\|_{\Spc X}, \|\Proj_{\Spc P'}\{f\}\|_{\Spc P'})$. The space $\Spc X_{\Lop_{\rm R}}$ is complete and
 isomorphic to $\Spc X_{\rm Rad} \times \Spc P'$
with $f =\Lop^\ast_{\rm R}\{v\} + p_0^\ast \mapsto (v,p_0^\ast)=(\Lop_{\rm R}^{\ast\dagger}\{f\}, \Proj_{\Spc P'}\{f\})$. Moreover, $\Spc S(\R^d) \embedD L_{1,n_0}(\R^d) \embedD  \Spc X_{\Op L_{\rm R}}(\R^d)\embedD \Spc S'(\R^d)$ with the embedding being continuous and dense.
\end{theorem}

Theorem \ref{Theo:PredualNative}  obviously also applies to scenarios where the null space is trivial by setting
$\Spc P'=\{0\}$ and only retaining \eqref{Eq:isoPrenative1}, in which case $\Spc X_{\Lop_{\rm R}}(\R^d)=\Lop^\ast_{\rm R}(\Spc X_{\rm Rad})$ and $\Spc S(\R^d) \embedD L_{1}(\R^d) \embedD  \Spc X_{\Op L_{\rm R}}(\R^d)\embedD \Spc S'(\R^d)$.
%Properties \eqref{Eq:isoPrenative1}-\eqref{Eq:isoPrenative2} imply that $\Lop_{\rm R}^{\ast\dagger}$ is the Moore-Penrose inverse of $\Lop_{\rm R}^{\ast}$ on $ \Spc X_{\Lop_{\rm R}}(\R^d)$; i.e., $
%\Lop_{\rm R}^{\ast}\Lop_{\rm R}^{\ast\dagger}\Lop_{\rm R}^{\ast}=\Lop_{\rm R}^{\ast}$
%$\Lop_{\rm R}^{\ast\dagger}\Lop_{\rm R}^{\ast}\Lop_{\rm R}^{\ast\dagger}=\Lop_{\rm R}^{\ast\dagger}$, $(\Lop_{\rm R}^{\ast}\Lop_{\rm R}^{\ast\dagger})^\ast$
\begin{proof}
Since $\Op R^{\ast}\Op K_{\rm rad}\big(\Spc S_{\rm Rad}
%(\R \times \mathbb{S}^{d-1})
\big)=\Spc S(\R^d)$, the image of $\Spc S_{\rm Rad}$ under $\Lop_{\rm R}^\ast$ %=\Lop^\ast\Op K\Op R^{\ast}$ 
is a vector space denoted by $\Spc S_{\Lop_{\rm R}^\ast}(\R^d)=\Lop_{\rm R}^{\ast}\big(\Spc S_{\rm Rad})=\Lop^\ast\big(\Spc S(\R^d)\big)$. % with $\Spc S_{\Lop_{\rm R}^\ast}(\R^d) \subset L_{1,n_0}(\R^d)$. 
The spline-admissibility of $\Lop$ implies that $\Lop^\ast$ is injective on $\Spc S(\R^d)$ which, in turn, translates into the injectivity of $\Lop_{\rm R}^\ast$ on $\Spc S_{\rm Rad}$. The latter statement is equivalent to the existence of a linear operator $\left.\Lop_{\rm R}^{\ast-1}\right|_{\Spc S_{\Lop_{\rm R}^\ast}}=\Lop_{\rm R}^{\ast-1}$ (for short)  such that, for any $u=\Lop_{\rm R}^{\ast}\{ \phi\}\in \Spc S_{\Lop_{\rm R}^\ast}(\R^d)$ with $\phi \in \Spc S_{\rm Rad}$, it holds that
\begin{align*}
\Lop_{\rm R}^{\ast-1}\{u\}=\Lop_{\rm R}^{\ast-1}\Lop_{\rm R}^\ast \{\phi\} =\phi.
\end{align*}
Hence, if $\phi \mapsto \|\phi\|_{\Spc X}$ is a valid norm on $\Spc S_{\rm Rad}$, then
the same holds true for $u \mapsto \|u\|_{\Spc U}\eqdef \|\Lop_{\rm R}^{\ast-1}\{u\} \|_{\Spc X}$ on
$\Spc S_{\Lop_{\rm R}^\ast}(\R^d)$. This means that the normed spaces $(\Spc S_{\rm Rad}
%(\R \times \mathbb{S}^{d-1})
,\|\cdot\|_{\Spc X})$ and $(\Spc S_{\Lop_{\rm R}^\ast}(\R^d),\|\cdot\|_{\Spc U})$ are (isometrically) isomorphic.
Moreover, since $\Spc S_{\Lop_{\rm R}^\ast}(\R^d) \subset L_{1,n_0}(\R^d)$ (Condition 4 in Definition \ref{Def:SplineAdNonTrivial}) and $\Spc S_{\rm Rad} \embedD \Spc X_{\rm Rad}$, the inverse operator $\left.\Lop_{\rm R}^{\ast-1}\right|_{\Spc S_{\Lop_{\rm R}^\ast}}$ coincides on $\Spc S_{\Lop_{\rm R}^\ast}$ with $\Lop_{\rm R}^{\ast\dagger}: L_{1,n_0}(\R^d) \toC \Spc X_{\rm Rad}$  whose impulse response is characterized in Theorem \ref {Theo:GenImpulse}. The well-posedness and boundedness of $\Lop_{\rm R}^{\ast\dagger}$ on $L_{1,n_0}(\R^d)$ for $n_0=\lceil \gamma_0-1\rceil$ follows from Theorem \ref{Theo:Stability} and \eqref{Eq:Continuitybound} in Theorem \ref {Theo:GenImpulse}, which provides the required stability condition. The other fundamental ingredient is 
$\langle m_\V k, u\rangle=\langle m_\V k, \Lop_{\rm R}^\ast \{\phi\}\rangle=\langle \Lop_{\rm R} \{m_\V k\}, \phi\rangle_{\rm Rad}=0$ for all $u\in 
\Spc S_{\Lop^\ast}(\R^d)$ and $|\V k|\le n_0$, which implies that $\Proj_{\Spc P'}\{u\}=0$. Consequently, we have that
$$\Lop_{\rm R}^{\ast\dagger}\{u\}=\Lop_{\rm R}^{\ast-1}(\Identity-\Proj_{\Spc P'})\{u\}=\Lop_{\rm R}^{\ast-1}\{u\}=\Lop_{\rm R}^{\ast-1}\Lop_{\rm R}^\ast \{\phi\}=\phi,
%\Lop^{\ast-1}(\Identity-\Proj_{\Spc P'})\Lop^\ast \varphi=\Lop^{\ast-1}\Lop^{\ast}\varphi=\varphi,
$$
which confirms the equivalence of $\Lop_{\rm R}^{\ast\dagger}$ and $\left.\Lop_{\rm R}^{\ast-1}\right|_{\Spc S_{\Lop_{\rm R}^\ast}}$.

So far, we have shown that the operator $\Lop_{\rm R}^{\ast\dagger}: \big(\Spc S_{\Lop_{\rm R}^\ast},\|\cdot\|_{\Spc U}\big) \to \Spc X_{\rm Rad}$ is an isometry. As next step, we invoke the bounded linear transformation (BLT) theorem \cite[Theorem I.7, p 9]{Reed1980} to uniquely extend the operator to the completed space
$\Spc U=\overline{\big(\Spc S_{\Lop_{\rm R}^\ast},\|\cdot\|_{\Spc U}\big)}$ which, by construction, is the Banach space equipped with the norm $\|\cdot\|_{\Spc U}$. This extension argument also applies the other way around:  Since
${(\Spc S_{\Lop^\ast}(\R^d),\|\cdot\|_{\Spc U})} \embedIso \Spc U$, 
the operator $\Lop_{\rm R}^\ast: (\Spc S_{\rm Rad},\|\cdot\|_{\Spc X}) \to \Spc U$ 
%is an isometry that 
has a unique (isometric) extension $\Lop_{\rm R}^\ast: \Spc X_{\rm Rad} \to \Spc U$ with $\Spc X_{\rm Rad}$ being the closure
$\Spc X_{\rm Rad} =\overline{(\Spc S_{\rm Rad},\|\cdot\|_{\Spc X})}$.
This proves that  the spaces $\Spc X_{\rm Rad}$ and $\Spc U$ are isometrically isomorphic with $\Spc U=\Lop_{\rm R}^{\ast}\big(\Spc X_{\rm Rad}\big)$ and 
$\Spc X_{\rm Rad}=\Lop_{\rm R}^{\ast\dagger}\big(\Spc U\big)$. In addition, we have that $\Spc U \perp \Spc P$, which means that
$\Proj_{\Spc P'}\{u\}=0$ for all $u\in \Spc U$.
Since $\Spc U$ and $\Spc P'$ are both Banach spaces, the direct-sum space $\Spc X_{\Lop_{\rm R}}=\Spc U \oplus \Spc P'$, equipped with the composite norm $\|f\|_{\Spc X_{\Lop_{\rm R}}}=\max(\|\Proj_{\Spc U}\{f\}\|_{\Spc U},
\|\Proj_{\Spc P'}\{f\}\|_{\Spc P'})$, is complete (Banach property) and isomorphic to $\Spc X_{\rm Rad} \times \Spc P'$. The final element is to recognize that
$\|\Proj_{\Spc U}\{f\}\|_{\Spc U}=\|\Op L_{\rm R}^{\ast\dagger}\Proj_{\Spc U}\{f\}\|_{\Spc X}=\|\Op L_{\rm R}^{\ast\dagger}\{f\}\|_{\Spc X}$, where $\Proj_{\Spc U}=(\Identity-\Proj_{\Spc P'})$.
This direct-sum decomposition has an equivalent description in terms of operators, which is the form given in the statement of Theorem \ref{Theo:PredualNative}.
Specifically, the isomorphism between $\Spc U$ and $\Spc X$ is expressed by \eqref{Eq:isoPrenative1} and \eqref{Eq:isoPrenative2} for $f \in \Spc U$. This is complemented by the null-space property \eqref{Eq:PrenativeNull}, which ensures that the components of $f$ that are in $\Spc P'$ are annihilated by $\Lop_{\rm R}^\ast$.

Embeddings:
The denseness of $\Spc S$ in $\Spc X_{\Op L_{\rm R}}$ follows from the observation that
$\Spc S(\R^d)=(\Identity-\Proj_{\Spc P'})\big(\Spc S(\R^d)\big) \oplus \Spc P'$. Since, by construction, one has that
$(\Identity-\Proj_{\Spc P'})(\Spc S) \embedD \Spc U$ and $(\Identity-\Proj_{\Spc P'})(L_{1,n_0}) \embedD \Spc U$, one also has that
$\Spc S(\R^d) \embedD L_{1,n_0}(\R^d) \embedD \Spc X_{\Op L_{\rm R}}(\R^d)$.

As for the relation $\Spc X_{\Op L_{\rm R}}=\Spc U \oplus \Spc P' \embedD  \Spc S'(\R^d)$, we already have that $\Spc P'\embedC \Spc S(\R^d) \embedC \Spc S'(\R^d)$, by construction. To show that $\Spc U \embedC \Spc S'(\R^d)$, we invoke the intertwining relation and express $\Spc U=\Lop^\ast\Op R^\ast \Op K_{\rm rad}(\Spc X_{\rm Rad})$ as $\Spc U=\Op R^\ast \Op Q_{\rm rad}(\Spc X_{\rm Rad})$, where 
$\Op Q_{\rm rad}=\Op L_{\rm rad}\Op K_{\rm rad}: \Spc X_{\rm Rad} \to \Spc S'(\R \times \mathbb{S}^{d-1})$.
% is a radial convolution operator whose frequency response is $\widehat{Q}_{\rm rad}(\omega)=c_d \widehat{L}_{\rm rad}(\omega)|\omega|^{d-1}$. 
Our hypotheses on $\widehat{L}_{\rm rad}$ ensure that $\Op Q_{\rm rad}\{\phi\}$ is well-defined for every $\phi \in \Spc X$, with the operator being continuous in the weak topology of
$\Spc S'(\R \times \mathbb{S}^{d-1})$. (As the latter is a complete nuclear space, the weak (sequential) convergence also ensures continuity in the strong topology \cite{Treves2006}.)
Since $\Op R^\ast: \Spc S'(\R \times \mathbb{S}^{d-1}) \toC \Spc S'(\R^d)$, as implied by  \eqref{Eq:RadjDis}, the composed operator $\Op R^\ast \Op Q_{\rm rad}:  \Spc X_{\rm Rad} \to \Spc S'(\R^d)$ is continuous as well, which proves that $\Spc U$ and $\Spc X_{\Op L_{\rm R}}$ are both continuously embedded in $\Spc S'(\R^d)$. The latter embedding is also dense by transitivity since $\Spc S(\R^d) \subset \Spc X_{\Op L_{\rm R}}$ and $\Spc S(\R^d) \embedD \Spc S'(\R^d)$ \cite{Schwartz:1966,Gelfand-Shilov1964}.

\end{proof}

\appendix
\section{Appendix: Duality Maps}
\label{App:DualityMap}
The generalization of the Cauchy-Schwarz inequality for any dual pair $(\Spc X, \Spc X')$ of Banach spaces  is
the so-called {\em duality inequality}
\begin{align}
\label{Eq:Dualitybound}
\forall (f,x) \in \Spc X' \times \Spc X: \langle f, x\rangle_{\Spc X' \times \Spc X}\le \|f\|_{\Spc X'}\|x\|_{\Spc X},
\end{align}
which is tightly linked to the definition of the dual norm given by
\begin{align}
\label{Eq:DualNorm}
\|f\|_{\Spc X'}\eqdef\sup_{x \in \Spc X: \|x\|_{\Spc X}\le 1} \langle f, x\rangle_{\Spc X' \times \Spc X}.
\end{align}
\begin{definition}[Strict convexity]
A Banach space $\Spc X$ %---or its associated norm $\|\cdot\|_{\Spc X}$---
is said to be strictly convex if, for all $x_1,x_2 \in \Spc X$ such that $\|x_1\|_{\Spc X} =\|x_2\|_{\Spc X} =1$
and $x_1 \ne x_2$, one has that $\|\theta x_1+(1-\theta)x_2\|_{\Spc X}< 1$ for all
$\theta \in (0,1)$.
\end{definition}

It is obvious from \eqref{Eq:DualNorm} that \eqref{Eq:Dualitybound} is sharp. Moreover, when $\Spc X$ 
is reflexive and strictly convex, there is a single element $x^\ast \in \Spc X'$ (the Banach conjugate of $x$)
such that $\|x^\ast\|_{\Spc X'}=\|x\|_{\Spc X}$ (isometry) and
$\langle x^\ast, x\rangle_{\Spc X' \times \Spc X}=\|x^\ast\|_{\Spc X'}\|x\|_{\Spc X}$ (sharp duality bound) \cite{Cioranescu2012}. This leads to the definition of the corresponding duality map $\Op J_\Spc X: \Spc X \to \Spc X'$ as
\begin{align}
\Op J_{\Spc X}\{x\}=x^\ast.
\end{align}
Since the dual of $\Spc X'$ is strictly convex as well, we have that $\Op J_\Spc X^{-1}=\Op J_{\Spc X'}: \Spc X' \to \Spc X$ with $\Op J_{\Spc X'}\{x^\ast\}=x$, where $(x^\ast)^\ast=x \in \Spc X''=\Spc X$
is the unique Banach conjugate of $x^\ast \in \Spc X'$.

A relevant example of reflexive and strictly convex Banach space is $\Spc X=L_q(\R \times \mathbb{S}^{d-1})$ for $q \in (1,\infty)$.
Its topological dual is
$\Spc X'=L_p(\R \times \mathbb{S}^{d-1})$ with $p=q/(q-1)$ being the conjugate exponent of $q$.
For that particular pair,
\eqref{Eq:Dualitybound} reduces to the H\"older inequality for hyper-spherical functions.
From \cite[Chapter 4]{Cioranescu2012}, the corresponding duality map $\Op J_q: L_q(\R \times \mathbb{S}^{d-1}) \to L_p(\R \times \mathbb{S}^{d-1})$
is 
\begin{align}
\label{Eq:LpDual}
\Op J_{q}\{\nu\}(\V z)=\nu^\ast(\V z)=\frac{\left|\nu(\V z)\right|^{q-1}}{\|\nu\|^{q-2}_{L_q}} {\rm sign}\big(\nu(\V z)\big),
\end{align}
which establishes a one-to-one isometric mapping between the spaces $L_q$ and 
$L_p=(L_q)'$ with the property that $\Op J^{-1}_{q}=\Op J_{p}$.

\begin{proposition}[Banach isometries]
\label{Prop:DualMaps}
Let $(\Spc X, \Spc X')$ be a dual pair of reflexive and strictly convex Banach spaces with corresponding duality map $\Op J_\Spc X: \Spc X \to \Spc X'$. We consider two generic types of linear isometries.
%(The latter is the non-linear map which, for any $x\in \Spc X$, returns the unique Banach conjugate $x^\ast=\Op J_\Spc X\{x\} \in \Spc X'$ such that $\|x^\ast\|_{\Spc X'}=\|x\|_{\Spc X}$ (isometry) and
%$\langle x^\ast, x\rangle_{\Spc X' \times \Spc X}=\|x^\ast\|_{\Spc X'}\|x\|_{\Spc X}$ (sharp duality bound).)

1) One-to-one map: Let $\Op T: \Spc X \to \Spc Y=\Op T(\Spc X)$ be an injective operator whose inverse is denoted by $\Op T^{-1}$ with $\Op T^{-1}\Op T=\Identity$ on $\Spc X$. Then, $\Spc Y=\Op T(\Spc X)=\{y=\Op T\{x\}: x \in \Spc X\}$ equipped with the norm $\|y\|_{\Spc Y}\eqdef\|\Op T^{-1}\{y\}\|_{\Spc X}$ is a Banach space. Its 
continuous dual is the Banach space $\Spc Y'=\Op T^{-1\ast}(\Spc X')$ with
$\|y'\|_{\Spc Y'}=\| \Op T^\ast\{ y'\}\|_{\Spc X'}$, while the
corresponding duality map is
$\Op J_{\Spc Y}= (\Op T^\ast)^{-1}\circ \Op J_{\Spc X}\circ\Op T^{-1}: \Spc Y \to \Spc Y'$.

2) Projection: Let $\Op P: \Spc X \to \Spc U \embedIso \Spc X$ be a continuous projection operator on $\Spc X$ with $\|\Op P\|=1$. Then, $\big(\Op P(\Spc X),\Op P^\ast(\Spc X')\big)=\big(\Spc U,\Spc U'\big)$ is a dual pair of Banach subspaces with corresponding duality map
$\Op J_{\Spc U}=\Op P^\ast\circ \Op J_{\Spc X}\circ\Op P: \Spc U \to \Spc U'$.

\end{proposition}
\begin{proof}
We first recall that the dual of a reflexive and strictly convex Banach space is reflexive (by definition) and strictly convex as well.
\item 1) Injective operator: 
For the first property, we refer to \cite[Proposition 1]{Unser2022}. The key observation is that the operators $\Op T: \Spc X \to \Spc Y$ and 
$\Op T^{-1}: \Spc Y \to \Spc X$, as well as their adjoints, are isometries with
$(\Op T^\ast)^{-1}=\Op T^{-1\ast}$. The argument then primarily rests upon the duality inequality 
\begin{align}
\langle y', y\rangle_{\Spc Y' \times \Spc Y}=\langle y', \Op T\Op T^{-1}\{y\}\rangle_{\Spc Y' \times \Spc Y}&=\langle \Op T^\ast\{y'\}, \Op T^{-1}\{y\}\rangle_{\Spc X' \times \Spc X}\nonumber \\
&\le \|\Op T^\ast\{y'\}\|_{\Spc X'}\; \|\Op T^{-1}\{y\}\|_{\Spc X},
\end{align}
which is sharp if and only if $x=\Op T^{-1}\{y\}$ and $x'=\Op T^\ast\{y'\}$ (resp., $y'$ and $y$)
are Banach conjugates, so that $x'=\Op J_{\Spc X}\{x\}$.

\item 2) Projection operator: The first part is obtained by using a standard argument with (projected) Cauchy sequences.
For the second part, let $u \in \Spc U\embedIso \Spc X$ with Banach conjugate
$u^\ast \in \Spc X'$. Then,
\begin{align}
\|u^\ast\|_{\Spc X'}\|u\|_{\Spc X}
=\langle u^\ast, u\rangle_{\Spc X' \times \Spc X} & =
\langle u^\ast, \Op P^2 u\rangle_{\Spc X' \times \Spc X}=\langle \Op P^\ast u^\ast, \Op P \Op u\rangle_{\Spc U' \times \Spc U}\nonumber\\
& \ \le \|\Op P^\ast\| \|u^\ast\|_{\Spc U'} \|\Op P\| \|u\|_{\Spc U},
\end{align}
from which  we deduce that $\langle \Op P^\ast u^\ast, \Op P \Op u\rangle_{\Spc U' \times \Spc U}=\|u^\ast\|_{\Spc U'} \|u\|_{\Spc U'}$. We conclude that $u=\Op P u$ and $u^\ast=\Op P^\ast u^\ast=
\Op P^\ast \circ \Op J_{\Spc X}\{\Op P u\}$ are (unique) Banach conjugates of each other.
\end{proof}

\subsection*{Acknowlegdments}
The research was supported by the European Commission under Grant ERC-2020-AdG 
FunLearn-101020573 .
The author is thankful to Sebastian Neumayer for very helpful discussions.

%\bibliographystyle{plain}
%\bibliography{Unser_DNNs}\label{refs}
\bibliographystyle{siam}
%
% To get the path, drag file "Unser.bib" on terminal
%
% ln -s ~/Google\ Drive/My\ Drive/ MyDrive
% ln -s ~/Google\ Drive/ MyDrive
%
\bibliography{/Users/unser/MyDrive/Bibliography/Bibtex_files/Unser.bib}

\begin{thebibliography}{10}

\bibitem{Alvarez2012}
{\sc M.~A. Alvarez, L.~Rosasco, and N.~D. Lawrence}, {\em Kernels for
  vector-valued functions: A review}, Foundations and Trends in Machine
  Learning, 4 (2012), pp.~195--266.

\bibitem{Aronszajn1950}
{\sc N.~Aronszajn}, {\em Theory of reproducing kernels}, Transactions of the
  American Mathematical Society, 68 (1950), pp.~337--404.

\bibitem{Aziznejad2021}
{\sc S.~Aziznejad and M.~Unser}, {\em Multikernel regression with sparsity
  constraint}, {SIAM} Journal on Mathematics of Data Science, 3 (2021),
  pp.~201--224.

\bibitem{Bach2017}
{\sc F.~Bach}, {\em Breaking the curse of dimensionality with convex neural
  networks}, Journal of Machine Learning Research, 18 (2017), pp.~1--53.

\bibitem{Barron1993}
{\sc A.~R. Barron}, {\em Universal approximation bounds for superpositions of a
  sigmoidal function}, {IEEE} Transactions on Information Theory, 39 (1993),
  pp.~930--945.

\bibitem{Bartolucci2021}
{\sc F.~Bartolucci, E.~D. Vito, L.~Rosasco, and S.~Vigogna}, {\em Understanding
  neural networks with reproducing kernel {B}anach spaces}, arXiv:2109.09710,
  (2021).

\bibitem{Berlinet2004}
{\sc A.~Berlinet and C.~Thomas-Agnan}, {\em Reproducing Kernel Hilbert Spaces
  in Probability and Statistics}, vol.~3, Kluwer Academic Boston, 2004.

\bibitem{Bishop2006}
{\sc C.~M. Bishop}, {\em Pattern Recognition and Machine Learning}, Springer,
  2006.

\bibitem{Boyer2018}
{\sc C.~Boyer, A.~Chambolle, Y.~De~Castro, V.~Duval, F.~De~Gournay, and
  P.~Weiss}, {\em On representer theorems and convex regularization}, SIAM
  Journal of Optimization, 29 (2019), pp.~1260--1281.

\bibitem{Bredies2018}
{\sc K.~Bredies and M.~Carioni}, {\em Sparsity of solutions for variational
  inverse problems with finite-dimensional data}, Calculus of Variations and
  Partial Differential Equations, 59 (2020), p.~26.

\bibitem{Buhmann2003}
{\sc M.~D. Buhmann}, {\em Radial Basis Functions}, Cambridge University Press,
  2003.

\bibitem{Chen2002}
{\sc Z.~Chen and S.~Haykin}, {\em On different facets of regularization
  theory}, Neural Computation, 14 (2002), pp.~2791--2846.

\bibitem{Cioranescu2012}
{\sc I.~Cioranescu}, {\em Geometry of Banach Spaces, Duality Mappings and
  Nonlinear Problems}, vol.~62, Springer Science \& Business Media, 2012.

\bibitem{Cybenko1989}
{\sc G.~Cybenko}, {\em Approximation by superpositions of a sigmoidal
  function}, Mathematics of Control, Signals and Systems, 2 (1989),
  pp.~303--314.

\bibitem{DeBoor1976}
{\sc C.~de~Boor}, {\em On ``best'' interpolation}, Journal of Approximation
  Theory, 16 (1976), pp.~28--42.

\bibitem{deBoor1966}
{\sc C.~de~Boor and R.~E. Lynch}, {\em On splines and their minimum
  properties}, Journal of Mathematics and Mechanics, 15 (1966), pp.~953--969.

\bibitem{Duchon1977}
{\sc J.~Duchon}, {\em Splines minimizing rotation-invariant semi-norms in
  {S}obolev spaces}, in Constructive Theory of Functions of Several Variables,
  W.~Schempp and K.~Zeller, eds., Springer-Verlag, Berlin, 1977, pp.~85--100.

\bibitem{Fisher1975}
{\sc S.~D. Fisher and J.~W. Jerome}, {\em Spline solutions to {$L_1$} extremal
  problems in one and several variables}, Journal of Approximation Theory, 13
  (1975), pp.~73--83.

\bibitem{Gelfand-Shilov1964}
{\sc I.~M. Gelfand and G.~Shilov}, {\em Generalized Functions. {V}ol. 1.
  {P}roperties and Operations}, Academic Press, New York, 1964.

\bibitem{Gelfand1966}
\leavevmode\vrule height 2pt depth -1.6pt width 23pt, {\em Generalized
  Functions. {V}ol. 5. {I}ntegral Geometry and Representation Theory}, Academic
  Press, New York, 1966.

\bibitem{Grafakos2008}
{\sc L.~Grafakos}, {\em Classical {F}ourier Analysis}, Springer, 2008.

\bibitem{He2016}
{\sc K.~He, X.~Zhang, S.~Ren, and J.~Sun}, {\em Deep residual learning for
  image recognition}, in Proc. IEEE Conference on Computer Vision and Pattern
  Recognition, 2016, pp.~770--778.

\bibitem{Helgason2011}
{\sc S.~Helgason}, {\em Integral Geometry and {R}adon Transforms}, Springer,
  2011.

\bibitem{Hornik1989}
{\sc K.~Hornik, M.~Stinchcombe, and H.~White}, {\em Multilayer feedforward
  networks are universal approximators}, Neural Networks, 2 (1989),
  pp.~359--366.

\bibitem{Kostadinova2014}
{\sc S.~Kostadinova, S.~Pilipovi{\'c}, K.~Saneva, and J.~Vindas}, {\em The
  ridgelet transform of distributions}, Integral Transforms and Special
  Functions, 25 (2014), pp.~344--358.

\bibitem{Ludwig1966}
{\sc D.~Ludwig}, {\em The {R}adon transform on {E}uclidean space},
  Communications on Pure and Applied Mathematics, 19 (1966), pp.~49--81.

\bibitem{Mammen1997}
{\sc E.~Mammen and S.~van~de Geer}, {\em Locally adaptive regression splines},
  Annals of Statistics, 25 (1997), pp.~387--413.

\bibitem{Meinguet1979}
{\sc J.~Meinguet}, {\em Multivariate interpolation at arbitrary points made
  simple}, Zeitschrift fur Angewandte Mathematik und Physik, 30 (1979),
  pp.~292--304.

\bibitem{Mhaskar1992}
{\sc H.~Mhaskar and C.~A. Micchelli}, {\em Approximation by superposition of
  sigmoidal and radial basis functions}, Advances in Applied Mathematics, 13
  (1992), pp.~350--373.

\bibitem{Micchelli1986}
{\sc C.~A. Micchelli}, {\em Interpolation of scattered data: {D}istance
  matrices and conditionally positive definite functions}, Constructive
  Approximation, 2 (1986), pp.~11--22.

\bibitem{Micchelli2006}
{\sc C.~A. Micchelli, Y.~Xu, and H.~Zhang}, {\em Universal kernels}, Journal of
  Machine Learning Research, 7 (2006), pp.~2651--2667.

\bibitem{Natterer1984}
{\sc F.~Natterer}, {\em The Mathematics of Computed Tomography}, John Willey \&
  Sons Ltd., 1984.

\bibitem{Neumayer2022}
{\sc S.~Neumayer and M.~Unser}, {\em Explicit representations for {B}anach
  subspaces of {L}izorkin distributions}, Preprint arXiv:2203.05312,  (2022).

\bibitem{Ongie2020b}
{\sc G.~Ongie, R.~Willett, D.~Soudry, and N.~Srebro}, {\em A function space
  view of bounded norm infinite width {ReLU} nets: The multivariate case},
  International Conference on Representation Learning (ICLR),  (2020).

\bibitem{Parhi2020}
{\sc R.~Parhi and R.~D. Nowak}, {\em The role of neural network activation
  functions}, IEEE Signal Processing Letters, 27 (2020), pp.~1779--1783.

\bibitem{Parhi2021b}
{\sc R.~Parhi and R.~D. Nowak}, {\em Banach space representer theorems for
  neural networks and ridge splines}, Journal of Machine Learning Research, 22
  (2021), pp.~1--40.

\bibitem{Pinkus1999}
{\sc A.~Pinkus}, {\em Approximation theory of the {MLP} model in neural
  networks}, Acta Numerica, 8 (1999), pp.~143--195.

\bibitem{Poggio1990b}
{\sc T.~Poggio and F.~Girosi}, {\em Networks for approximation and learning},
  Proceedings of the IEEE, 78 (1990), pp.~1481--1497.

\bibitem{Poggio1990}
\leavevmode\vrule height 2pt depth -1.6pt width 23pt, {\em Regularization
  algorithms for learning that are equivalent to multilayer networks}, Science,
  247 (1990), pp.~978--982.

\bibitem{Poggio2003}
{\sc T.~Poggio and S.~Smale}, {\em The mathematics of learning: Dealing with
  data}, Notices of the AMS, 50 (2003), pp.~537--544.

\bibitem{Ramm2020}
{\sc A.~G. Ramm and A.~I. Katsevich}, {\em The {R}adon transform and local
  tomography}, CRC Press, 2020.

\bibitem{Reed1980}
{\sc M.~Reed and B.~Simon}, {\em Methods of Modern Mathematical Physics. {V}ol.
  1: {F}unctional Analysis}, Academic Press, San Diego, 1980.

\bibitem{Samko1982denseness}
{\sc S.~G. Samko}, {\em Denseness of the {L}izorkin-type spaces {$\Phi_V$} in
  ${L_p(\R^n)}$}, Mathematical Notes of the Academy of Sciences of the USSR, 31
  (1982), pp.~432--437.

\bibitem{Samko1993}
{\sc S.~G. Samko, A.~A. Kilbas, and O.~I. Marichev}, {\em Fractional Integrals
  and Derivatives: Theory and Applications}, Gordon and Breach Science
  Publishers, 1993.

\bibitem{Savarese2019}
{\sc P.~Savarese, I.~Evron, D.~Soudry, and N.~Srebro}, {\em How do infinite
  width bounded norm networks look in function space?}, in Conference on
  Learning Theory. PMLR, 2019, pp.~2667--2690.

\bibitem{Scholkopf2001}
{\sc B.~Sch{\"o}lkopf, R.~Herbrich, and A.~J. Smola}, {\em A generalized
  representer theorem}, in Computational Learning Theory, D.~Helmbold and
  B.~Williamson, eds., Springer Berlin Heidelberg, 2001, pp.~416--426.

\bibitem{Scholkopf1997}
{\sc B.~Sch{\"o}lkopf, K.-K. Sung, C.~J.~C. Burges, F.~Girosi, P.~Niyogi,
  T.~Poggio, and V.~Vapnik}, {\em Comparing support vector machines with
  {G}aussian kernels to radial basis function classifiers}, {IEEE} Transactions
  on Signal Processing, 45 (1997), pp.~2758--2765.

\bibitem{Schwartz:1966}
{\sc L.~Schwartz}, {\em Th\'eorie des Distributions}, Hermann, Paris, 1966.

\bibitem{Sonoda2017}
{\sc S.~Sonoda and N.~Murata}, {\em Neural network with unbounded activation
  functions is universal approximator}, Applied and Computational Harmonic
  Analysis, 43 (2017), pp.~233--268.

\bibitem{Treves2006}
{\sc F.~Tr{\`e}ves}, {\em Topological Vector Spaces, Distributions and
  Kernels}, Dover Publications, New York, 2006.

\bibitem{Unser_2020}
{\sc M.~Unser}, {\em A unifying representer theorem for inverse problems and
  machine learning}, Foundations of Computational Mathematics, 21 (2021),
  pp.~941--960.

\bibitem{Unser2022_Ridges}
{\sc M.~Unser}, {\em Ridges, neural networks, and the {R}adon transform},
  arXiv:2203.02543,  (2022).

\bibitem{Unser2022}
{\sc M.~Unser and S.~Aziznejad}, {\em Convex optimization in sums of {B}anach
  spaces}, Applied and Computational Harmonic Analysis, 56 (2022), pp.~1--25.

\bibitem{Unser2000a}
{\sc M.~Unser and T.~Blu}, {\em Fractional splines and wavelets}, {SIAM}
  Review, 42 (2000), pp.~43--67.

\bibitem{Unser2017}
{\sc M.~Unser, J.~Fageot, and J.~P. Ward}, {\em Splines are universal solutions
  of linear inverse problems with generalized-{TV} regularization}, SIAM
  Review, 59 (2017), pp.~769--793.

\bibitem{Unser2014book}
{\sc M.~Unser and P.~D. Tafti}, {\em An Introduction to Sparse Stochastic
  Processes}, Cambridge University Press, 2014.

\bibitem{Wahba1990}
{\sc G.~Wahba}, {\em Spline Models for Observational Data}, Society for
  Industrial and Applied Mathematics, Philadelphia, 1990.

\bibitem{Wendland2005}
{\sc H.~Wendland}, {\em Scattered Data Approximations}, Cambridge University
  Press, 2005.

\end{thebibliography}

% To construct restricted sublibrary,  under tools in Jabref, "New sublibrary based on AUX file"

\end{document}